\documentclass{article}

 \usepackage[preprint]{neurips_2026}

\usepackage[utf8]{inputenc} 
\usepackage[T1]{fontenc}    
\usepackage{hyperref}       
\usepackage{url}            
\usepackage{booktabs}       
\usepackage{amsfonts}       
\usepackage{nicefrac}       
\usepackage{xcolor}         

\usepackage{optidef}
\usepackage{mathrsfs}
\usepackage{float}
\usepackage[dvipsnames]{xcolor} 
\usepackage{graphicx}
\usepackage{subcaption}
\usepackage{wrapfig}
\usepackage{enumitem}
\usepackage{amsmath}
\usepackage{amsthm}
\usepackage{amssymb}
\usepackage{algorithm,algorithmic}
\usepackage{comment}
\usepackage{multirow}
\usepackage{mathtools}
\usepackage{dsfont}
\usepackage{tikz}
\usetikzlibrary{arrows.meta,calc,positioning,decorations.pathreplacing}
\usepackage{titletoc}

\usepackage{amsmath,amsfonts,bm}

\def\eqref#1{equation~\ref{#1}}

\def\1{\bm{1}}

\def\vd{{\bm{d}}}
\def\ve{{\bm{e}}}

\def\vg{{\bm{g}}}

\def\vr{{\bm{r}}}
\def\vs{{\bm{s}}}

\def\vx{{\bm{x}}}
\def\vy{{\bm{y}}}
\def\vz{{\bm{z}}}

\def\mA{{\bm{A}}}
\def\mB{{\bm{B}}}
\def\mC{{\bm{C}}}
\def\mD{{\bm{D}}}

\def\mG{{\bm{G}}}

\def\mI{{\bm{I}}}

\def\mO{{\bm{O}}}

\def\mQ{{\bm{Q}}}
\def\mR{{\bm{R}}}
\def\mS{{\bm{S}}}

\def\mW{{\bm{W}}}
\def\mX{{\bm{X}}}
\def\mY{{\bm{Y}}}
\def\mZ{{\bm{Z}}}

\DeclareMathAlphabet{\mathsfit}{\encodingdefault}{\sfdefault}{m}{sl}
\SetMathAlphabet{\mathsfit}{bold}{\encodingdefault}{\sfdefault}{bx}{n}

\def\gB{{\mathcal{B}}}

\newtheorem{theorem}{Theorem}[section]

\newtheorem{definition}[theorem]{Definition}
\newtheorem{proposition}[theorem]{Proposition}
\newtheorem{corollary}[theorem]{Corollary}

\DeclareMathOperator{\Tr}{Tr}

\hypersetup{
    colorlinks=true,
    linkcolor=MidnightBlue,
    citecolor=MidnightBlue,
    urlcolor=MidnightBlue,
    anchorcolor=MidnightBlue,
    allcolors=MidnightBlue, 
}

\newtheorem*{algorithm_state*}{Algorithm}
\allowdisplaybreaks
\setitemize{itemsep=2pt,topsep=0pt,parsep=0pt,partopsep=0pt,leftmargin=*}

\bibpunct{[}{]}{,}{n}{,}{,}
\setcitestyle{numbers,open={[},close={]}}
\title{Normative Networks for Source Separation via Local Plasticity and Dendritic Computation
}

\author{%
Bariscan Bozkurt$^{1,3,}$\footnotemark[1]\enspace\footnotemark[2] \quad
Efe Ali Gorguner$^{2,}$\footnotemark[2] \quad
Francesco Innocenti$^{3,4}$ \quad
Rafal Bogacz$^{3,4}$\\
$^{1}$Gatsby Computational Neuroscience Unit, University College London, UK\\
$^{2}$Department of Computer Science, University of Oxford, UK\\
$^{3}$Brain Network Dynamics Unit, University of Oxford, UK\\
$^{4}$MRC Centre of Research Excellence in Restorative Neural Dynamics, UK
}

\begin{document}

\maketitle
\footnotetext[1]{Work completed in part while visiting the University of Oxford.}
\footnotetext[2]{Equal contribution.}

\begin{abstract}

Blind source separation (BSS) is a natural framework for studying how latent causes may be recovered from sensory mixtures, but deriving online and biologically plausible algorithms for structured (i.e., constrained to known domains) and potentially correlated sources remains challenging. Recent work has derived neural networks for BSS from maximization of an entropy measure, yet its online implementations involve complex and nonlocal recurrent dynamics. Motivated by this perspective, we propose \textit{Predictive Entropy Maximization}, which achieves competitive performance in BSS, using only local weight updates. The method employs a close approximation of an entropy measure, yielding an objective function with easily interpretable components. Minimizing this objective leads to a predictive neural architecture in which feedforward synapses follow an error-driven rule (that can be realized through dendritic mechanisms), lateral inhibitory connections are learned with local Hebbian plasticity, and source-domain constraints are enforced through simple output nonlinearities. We derive explicit spectral bounds on the surrogate error, characterizing when the approximation is accurate. Empirically, \textit{Predictive Entropy Maximization} remains robust under increasing source correlation and observation noise, outperforms biologically plausible algorithms that rely on stronger independence or decorrelation assumptions, and remains competitive with exact determinant- and correlative-information-based baselines. These results show how local plasticity and adaptive lateral inhibition can emerge from maximizing a regularized second-order entropy over structured source domains. Our implementation code is available at \url{https://github.com/BariscanBozkurt/Predictive-Entropy-Maximization}.\looseness=-1

\end{abstract}

\section{Introduction and related works}

Natural sensory inputs are rarely generated by a single cause. In audition, several sound sources may overlap; in vision, objects, textures, and illumination are superimposed; in olfaction, sensory responses reflect mixtures of odorants. The problem of recovering the underlying causes from mixtures alone is known as blind source separation (BSS) \citep{comon2010handbook,cichocki2009nonnegative}. BSS is a classical problem in signal processing and machine learning, and it has also served as a normative model of sensory representation learning in neuroscience. Sparse and independent decompositions of natural stimuli produce features reminiscent of receptive fields in early sensory systems \citep{Bell_ICA_InfoMax,olshausen1996emergence,lewicki2002efficient}, while the cocktail-party problem illustrates the broader challenge of source separation in the presence of strong interference \citep{cherry1953some,bronkhorst2000cocktail,haykin2005cocktail}.\looseness=-1

Because BSS is ill-posed without additional structure, identifiability must come from assumptions on the latent sources. Independent component analysis (ICA), including the information-maximization (InfoMax) formulation, achieves this through statistical independence \citep{Bell_ICA_InfoMax,Hyvrinen1999IndependentCA,HyvarinenOja2000ICA}. A complementary geometric line of work instead assumes that the sources lie in a structured domain---for example, a simplex (nonnegative entries summing to one) \citep{Lin2015SimplexIdentifiability,Lin2018Simplex}, a sparse or bounded set \citep{babatas2018algorithmic,erdogan2013class}, the nonnegative orthant (all coordinates nonnegative) \citep{Paatero_NMF,Xiao_NMF}, or, more generally, a polytope (a convex bounded set with finitely many vertices) \citep{Tatli_PMF,Tatli_GeneralizedPMF,bozkurt_identifiable}.
These geometric approaches seek outputs that spread across the available dimensions of the source domain, through maximization of the determinant of the covariance matrix of output signals.
An appealing feature of this approach is that it retains information about inputs as it avoids degenerate low-rank solutions.
Another useful property of these methods is that it can recover dependent or correlated sources (under suitable scattering assumptions)
\citep{Lin2015SimplexIdentifiability,Tatli_PMF}. 
Recent work further shows that determinant-based structured source separation admits an information-theoretic interpretation through a second-order entropy measure, called \textit{correlative entropy} (based on the log-determinant of the regularized output covariance matrix) \citep{Erdogan_CorInfoMax}.\looseness=-1

A separate challenge, central to neuroscience and neuromorphic computing, is how to realize such objectives with local dynamics and plasticity. Biologically plausible BSS networks have been developed for ICA settings \citep{Isomura2016,Isomura2018,bahroun2021a}, for nonnegative and bounded source separation through similarity matching \citep{Pehlevan_NSM_2017,Erdogan2020BlindBS}, and more recently for correlated-source separation under determinant-based objectives \citep{simsek_OnlineBCA,bozkurt2022biologicallyplausible,bozkurt2023correlative}. These works show that source-domain structure can be implemented through neural nonlinearities, inhibition, and local learning rules. Nonetheless, the online determinant-based formulations for correlated sources still tend to induce complex recurrent dynamics, because optimizing the exact log-determinant objective naturally brings inverse-covariance
into the update.

Our goal is therefore to retain the robustness of determinant-based correlated-source separation while replacing these dynamics with a more directly local and biologically plausible recurrent implementation.
Specifically, we strive to derive neural networks which satisfy previously proposed criteria of plausibility \citep{bogacz2017tutorial}: (i) local computation, where a neuron's activity is only based on its inputs, and (ii) local plasticity, where weight modification is based on the variables encoded in pre- and post-synaptic neurons. To this end, we replace the exact log-determinant with an online second-order surrogate obtained by a Taylor expansion around the diagonal part of the output covariance. The resulting objective decomposes into two interpretable terms: one encourages large variance in each output dimension, while the other penalizes redundant dependencies through normalized cross-covariances. This, in turn, yields a predictive neural architecture with a direct local interpretation: feedforward synapses follow a local error-driven rule, recurrent interactions are mediated by local covariance traces and implement adaptive lateral inhibition, and source-domain constraints are enforced through simple output nonlinearities. In this way, we preserve the determinant-based geometric viewpoint while obtaining a more directly local and biologically plausible recurrent implementation. \looseness=-1

Beyond determinant-based BSS, our construction is also connected to two adjacent lines of work. First, our approach is related to mechanistic models of learning in neural circuits. These models include predictive coding which emphasizes recurrent inference and local prediction errors as a route to learning \citep{rao1999predictive,whittington2019theories}, and supervised Correlative Information Maximization (CorInfoMax) which shows that predictive pathways and lateral interactions can arise from a log-determinant objective that encourages layer activities to spread across their available dimensions \citep{bozkurt2023supervisedcorinfomax}. Likewise, predictive-plasticity models highlight the joint role of predictive and Hebbian mechanisms in sensory representation learning \citep{Halvagal2023}. 
Second, more abstractly, our surrogate objective is also related to redundancy-reduction methods from self-supervised learning, such as VICReg, which combine variance preservation with penalties on cross-component correlations \citep{bardes2022vicreg}. These latter methods are not biologically plausible models, but they provide a useful point of comparison for the variance-covariance structure of our objective. Our setting remains distinct from both lines of work, since observations arrive one at a time as linear mixtures of latent sources, and the recovered sources must satisfy explicit domain constraints.


The main contributions of our \textit{Predictive Entropy Maximization} model are:
\begin{itemize}
    \item A Taylor approximation of the online determinant-maximization objective, obtained by expanding the log-determinant of the regularized output covariance around its diagonal part. Unlike exact online determinant-based formulations, the resulting updates depend directly on covariance statistics rather than inverse-covariance states.
    
    \item A family of predictive neural architectures for different source domains, induced by the proposed approximation. Feedforward synapses follow local error-driven plasticity, recurrent interactions are mediated by local covariance traces, and source-domain constraints are enforced through simple output nonlinearities, as illustrated in Figure~\ref{fig:predictivedecorr_two_architectures}.

    \item A theoretical characterization of the surrogate, including explicit error bounds (derived through an exact spectral representation of the Taylor remainder).

    \item Empirical evidence that our method remains robust with increasing source correlation and observation noise, competes with bio-plausible baselines built on stronger decorrelation assumptions, and extends naturally to auditory source separation and sparse receptive-field learning.
\end{itemize}
The remainder of the paper is organized as follows. After defining the BSS setting and the determinant-maximization objective (Section~\ref{sec:ProblemSetup}), we derive the online Taylor-approximated objective and the resulting predictive neural dynamics (Section~\ref{sec:DecorrelationObjective}). We then present numerical results evaluating our model (Section~\ref{sec:NumericalExperiments}), before concluding with the limitations and future directions of this work (Section~\ref{sec:Conclusion}).

\section{Problem setup for blind source separation}
\label{sec:ProblemSetup}

We now review the formulation of the linear BSS setting considered in this work \citep{comon2010handbook,cichocki2009nonnegative}. Each observation vector \(\vx(t)\in\mathbb{R}^m\) is generated by linearly mixing an unknown source vector \(\vs(t)\in\mathcal P\subset\mathbb{R}^n\) through an unknown matrix \(\mA\in\mathbb{R}^{m\times n}\):
\begin{align*}
    \vx(t) = \mA \vs(t), \qquad t \in \{1,2,\ldots,T\}.
\end{align*}
Here, the set $\mathcal P$ denotes a prescribed source domain that encodes structural assumptions on the source vectors. We focus on the determined or overdetermined regime ($m \ge n$) and assume that $\mA$ has full column rank. Stacking the samples in the snapshot matrices $\mS = [\vs(1), \dots, \vs(T)]$ and $\mX = [\vx(1), \dots, \vx(T)]$, the model can be written compactly as $\mX = \mA \mS$. The goal of BSS is to recover the source vectors from the mixtures alone by learning a separator $\mW \in \mathbb{R}^{n \times m}$ such that the outputs $\vy(t) = \mW \vx(t)$ match the original sources. Since linear BSS is non-unique without additional structural assumptions on $\mathcal P$, successful recovery is defined only up to permutation and sign ambiguities, as formalized next.

\begin{definition}[Ideal separation]
A separator \(\mW\) achieves ideal separation if the recovered outputs satisfy \(\mY=\mW\mX=\mathbf{\Pi}\mathbf{\Lambda}\mS\), where \(\mathbf{\Pi}\) is a permutation matrix and \(\mathbf{\Lambda}\) is a diagonal sign matrix. 
\label{def:IdealSeparation}
\end{definition}

The BSS model is \textit{identifiable} over a source domain \(\mathcal P\) if every exact factorization \(\mX=\mA'\mS'\) with columns of \(\mS'\) in \(\mathcal P\) agrees with the true factorization up to the ideal-separation ambiguity of Definition~\ref{def:IdealSeparation}. For the source domains considered here, this identifiability property is established in the determinant-maximization literature for polytopic and simplex-structured source models \citep{Lin2015SimplexIdentifiability, Tatli_GeneralizedPMF}.\looseness=-1

\noindent\textbf{Determinant-maximization.}
A natural way to exploit the geometry of \(\mathcal P\) is through the determinant-maximization criterion, which selects, among all feasible outputs constrained to lie in \(\mathcal P\), those with maximal second-order spread as measured by the log-determinant of their sample autocorrelation or autocovariance matrix \citep{Lin2018Simplex,Tatli_PMF,Erdogan_CorInfoMax}. The corresponding recovery guarantees additionally require the source samples to be sufficiently scattered in \(\mathcal P\). Geometrically, this means that the convex hull of the source samples must capture the shape of \(\mathcal P\) more faithfully than its best ellipsoidal approximation
\citep{Lin2018Simplex,Tatli_PMF}.\looseness=-1

Accordingly, we consider the batch determinant-maximization problem
\begin{equation}
\label{eq:det_max_obj}
\underset{\mW,\mY}{\text{minimize}}
\enspace
-\log\det(\hat{\mC})
\quad
\text{subject to}
\quad
\vy(t)=\mW\vx(t),\enspace \text{and}\enspace \vy(t)\in\mathcal P, \quad t=1,\ldots,T.
\end{equation}
In the above objective, $\hat{\mC}$ is the sample autocovariance matrix defined as:
\begin{equation}
\label{eq:covariance_definition}
\hat{\mC}
=
\frac{1}{T}\sum_{t=1}^{T}
\bigl(\vy(t)-\hat{\bm{\mu}}\bigr)\bigl(\vy(t)-\hat{\bm{\mu}}\bigr)^\top, 
\end{equation}
where \(\hat{\bm{\mu}}
=
\frac{1}{T}\sum_{t=1}^{T}\vy(t)\) is the sample mean. 

\noindent\textbf{Entropy interpretation.}
The above optimization problem has an information-theoretic interpretation. 
Following recent work \citep{Erdogan_CorInfoMax}, the regularized log-determinant of a sample covariance matrix is approximately proportional to a deterministic second-order entropy measure, which we refer to as the \textit{correlative entropy} (CE):
\begin{align}
\hat{\mathcal H}^{(\varepsilon)}_{\mathrm{CE}}(\mY)
=
\frac{1}{2}\log\det(\hat{\mC}+\varepsilon\mI)
+
\frac{n}{2}\log(2\pi e),
\label{eq:CorEntropy}
\end{align}
where \(\varepsilon>0\) is a small regularization constant. Under this interpretation, the optimization problem in Eq.~\ref{eq:det_max_obj} can be viewed as maximizing the output correlative entropy subject to the source-domain and mixing constraints. Further details on this second-order entropy and the related correlative-information definitions are given in Appendix~\ref{appendix:Review_CorInfoMax}.

\noindent\textbf{Source domains.}
In this paper, we consider the following source domains for which the BSS model is identifiable under the scattering assumptions discussed above:
\begin{itemize}
    \item \textit{Antisparse sources:} \(\gB_{\mathrm{max}} := \{\vs \in \mathbb{R}^n :
    \max_i |s_i| \le 1 \} \).
    \item \textit{Nonnegative antisparse sources:} \(\gB_{\mathrm{max},+} := \gB_{\mathrm{max}} \cap \mathbb{R}_+^n\), where \(\mathbb{R}_+^n\) is the nonnegative orthant.
    \item \textit{Sparse sources:} \(\gB_1 := \{\vs \in \mathbb{R}^n : \|\vs\|_1 \le 1\}\), where \(\|\cdot\|_1\) denotes the \(\ell_1\) norm.
    \item \textit{Nonnegative sparse sources:} \(\gB_{1,+} := \gB_1 \cap \mathbb{R}_+^n\).
    \item \textit{Simplex sources:} \(\Delta := \{\vs \in \mathbb{R}^n : s_i \geq 0\ \forall i,\ \sum_i s_i
    = 1\}\).
\end{itemize}

\section{Online surrogate objective for blind source separation}
\label{sec:DecorrelationObjective}

In this section, we derive an online approximation to the regularized correlative-entropy objective in Eq.~\ref{eq:CorEntropy} and show how it leads to biologically plausible learning dynamics. We also summarize the corresponding approximation error.

\noindent\textbf{From batch correlative entropy to an online objective.} 
To make the correlative-entropy objective online, we replace the batch output covariance by an exponentially weighted estimate, following the standard streaming-statistics construction also used in online CorInfoMax \citep{bozkurt2023correlative}. At time \(t\), the algorithm receives \(\vx(t)\), infers \(\vy(t)\) using pre-update weights and statistics, and then updates these state variables. Specifically, we replace \(\hat{\mC}\) defined in Eq.~\ref{eq:CorEntropy} with:
\begin{align}
    \hat{\mC}^{\lambda}(t)
    =
    \frac{1}{\eta(t)}
    \sum_{t'=1}^{t}
    \lambda^{\,t-t'}
    \bigl(\vy(t')-\hat{\bm{\mu}}^\lambda(t')\bigr)
    \bigl(\vy(t')-\hat{\bm{\mu}}^\lambda(t')\bigr)^\top,
    \label{eq:ExponentiallyWeightedSampleCovariance}
\end{align}
where \(\lambda\in(0,1)\) is a forgetting factor, $\eta(t)=\sum_{t'=1}^{t}\lambda^{\,t-t'},$ and $\hat{\bm{\mu}}^\lambda(t)
=
\eta(t)^{-1}
\sum_{t'=1}^{t}\lambda^{\,t-t'}\vy(t').$
The corresponding exact finite-time recursions are given in Appendix~\ref{appendix:UpdateRuleDerivations}. Since \(\eta(t)^{-1}\to 1-\lambda\) as \(t\to\infty\), the estimators can be updated online based on a new sample $\vy(t)$:
\begin{align}
    \hat{\bm{\mu}}^\lambda(t)
    &= \lambda \hat{\bm{\mu}}^\lambda(t-1) + (1-\lambda)\vy(t),
    \label{eq:mean_update}
    \\
    \hat{\mC}^{\lambda}(t)
    &= \lambda \hat{\mC}^{\lambda}(t-1)
    + (1-\lambda)\bigl(\vy(t)-\hat{\bm{\mu}}^\lambda(t)\bigr)\bigl(\vy(t)-\hat{\bm{\mu}}^\lambda(t)\bigr)^\top.
    \label{eq:cov_update}
\end{align}

Thus, the exact online counterpart of the batch problem is to minimize the negative output-entropy term, \(-\log\det(\hat{\mC}^{\lambda}(t)+\varepsilon\mI)\), under the source-domain constraint \(\vy(t)\in\mathcal P\) and the input-output relation \(\vy(t)=\mW\vx(t)\). Existing online formulations based on correlative information maximize a related mutual-information objective built from this output-entropy term together with an additional error-entropy term \citep{bozkurt2023correlative}. Direct gradients of these log-determinant objectives involve inverse covariance or correlation matrices, leading to inverse second-order states updated through the matrix inversion lemma. These inverse-state updates are less directly synapse-local than covariance-trace updates, making the local-plasticity interpretation less immediate; see Appendix~\ref{app:Comparison}. In contrast, our approach introduced below avoids this step altogether. By replacing the exact regularized output-entropy term with a second-order Taylor surrogate, we obtain an objective whose gradients depend directly on covariance statistics rather than on their inverse. This preserves the entropy-maximization viewpoint while leading to recurrent interactions and plasticity rules with a more direct local interpretation.

\noindent\textbf{Approximate entropy objective.} 
We present here our first result - an approximate objective that can optimised for BSS.
Since below we apply a Taylor expansion around the diagonal of the covariance estimate $\hat{\mC}^{\lambda}(t)$, we decompose it as \(\hat{\mC}^{\lambda}(t)=\hat{\mD}^{\lambda}(t)+\hat{\mO}^{\lambda}(t)\), where \(\hat{\mD}^{\lambda}(t)\) is diagonal and \(\hat{\mO}^{\lambda}(t)\) collects the off-diagonal cross-covariances. 
We apply a second-order Taylor expansion to the $\log \det$ objective of the regularized covariance matrix, as formalized in Theorem~\ref{thm:logdet_spectral_remainder_main} of Appendix~\ref{sec:TaylorRemainderAnalysis},
\begin{align}
    -\log\det\!\bigl(\hat{\mC}^{\lambda}(t)+\varepsilon\mI\bigr)
    =
    \underbrace{
    - \sum_{i=1}^{n} \log \bigl(\hat{v}_i(t)+\varepsilon\bigr)
    \vphantom{
    \frac{1}{2}
    \sum_{i=1}^{n}
    \sum_{\substack{j=1 \\ j\neq i}}^{n}
    \frac{
        \hat{c}_{ij}(t)^{2}
    }{
        \bigl(\hat{v}_i(t)+\varepsilon\bigr)
        \bigl(\hat{v}_j(t)+\varepsilon\bigr)
    }}
    }_{\mathclap{\text{\scriptsize variance expansion}}}
    +
    \underbrace{
    \frac{1}{2}
    \sum_{i=1}^{n}
    \sum_{\substack{j=1 \\ j\neq i}}^{n}
    \frac{
        \hat{c}_{ij}(t)^{2}
    }{
        \bigl(\hat{v}_i(t)+\varepsilon\bigr)
        \bigl(\hat{v}_j(t)+\varepsilon\bigr)
    }
    }_{\mathclap{\text{\scriptsize normalized covariance penalty}}}
    + R_2(t),
    \label{eq:DecorrelationObjectiveBSS}
\end{align}
where \(R_2(t)\) denotes the Taylor remainder, \(\hat v_i(t):=[\hat{\mD}^{\lambda}(t)]_{ii}=[\hat{\mC}^{\lambda}(t)]_{ii}\) denotes the \(i\)-th output variance, and \(\hat c_{ij}(t):=[\hat{\mO}^{\lambda}(t)]_{ij}=[\hat{\mC}^{\lambda}(t)]_{ij}\) for \(i\neq j\) denotes the corresponding cross-covariance. Up to the additive constant in the entropy definition, maximizing the output correlative entropy is therefore equivalent to minimizing the right-hand side of Eq.~\ref{eq:DecorrelationObjectiveBSS}. The resulting surrogate decomposes this objective into two interpretable terms: (i) a variance-expansion term and (ii) a variance-normalized cross-covariance penalty. Importantly, it does not target exact decorrelation, as uncorrelated $\vy(t)$ do not necessarily optimize this objective. Instead, it favors outputs with large component-wise variance and sufficiently small \emph{normalized} cross-covariances, subject to the geometric constraint \(\mathcal P\). This is precisely the regime relevant to structured source recovery: successful separation does not require zero cross-covariances, but rather nondegenerate spread within the source domain together with controlled inter-component dependence. This also explains why the method can recover correlated sources beyond the scope of ICA-based approaches built on mutual independence, as demonstrated by our numerical results in Section \ref{sec:NumericalExperiments}.

\noindent\textbf{Bound for the surrogate approximation error.}
The error $|R_2(t)|$ resulting from the above approximation can be quantified directly. Let \(\hat{\mD}^{\lambda,\varepsilon}(t):=\hat{\mD}^{\lambda}(t)+\varepsilon\mI\), and define \(\hat{\mB}^{\lambda,\varepsilon}(t):=(\hat{\mD}^{\lambda,\varepsilon}(t))^{-1/2}\hat{\mO}^{\lambda}(t)(\hat{\mD}^{\lambda,\varepsilon}(t))^{-1/2}\).
Corollary~\ref{corr:OnlineErrorBound} in Appendix~\ref{sec:TaylorRemainderAnalysis} shows that the Taylor remainder \(R_2(t)\) in Eq.~\ref{eq:DecorrelationObjectiveBSS} satisfies
\begin{align}
|R_2(t)|
& \le
\frac{
\|\hat{\mB}^{\lambda,\varepsilon}(t)\|_F^2
\|\hat{\mB}^{\lambda,\varepsilon}(t)\|_2
}{
3\bigl(1+\lambda_{\min}(\hat{\mB}^{\lambda,\varepsilon}(t))\bigr)
}.
\label{eq:TaylorBound_main}
\end{align}
Here, \(\lambda_{\min}(\cdot)\) denotes the smallest eigenvalue, \(\|\cdot\|_2\) denotes the spectral norm, and \(\|\cdot\|_F\) denotes the Frobenius norm. The numerator \(\|\hat{\mB}^{\lambda,\varepsilon}(t)\|_F^2\|\hat{\mB}^{\lambda,\varepsilon}(t)\|_2\) measures the overall size of the regularized normalized off-diagonal covariance: the Frobenius norm captures its aggregate magnitude, while the spectral norm captures its largest absolute eigenvalue. The denominator \(1+\lambda_{\min}(\hat{\mB}^{\lambda,\varepsilon}(t))\) shows that the bound worsens as the smallest eigenvalue approaches \(-1\); hence the surrogate is accurate when the normalized off-diagonal covariance is small and the spectrum stays away from \(-1\). In Section~\ref{sec:NumericalExperiments}, Figure~\ref{fig:Taylor_Analysis_Panel} reports a diagnostic of the resulting error bound as source correlation varies.

\textbf{Two-timescale optimization.}
Since biological neural networks can only learn online, we wish to generate the output \(\vy\) and then update the parameters \(\mW\) based on the currently provided input \(\vx\).
Therefore at each step or sample $t$, first the output \(\vy\) is optimized, and then the weights \(\mW\) are updated. 
To enable this bi-level optimization, we replace the hard relation \(\vy(t)=\mW\vx(t)\) with a quadratic penalty, adding to the online objective of Eq.~\ref{eq:DecorrelationObjectiveBSS}: 
\begin{equation}
    \mathcal{J}_t
    =
    - \sum_{i=1}^{n} \log \bigl(\hat{v}_i(t)+\varepsilon\bigr)
    +
    \frac{1}{2}
    \sum_{i=1}^{n}
    \sum_{\substack{j=1 \\ j\neq i}}^{n}
    \frac{\hat{c}_{ij}(t)^{2}}
         {\bigl(\hat{v}_i(t)+\varepsilon\bigr)\bigl(\hat{v}_j(t)+\varepsilon\bigr)}
    +
    \gamma \|\vy(t)-\mW\vx(t)\|_2^2,
    \label{eq:full_cost_J}
\end{equation}
and \(\gamma>0\) controls the strength of the prediction-error term.

To optimize Eq.~\ref{eq:full_cost_J}, we adopt a two-timescale procedure (which naturally admits a neural interpretation as we will show later). For each incoming sample, the current output \(\vy(t)\) is first updated on a fast timescale.
During this fast relaxation, the input \(\vx(t)\), the weights \(\mW\), and the running statistics are held fixed, and the output \(\vy(t)\) is updated until the value minimizing \(\mathcal J_t\) over \(\vy(t)\in\mathcal P\) is found.
Once this fast inference stage has settled, the resulting output state is used to update the feedforward weights and the running second-order statistics on a slower timescale.

\noindent\textbf{Inference of output activity.}
We first introduce notation that will help us to describe the inference process.
For each streaming sample \(t\), we introduce a faster neural-relaxation time index
\(\tau=0,1,\ldots,\tau_{\max}\). The index \(t\) labels the slow arrival of new input samples and the corresponding parameter updates, whereas \(\tau\) labels the fast within-sample activity dynamics used to infer \(\vy(t)\). 
For notational convenience, let \(\hat\mu_k(t):=[\hat{\bm{\mu}}^\lambda(t)]_k\) denote the \(k\)-th component of the running mean and we define the centered activity as
\(
\bar y_k(t, \tau)=y_k(t, \tau)-\hat\mu_k(t).
\)

During inference we change the output $\vy$ to reduce the objective of Eq.~\ref{eq:full_cost_J}.
Using the gradient of Eq.~\ref{eq:full_cost_J}, we define the fast activity-update direction as the negative truncated gradient
\(\vd(t,\tau)\approx-{\nabla}_{\vy}\mathcal J_t\), where the truncation discards the higher-order term quadratic in the off-diagonal covariance entries for biological plausibility; see Appendix~\ref{Appendix:OutputGradDerivation}. Componentwise,
\begin{equation}
    d_k(t, \tau)
 =
\underbrace{
    \frac{\bar y_k(t, \tau)}{\hat{v}_k(t)+\varepsilon}
    \vphantom{
    \displaystyle
    \sum_{\substack{j=1\\j\neq k}}^{n}
    \frac{\hat{c}_{kj}(t)\,\bar y_j(t, \tau)}
         {\bigl(\hat{v}_k(t)+\varepsilon\bigr)\bigl(\hat{v}_j(t)+\varepsilon\bigr)}
    }
}_{\mathclap{\text{\scriptsize variance drive}}}
-
\underbrace{
    \sum_{\substack{j=1\\j\neq k}}^{n}
    \frac{\hat{c}_{kj}(t)\,\bar y_j(t, \tau)}
         {\bigl(\hat{v}_k(t)+\varepsilon\bigr)\bigl(\hat{v}_j(t)+\varepsilon\bigr)}
}_{\mathclap{\text{\scriptsize covariance reduction}}}
-
\underbrace{
    \gamma
    \left(
        y_k(t, \tau)-\sum_{\ell=1}^{m}W_{k\ell}(t-1)x_\ell(t)
    \right)
}_{\mathclap{\text{\scriptsize predictive correction}}}.
\label{eq:output_inference_direction}
\end{equation}
 The corresponding projected activity update is
\begin{equation}
    y_k(t, \tau+1)
    =
    \sigma_{\mathcal P}\!\left(
        y_k(t, \tau)+\eta_y(t,\tau)\,d_k(t, \tau)
    \right),
    \label{eq:output_inference_update}
\end{equation}
where \(\eta_y(t,\tau)>0\) is the inference step size.
Terms in Eq.~\ref{eq:output_inference_direction} have intuitive interpretation.
The first term in \(d_k(t, \tau)\) promotes regularized variance expansion, the second reduces covariance,
and the third pulls the output activity toward its feedforward prediction.

The form of \(\sigma_{\mathcal P}\) in Eq.~\ref{eq:output_inference_update} is determined by the source domain. For box-type domains such as \(\gB_\mathrm{max}\) and \(\gB_{\mathrm{max},+}\), it reduces to elementwise saturation, namely clipping to \([-1,1]\) or \([0,1]\). For \(\gB_1\), \(\gB_{1,+}\), and \(\Delta\), we instead project the output on the source domain
as in standard proximal algorithms \citep{parikh2014proximal}, and the corresponding domain-specific derivations are given in Appendix~\ref{Appendix:NetworkDynamicsIllustrations}.

\noindent\textbf{Parameter updates.}
Once the fast inference dynamics have settled for sample \(t\), the resulting output state is used to update the weights $\mW$ and the running output statistics on a slower timescale. 
With the settled output \(\vy(t)\) held fixed, the only \(\mW\)-dependent part of \(\mathcal J_t\) is the quadratic prediction loss. A gradient descent step on \(\frac{1}{2}\|\vy(t)-\mW\vx(t)\|_2^2\) with learning rate $\alpha_W(t)$ gives
\begin{equation}
    \mW(t)=\mW(t - 1)+\alpha_W(t)\ve(t)\vx(t)^\top,
    \qquad
    \ve(t):=\vy(t)-\mW(t - 1)\vx(t),
    \label{eq:updateW}
\end{equation}
as derived in Appendix~\ref{appendix:UpdateRuleDerivations} (Eq.~\ref{eq:FeedforwardUpdateDerivation}). The running mean is updated according to Eq.~\ref{eq:mean_update}, and the centered activity \(\bar{\vy}(t)=\vy(t)-\hat{\bm{\mu}}^\lambda(t)\) is then used to update the cross-covariance and variance traces:\looseness=-1
\begin{equation}
    \hat{c}_{ij}(t)=\lambda\hat{c}_{ij}(t - 1)+(1-\lambda)\bar y_i(t)\bar y_j(t),\enspace i\neq j,
    \label{eq:updatec}
\end{equation}
\begin{equation}
  \hat{v}_i(t)=\lambda\hat{v}_i(t - 1)+(1-\lambda)\bar y_i(t)^2,\enspace i=1,\ldots,n,
  \label{eq:updatev}
\end{equation}
as derived in Appendix~\ref{appendix:UpdateRuleDerivations} (see Eq.s~\ref{eq:DiagonalVarianceUpdate} and \ref{eq:OffDiagonalCovarianceUpdate}). 
We refer to the resulting procedure as \textit{Predictive Entropy Maximization} (\textit{PEM}) and summarize it in Algorithm~\ref{alg:predictivedecorr} in Appendix~\ref{Appendix:NetworkDynamicsIllustrations}.


\noindent\textbf{Neural implementation.}
PEM can be naturally implemented in a neural network with local weight update rules. In this implementation, we assume that output signals $y_i$ are encoded in neural activity with dynamics on a fast time scale, while parameters $W_{ij}$ and $\hat{c}_{ij}$ are encoded in the weights of connections between neurons which evolve on a slower time scale.

\begin{figure*}[t]
\centering

\tikzset{
    inputdot/.style={circle, fill=red!75!black, inner sep=1.4pt},
    errnode/.style={
        rectangle,
        draw=black!60,
        fill=black!6,
        minimum width=8mm,
        minimum height=5.5mm,
        rounded corners=1pt,
        inner sep=1pt
    },
    ynode/.style={
        circle,
        draw=blue!75!black,
        fill=blue!3,
        minimum size=9.4mm,
        inner sep=0pt
    },
    inhibnode/.style={
        circle,
        draw=red!85!black,
        fill=red!8,
        minimum size=8.5mm,
        inner sep=0pt
    },
    ff/.style={->, semithick, black!75},
    lateral/.style={<->, semithick, red!85!black},
    excite/.style={->, semithick, black!75},
    inhibbar/.style={->, semithick, red!85!black},
    ebrace/.style={
        decorate,
        decoration={brace, amplitude=3pt, raise=0.5pt},
        draw=black!80,
        line width=0.45pt
    }
}

\begin{subfigure}[t]{0.47\textwidth}
\centering
\begin{tikzpicture}[>=Stealth, every node/.style={font=\small}, scale=0.92]

\node[left] at (1.10,  2.5) {$x_1$};
\node[inputdot] (x1) at (1.45,  2.5) {};

\node[left] at (1.10,  0.9) {$x_2$};
\node[inputdot] (x2) at (1.45,  0.9) {};

\node at (1.43, -0.10) {$\vdots$};

\node[left] at (1.10, -1.55) {$x_m$};
\node[inputdot] (xm) at (1.45, -1.55) {};

\node[errnode] (d1) at (4.35,  2.5) {};
\node[errnode] (d2) at (4.35,  0.9) {};
\node at (4.85, -0.10) {$\vdots$};
\node[errnode] (dn) at (4.35, -1.55) {};

\node[ynode] (y1) at (5.18,  2.5) {};
\node[ynode] (y2) at (5.18,  0.9) {};
\node[ynode] (yn) at (5.18, -1.55) {};

\draw[ebrace] ($(d1.center)+(0.18,0.54)$) -- ($(y1.center)+(-0.18,0.54)$)
    node[midway, yshift=7pt] {$e_1$};

\draw[ebrace] ($(d2.center)+(0.18,0.54)$) -- ($(y2.center)+(-0.18,0.54)$)
    node[midway, yshift=7pt] {$e_2$};

\draw[ebrace] ($(dn.center)+(0.18,0.54)$) -- ($(yn.center)+(-0.18,0.54)$)
    node[midway, yshift=7pt] {$e_n$};

\foreach \Y in {y1,y2,yn}{
    \begin{scope}
        \clip (\Y.center) circle (0.43cm);

        \draw[black!55, line width=0.20pt]
            ($(\Y.center)+(-0.20,0)$) -- ($(\Y.center)+(0.20,0)$);
        \draw[black!55, line width=0.20pt]
            ($(\Y.center)+(0,-0.20)$) -- ($(\Y.center)+(0,0.20)$);

        \draw[blue!90!black, line width=0.52pt]
            ($(\Y.center)+(-0.19,-0.10)$) --
            ($(\Y.center)+(-0.07,-0.10)$) --
            ($(\Y.center)+( 0.07, 0.10)$) --
            ($(\Y.center)+( 0.19, 0.10)$);
    \end{scope}
}

\foreach \src in {x1,x2,xm}{
    \foreach \dst in {d1,d2,dn}{
        \draw[ff] (\src.east) -- (\dst.west);
    }
}

\draw[lateral] (y1.east) to[bend left=16] (y2.east);
\draw[lateral] (y2.east) to[bend left=16] (yn.east);
\draw[lateral] (y1.east) to[bend left=32] (yn.east);

\node[above=-0.5mm, xshift=1.3mm] at (y1.north) {$y_1$};
\node[above=-0.5mm, xshift=1.3mm] at (y2.north) {$y_2$};
\node[above=-0.5mm, xshift=1.3mm] at (yn.north) {$y_n$};

\node[font=\Large] at (2.95, 3.25) {$\mW$};
\node[font=\Large, text=red!85!black] at (6.80, 0.55) {$\hat{\mC}^{\lambda}$};

\end{tikzpicture}
\caption{Antisparse network}
\label{fig:antisparse_predictivedecorr_architecture}
\end{subfigure}
\hfill
\begin{subfigure}[t]{0.47\textwidth}
\centering
\begin{tikzpicture}[>=Stealth, every node/.style={font=\small}, scale=0.92]

\node[left] at (1.10,  2.5) {$x_1$};
\node[inputdot] (x1) at (1.45,  2.5) {};

\node[left] at (1.10,  0.9) {$x_2$};
\node[inputdot] (x2) at (1.45,  0.9) {};

\node at (1.43, -0.10) {$\vdots$};

\node[left] at (1.10, -1.55) {$x_m$};
\node[inputdot] (xm) at (1.45, -1.55) {};

\node[errnode] (d1) at (4.35,  2.5) {};
\node[errnode] (d2) at (4.35,  0.9) {};
\node at (4.85, -0.10) {$\vdots$};
\node[errnode] (dn) at (4.35, -1.55) {};

\node[ynode] (y1) at (5.18,  2.5) {};
\node[ynode] (y2) at (5.18,  0.9) {};
\node[ynode] (yn) at (5.18, -1.55) {};

\draw[ebrace] ($(d1.center)+(0.18,0.54)$) -- ($(y1.center)+(-0.18,0.54)$)
    node[midway, yshift=7pt] {$e_1$};

\draw[ebrace] ($(d2.center)+(0.18,0.54)$) -- ($(y2.center)+(-0.18,0.54)$)
    node[midway, yshift=7pt] {$e_2$};

\draw[ebrace] ($(dn.center)+(0.18,0.54)$) -- ($(yn.center)+(-0.18,0.54)$)
    node[midway, yshift=7pt] {$e_n$};

\foreach \Y in {y1,y2,yn}{
    \begin{scope}
        \clip (\Y.center) circle (0.43cm);

        \draw[black!55, line width=0.20pt]
            ($(\Y.center)+(-0.20,0)$) -- ($(\Y.center)+(0.20,0)$);
        \draw[black!55, line width=0.20pt]
            ($(\Y.center)+(0,-0.20)$) -- ($(\Y.center)+(0,0.20)$);

        \draw[blue!90!black, line width=0.52pt]
            ($(\Y.center)+(-0.19,-0.13)$) --
            ($(\Y.center)+(-0.08, 0.00)$) --
            ($(\Y.center)+( 0.08, 0.00)$) --
            ($(\Y.center)+( 0.19, 0.13)$);
    \end{scope}
}

\foreach \src in {x1,x2,xm}{
    \foreach \dst in {d1,d2,dn}{
        \draw[ff] (\src.east) -- (\dst.west);
    }
}

\draw[lateral] (y1.east) to[bend left=16] (y2.east);
\draw[lateral] (y2.east) to[bend left=16] (yn.east);
\draw[lateral] (y1.east) to[bend left=32] (yn.east);

\node[inhibnode] (lam) at (7.10, 0.45) {$\lambda_L$};

\draw[excite] (y1.east) to[bend left=10] (lam.north west);
\draw[excite] (y2.north east) to[bend left=7] (lam.west);
\draw[excite] (yn.east) to[bend right=10] (lam.south west);

\draw[inhibbar] (lam.north west) to[bend left=22] (y1.east);
\draw[inhibbar] (lam.west) to[bend left=18] (y2.east);
\draw[inhibbar] (lam.south west) to[bend right=22] (yn.east);

\node[above=-0.5mm, xshift=1.3mm] at (y1.north) {$y_1$};
\node[above=-0.5mm, xshift=1.3mm] at (y2.north) {$y_2$};
\node[above=-0.5mm, xshift=1.3mm] at (yn.north) {$y_n$};

\node[font=\Large] at (2.95, 3.25) {$\mW$};
\node[font=\Large, text=red!85!black] at (6.75, 2.25) {$\hat{\mC}^{\lambda}$};

\end{tikzpicture}
\caption{Sparse network}
\label{fig:sparse_predictivedecorr_architecture}
\end{subfigure}

\caption{\textbf{Representative Predictive Entropy Maximization architectures for two source domains.}
\textbf{(a)} Antisparse architecture. Mixture inputs are mapped through feedforward weights \(\mW\) to local prediction
compartments, each paired with an output unit \(y_k\). 
Prediction errors  \(e_k\) are computed as the differences between somatic and dendritic activity.
The output layer is coupled through adaptive recurrent inhibitory interactions driven by the running output-covariance statistics \(\hat{\mC}^\lambda\), while the antisparse constraint is enforced locally by clipping.
\textbf{(b)} Sparse architecture. The same predictive feedforward pathway and adaptive recurrent inhibitory interactions are retained, but an additional shared inhibitory unit \(\lambda_L\) sends a common suppressive signal to all outputs. This shared inhibition enforces the sparsity constraint and yields soft-thresholding dynamics.}
\label{fig:predictivedecorr_two_architectures}
\end{figure*}
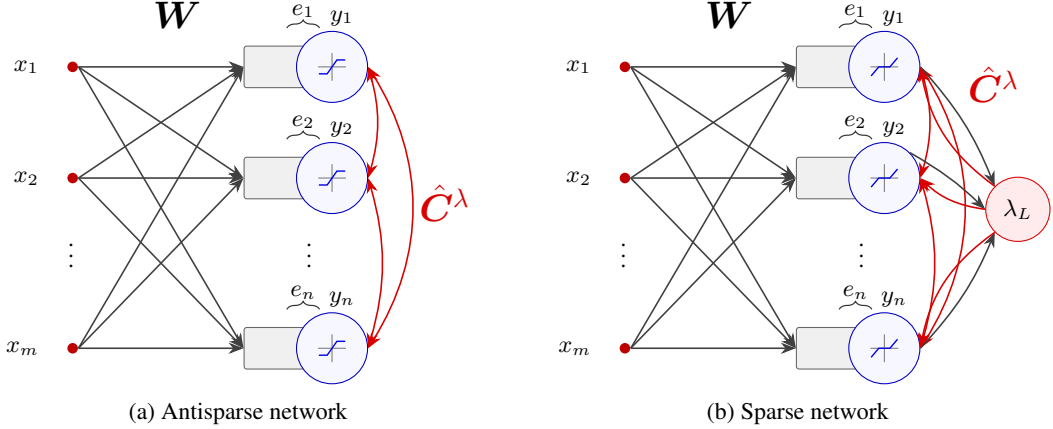

PEM naturally maps onto a two-layer network, where the input layer encodes the mixture inputs \(\vx\), while the output layer encodes the outputs $\vy$. To make the network mapping explicit, we rewrite the activity-update direction from Eq.~\ref{eq:output_inference_direction} as
\begin{equation}
d_k(t,\tau)
=
\underbrace{
\gamma
\sum_{\ell=1}^m W_{k\ell}(t-1)x_\ell(t)
}_{\mathclap{\text{\scriptsize feedforward input}}}
-
\underbrace{
\sum_{\substack{j=1\\j\neq k}}^n
\frac{\hat c_{kj}(t)\bar y_j(t,\tau)}
{(\hat v_k(t)+\varepsilon)(\hat v_j(t)+\varepsilon)}
}_{\mathclap{\text{\scriptsize recurrent input}}}
-
\underbrace{
\gamma
y_k(t,\tau)
+
\frac{\bar y_k(t,\tau)}{\hat v_k(t)+\varepsilon}
}_{\mathclap{\text{\scriptsize leak}}}.
\label{eq:output_inference_direction_grouped}
\end{equation}


Eq.~\ref{eq:output_inference_direction_grouped} describes changes in the output signals, which correspond to changes in the output-layer activity in the network in Figure~\ref{fig:antisparse_predictivedecorr_architecture}. The first term 
is the feedforward input to output unit \(k\) through weights \(W_{k\ell}\).
The second term maps to lateral recurrent inhibition via connections with weights $\hat{c}_{kj}$.
The remaining terms depend only on neuron \(k\)'s own activity (but not on any other neurons in the network), so they can be interpreted as a leak or decay of activity (dependent on parameters $\hat{\mu}_k$ and $\hat{v}_k$ that are intrinsic to the neuron). 

Assuming domains with limited total activity of output neurons (\(\gB_1\), \(\gB_{1,+}\), and \(\Delta\)) requires normalization of activity that could be driven by an additional neural population \(\lambda_L\), which acts as a shared inhibitory signal enforcing the corresponding population-level constraint (as shown in Figure~\ref{fig:predictivedecorr_two_architectures}b).

The parameter updates in Eqs. \ref{eq:updateW}-\ref{eq:updatev} are \textit{local}. The feedforward rule of Eq. \ref{eq:updateW} depends only on the presynaptic activity and the postsynaptic prediction error, and corresponds to a previously proposed plasticity rule \cite{urbanczik2014learning}. 
It has been postulated that biological neurons can implement such rule: If input $\mW(t - 1) \vx(t)$ comes to dendritic compartments (denoted by grey rectangles in Figure \ref{fig:predictivedecorr_two_architectures}a), then the neurons can compute errors $\ve(t)$ as the differences between their somatic and dendritic activity, and trigger plasticity proportional to these errors \cite{urbanczik2014learning}.

The change in recurrent connections in Eq. \ref{eq:updatec} is proportional to the product of normalized activity of pre and post-synaptic neurons, hence it corresponds to Hebbian plasticity.
This yields a more directly synapse-local interpretation of the recurrent updates compared to CorInfoMax \citep{bozkurt2023correlative} (where it relies on transformed inverse-covariance signals). 

The variance trace \(\hat v_i(t)\) determines the leak of neuron \(i\) and its change in Eq. \ref{eq:updatev} depends only on the centered activity of neuron \(i\).
The updates of the variance trace could correspond to 
homeostatic regulation \citep{Turrigiano2012}, that seeks to maintain average neural activity in a desired range.

\noindent\textbf{Unnormalized Predictive Entropy Maximization.}
A potential challenge for the neural implementation of Eq.~\ref{eq:output_inference_direction_grouped} is the dependence of recurrent inhibition between neurons $k$ and $j$ on $\hat{v}_k$ and $\hat{v}_j$, because we assumed above that these parameters determine neurons' leak, hence it is unclear how they could affect synaptic transmission specifically via recurrent connections.
Therefore, in our subsequent experiments, we also consider a simplified and arguably even more biologically plausible variant of \textit{PEM} called \textit{unnormalized PEM} (\textit{u-PEM}). This model is obtained by replacing the variance normalization \(
((\hat{v}_k(t)+\varepsilon)(\hat{v}_j(t)+\varepsilon))^{-1}\) from the `covariance reduction' term (Eq. \ref{eq:output_inference_direction}) with a parameter \(\gamma_{lateral}>0\) that controls the strength of the lateral coupling.
Hence, this gives the following cost function and the activity-update direction:
\begin{align*}
    &\mathcal{J}_t
    =
    - \sum_{i=1}^{n} \log (\hat{v}_i(t)+\varepsilon)
    +
    \frac{1}{2}
    \gamma_{lateral}
    \sum_{i=1}^{n}
    \sum_{\substack{j=1 \\ j\neq i}}^{n}
        \hat{c}_{ij}(t)^{2}
    +
    \gamma \|\vy(t)-\mW(t-1)\vx(t)\|_2^2,\\
    &d_k(t,\tau)
    =
    \frac{\bar y_k(t,\tau)}{\hat{v}_k(t)+\varepsilon}
    -
    \gamma_{lateral}\sum_{\substack{j=1\\j\neq k}}^{n}
    \hat{c}_{kj}(t)\,\bar y_j(t,\tau)
    -
    \gamma \left(y_k(t, \tau)-\sum_{\ell=1}^{m}W_{k\ell}(t - 1)x_\ell(t)\right).
\end{align*}
\textit{u-PEM} maps onto the same neural architecture as before, but now the lateral inhibition is directly implemented via the product of the weights and the neural activity, without needing further normalizing parameters. 
Further details are given in Appendix \ref{sec:SimplifiedPEM_Appendix}.

\section{Numerical results}
\label{sec:NumericalExperiments}

We compare \textit{PEM} and its unnormalized variant (\textit{u-PEM}) with both batch and online baselines. The batch baselines are \textit{CorInfoMax (Batch)} \citep{Erdogan_CorInfoMax} and \textit{ICA-InfoMax} \citep{Bell_ICA_InfoMax}; among them, \textit{CorInfoMax (Batch)} serves as the oracle reference, since it optimizes the exact determinant-maximization objective using the full dataset. The bio-plausible online baselines are \textit{CorInfoMax (Online)} \citep{bozkurt2023correlative} and \textit{Nonnegative Similarity Matching (Online)} \citep{Pehlevan_NSM_2017}. Among these, \textit{CorInfoMax (Online)} is the most directly comparable method, as it targets the exact log-determinant objective that our Taylor surrogate approximates. Our aim is therefore to match the robustness of exact determinant-based correlated-source separation while obtaining a more directly local and biologically plausible implementation. Performance is measured by the mean component signal-to-noise ratio (mSNR) defined in Eq.~\ref{eq:SNR_Mean_Definition}. Unless stated otherwise, all results are averaged over \(30\) independent realizations, and shaded bands denote \(95\%\) confidence intervals computed using Eq.~\ref{eq:CI_Definition}. Our Python implementation is provided with the supplementary material. Additional implementation details, extended results, and supplementary experiments, including comparisons with additional baselines, are reported in Appendix~\ref{appendix:Supplementary_Numerical_Experiments}. \looseness=-1

\noindent\textbf{Synthetic source separation.} We first consider synthetic linear mixtures with \(n=5\) sources, \(m=10\) mixtures, and \(T=10^5\) samples. For each realization, the mixing matrix \(\mA \in \mathbb{R}^{10\times 5}\) has i.i.d. standard normal entries, and the observations are generated as
\(
\vx(t)=\mA\vs(t)+\bm{\epsilon}(t),
\enspace
\bm{\epsilon}(t)\sim\mathcal{N}(\bm{0},\sigma^2\mI_{10}),
\enspace
t=1,\dots,T,
\)
with \(\sigma^2\) chosen to match the prescribed input SNR level \(\mathrm{SNR}_{\mathrm{in}}\).

We evaluate the method in two complementary regimes. In the first, we vary the source correlation level while keeping the observation noise fixed. In the second, we vary the observation noise while keeping the source-domain geometry fixed.

\begin{figure}[t!]
    \centering
    \begin{subfigure}{0.48\linewidth}
        \centering
        \includegraphics[width=\linewidth]{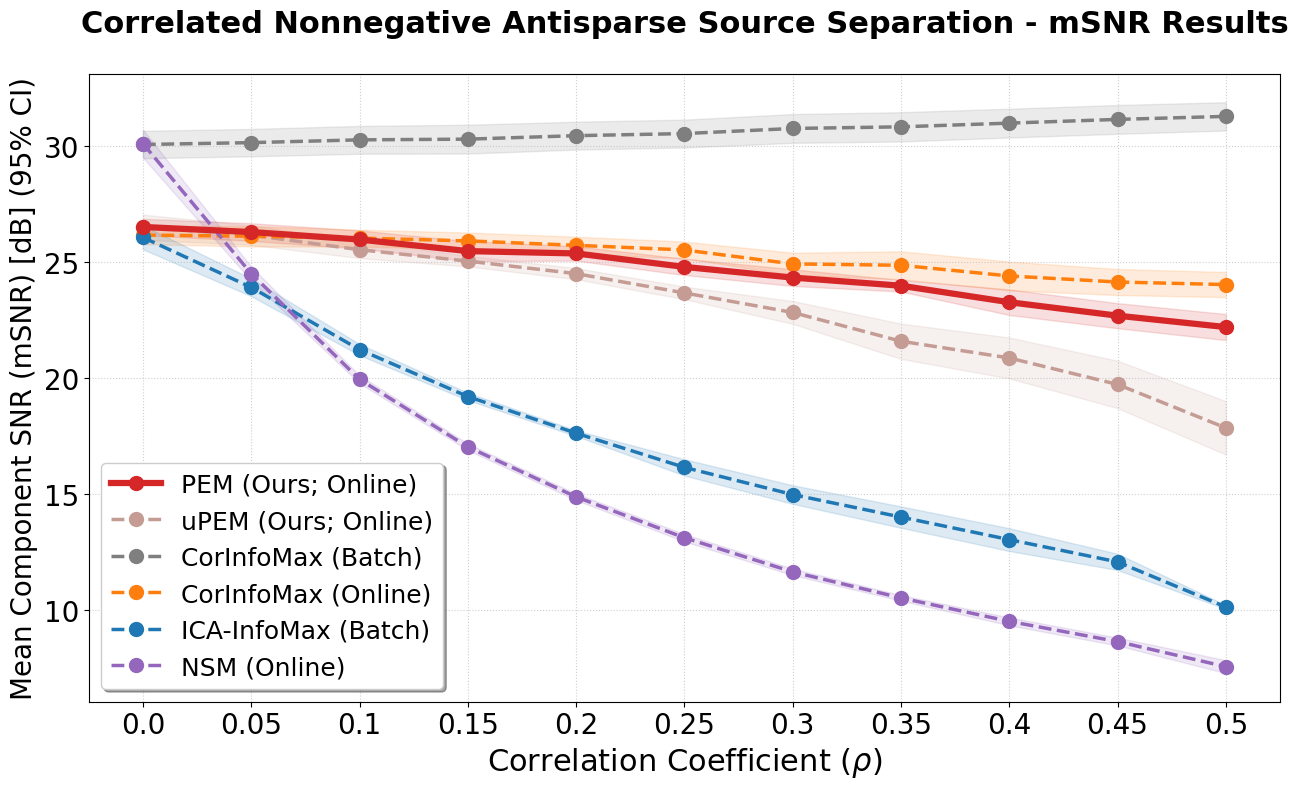}
        \caption{Correlated nonnegative antisparse (\(\gB_{\mathrm{max},+}\))}
        \label{fig:nn_antisparse_corr_main}
    \end{subfigure}
    \hfill
    \begin{subfigure}{0.48\linewidth}
        \centering
        \includegraphics[width=\linewidth]{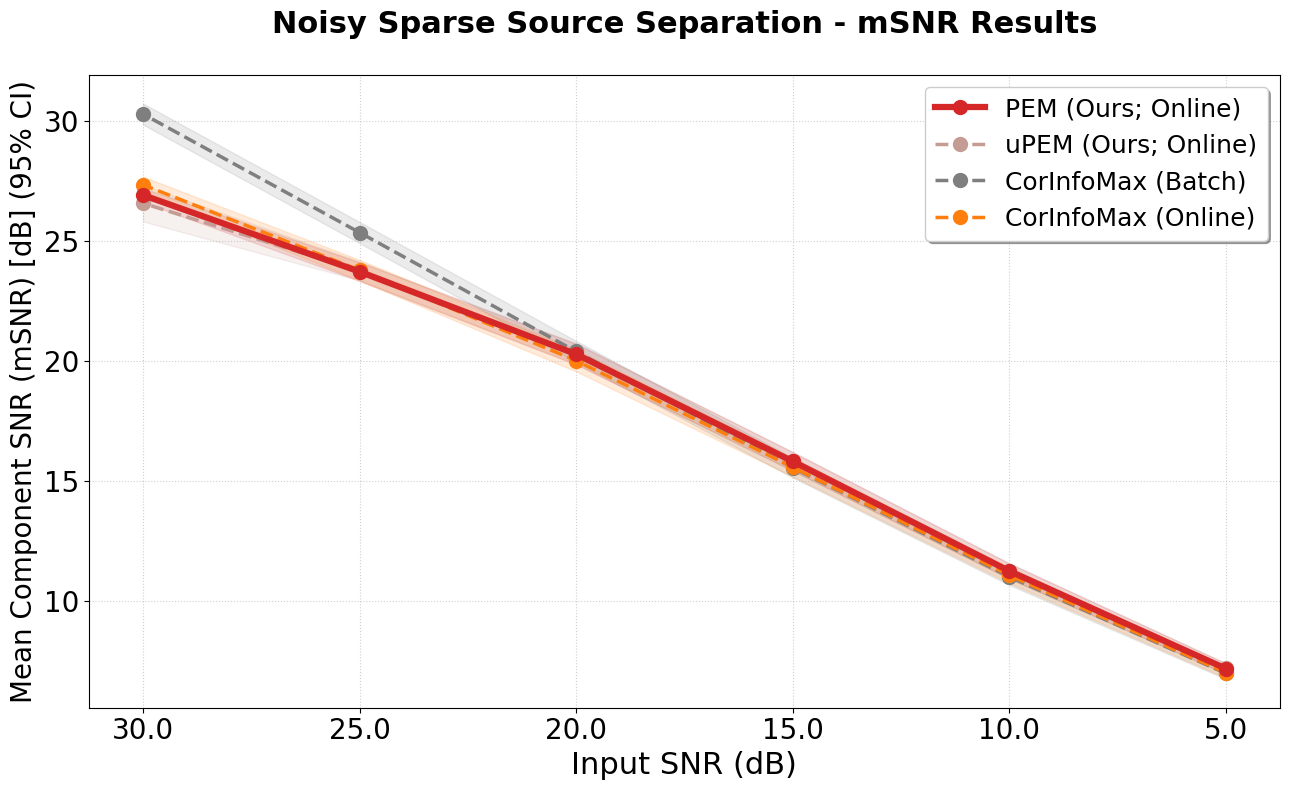}
        \caption{Noisy sparse (\(\gB_1\))}
        \label{fig:sparse_noisy_main}
    \end{subfigure}
    
    \caption{\textbf{Performance comparisons across source domains.}
    \textbf{(a)} Mean component SNR versus correlation \(\rho\) for nonnegative antisparse sources. \textit{Predictive Entropy Maximization} (\textit{PEM}) and its unnormalized variant (\textit{u-PEM}) remain robust and outperform online baselines relying on stronger independence or decorrelation assumptions.
    \textbf{(b)} Mean component SNR versus input SNR for sparse sources. The \textit{PEM} models stay close to the batch \textit{CorInfoMax (Batch)} baseline across noise levels while remaining online and biologically plausible. 
    }
    \label{fig:numerical_experiments_main}
    \vskip -0.2in
\end{figure}

For the correlation experiment, we consider the nonnegative antisparse domain \(\gB_{\mathrm{max},+}\). Sources are generated using a copula-\(t\) model, a standard way to generate correlated variables with uniform marginals; here \(\rho\) controls the dependence level and is varied over \(\{0,0.05,\dots,0.5\}\), while the mixture SNR is fixed at \(30\) dB. The results are shown in Figure \ref{fig:nn_antisparse_corr_main}. Both \textit{PEM} and its unnormalized variant (\textit{u-PEM}) remain robust as \(\rho\) increases and clearly outperform \textit{ICA-InfoMax} and \textit{NSM}, whose performance degrades rapidly as source correlation grows. \textit{PEM} appears more resistant to increased source correlation than \textit{u-PEM}, suggesting that the variance normalization inherited from the entropy objective contributes to robustness under correlated sources.

Relative to \textit{CorInfoMax}, the \textit{PEM} models preserve robust correlated-source separation while replacing inverse-covariance dynamics by the simpler surrogate-based updates introduced in Section \ref{sec:DecorrelationObjective}. As expected, the batch \textit{CorInfoMax (Batch)} baseline achieves the strongest overall performance.

For the noise experiment, we consider the sparse domain \(\gB_1\). Sources are generated uniformly from \(\gB_1\), and the mixture SNR is varied over \(\mathrm{SNR}_{\mathrm{in}} \in \{30,25,\dots,5\}\) dB. Figure \ref{fig:sparse_noisy_main} shows that both \textit{PEM} and \textit{u-PEM} track the batch \textit{CorInfoMax (Batch)} baseline closely across the full noise range. This indicates that replacing inverse-covariance dynamics by the Taylor surrogate does not materially degrade source separation performance, while still yielding an online and bio-plausible neural implementation.

Additional results for the antisparse, nonnegative sparse, and simplex domains follow the same pattern and are reported in Appendix Figure \ref{fig:numerical_experiments_appendix}.

\begin{wrapfigure}{t!}{0.42\textwidth}
    \vspace{-1.1em}
    \centering
    \includegraphics[width=\linewidth]{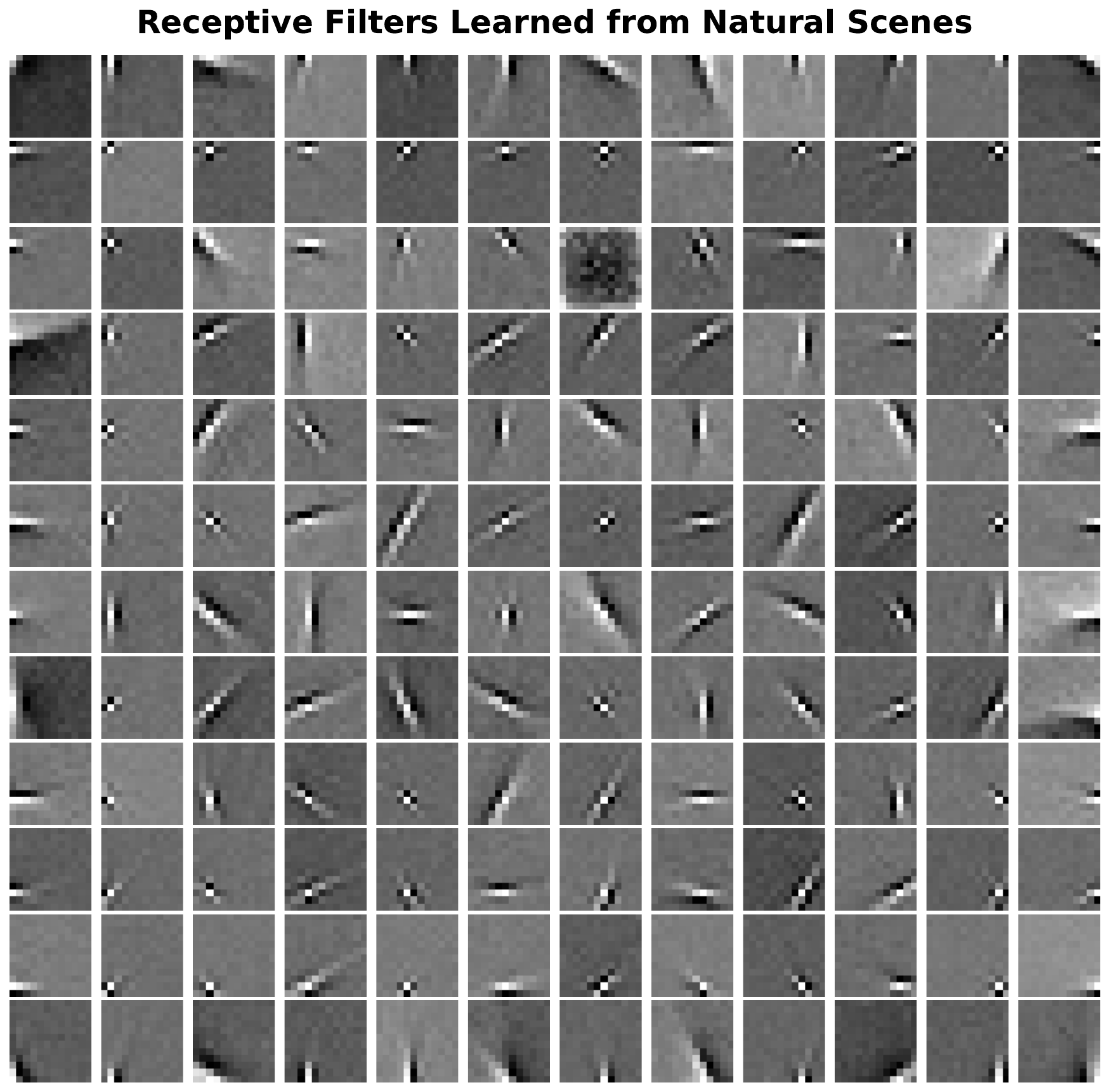}
    \caption{\textit{ Receptive fields} learned by the sparse \textit{Predictive Entropy Maximization} model from natural image patches exhibit localized and oriented structure characteristic of sparse sensory representations.\looseness=-1}
    \label{fig:receptive_field}
\end{wrapfigure}

\noindent\textbf{Learning sparse receptive fields.} We evaluate the sparse architecture illustrated in Figure \ref{fig:predictivedecorr_two_architectures}b on pre-whitened \(12\times 12\) natural image patches \citep{olshausen1996emergence}.\footnote{The dataset and original Sparsenet implementation are available at \url{https://www.rctn.org/bruno/sparsenet/}.} The network is trained on vectorized patches in \(\mathbb{R}^{144}\). The learned filters, shown in Figure~\ref{fig:receptive_field}, are localized and oriented, with the familiar Gabor-like structure associated with sparse coding models of the early visual cortex \citep{olshausen1996emergence}. This experiment shows that the same sparse neural mechanism used for source separation also captures statistically meaningful features from natural images.

\noindent\textbf{Auditory source separation.}
We also evaluate the sparse architecture on an audio separation task inspired by the cocktail-party setting \citep{cherry1953some,haykin2005cocktail}. We use three \(5\)-second audio sources from the \textit{librosa} library \citep{librosa_python}, namely \texttt{fishin}, \texttt{pistachio}, and \texttt{vibeace}, and transform the mixtures with a db4 discrete wavelet transform before running \textit{PEM}. This choice is motivated by the approximate sparsity of natural sounds in multiscale representations \citep{smith2006efficient,lewicki2002efficient,daubechies1992ten}; the transform-domain formulation and representative temporal alignments are reported in Appendix~\ref{appendix:Supplementary_Numerical_Experiments}. Across \(30\) random mixing scenarios, the three recovered sources achieve source-wise SNR values of \(24.12 \pm 1.98\) dB, \(25.07 \pm 2.14\) dB, and \(21.54 \pm 2.17\) dB (95\% CI), indicating that the sparse-domain dynamics remain effective on realistic audio mixtures.

\noindent\textbf{Diagnostic validation of the Taylor surrogate.} To complement the analytical results of Section~\ref{sec:DecorrelationObjective}, we empirically assess the accuracy of the second-order Taylor surrogate used to derive the neural dynamics. We consider the antisparse synthetic setting with controlled source correlation \(\rho \in \{0,\dots,0.5\}\), while fixing the input SNR at \(30\) dB.\looseness=-1

To make the approximation nontrivial from the outset, we initialize the covariance statistic with substantial off-diagonal structure by setting
\(
\hat{\mC}^{\lambda}(0)=\mG\mG^\top,
\enspace
\mG=\sqrt{0.2}\,\mI+\mZ,
\enspace
Z_{ij}\sim \mathcal{N}(0,1/5).
\)
At each recorded iteration, we compute both the exact Taylor remainder and the sharper spectral upper bound given by the first inequality in Eq.~\ref{eq:TaylorBound_main}, using the current covariance estimate \(\hat{\mC}^{\lambda}(t)\). Figure~\ref{fig:taylor_scatter} reports all recorded error--bound pairs across runs and correlation levels. The empirical points lie on or above the identity line, confirming that the analytical bound consistently upper-bounds the observed approximation error. Figure~\ref{fig:taylor_wrt_rho} summarizes the converged error and bound as functions of \(\rho\). As expected, the approximation error increases with source correlation, but remains small throughout the range considered and continues to be well controlled by the theoretical bound. Additional transient diagnostics, including time-evolution plots of the Taylor approximation error and its upper bound for \(\rho=0\) and \(\rho=0.4\), are reported in Appendix~\ref{appendix:Supplementary_Numerical_Experiments}.

\begin{figure}[t!]
    \centering
    \begin{subfigure}[b]{0.42\linewidth}
        \centering
        \includegraphics[width=\linewidth]{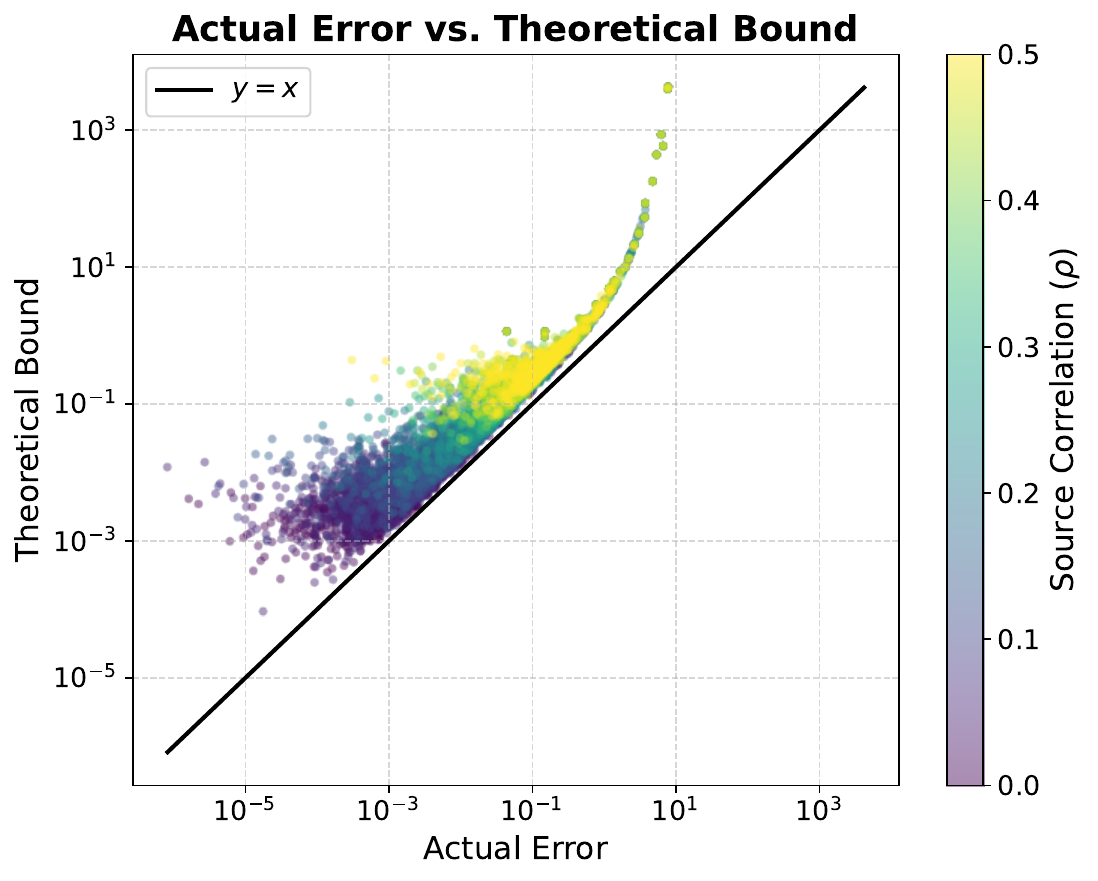}
        \caption{Actual error vs.\ theoretical bound}
        \label{fig:taylor_scatter}
    \end{subfigure}
    \begin{subfigure}[b]{0.55\linewidth}
        \centering
        \includegraphics[width=\linewidth]{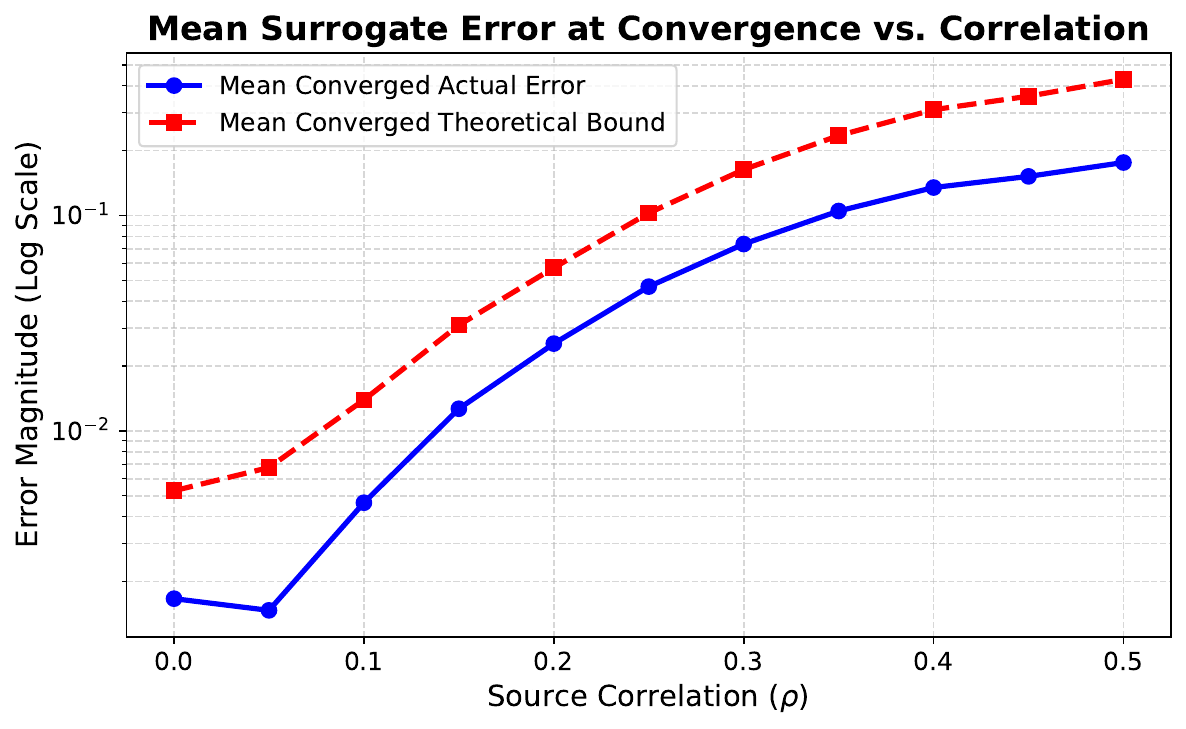}
        \caption{Converged error vs.\ source correlation}
        \label{fig:taylor_wrt_rho}
    \end{subfigure}
    \caption{\textbf{Taylor-surrogate diagnostics.}
\textbf{(a)} Exact Taylor remainder versus the theoretical bound from (Eq.~\ref{eq:TaylorBound_main}), over recorded runs and iterations, color-coded by source correlation \(\rho\). The solid line denotes \(y=x\).
\textbf{(b)} Mean Taylor remainder and corresponding bound (Eq.~\ref{eq:TaylorBound_main}) as functions of \(\rho\).\looseness=-1 }
\label{fig:Taylor_Analysis_Panel}
\end{figure}

\section{Conclusion}
\label{sec:Conclusion}
\noindent\textbf{Summary.}
In this work, we introduced \textit{Predictive Entropy Maximization}, an online determinant-based entropy maximization framework for blind source separation over identifiable source domains. By replacing the exact log-determinant objective with a second-order Taylor surrogate, we obtained a predictive neural architecture in which feedforward plasticity is local and error-driven, while recurrent interactions are mediated by covariance traces and implement adaptive lateral inhibition. We also provided explicit spectral bounds for the surrogate error, clarifying when this approximation is accurate. Empirically, the resulting method remains robust under increasing source correlation and observation noise, compares favorably with existing biologically plausible baselines, and extends naturally to sparse receptive-field learning and auditory source separation. Taken together, these results provide a normative account of how structured source separation can give rise to simple local plasticity rules and interpretable inhibitory dynamics in recurrent networks.

\noindent\textbf{Limitations and future directions.} The performance of the proposed model can be sensitive to hyperparameter choices and to the number of observed mixtures, as analyzed empirically in Appendix \ref{appendix:simulation_hyperparameters}-\ref{appendix:ablation_studies}. In addition, similar to other biologically plausible recurrent BSS networks, the computational cost of our method scales with the number of fast neural-dynamics iterations (see Appendix \ref{appendix:Compute_Complexity}), motivating future work on faster inference schemes.

\section*{Acknowledgements}
B.B. is supported by the Gatsby Charitable Foundation. This work was supported by the Wellcome Trust grant 313955/Z/24/Z and the Medical Research Council grant UKRI/MR/B000936/1. Part of this work was completed while B.B. was a visiting PhD student at the University of Oxford. B.B. thanks Vasco Portilheiro and Houssam Zenati for helpful discussions on an earlier version of this work.\looseness=-1
\newpage
\bibliography{bibliography}



\clearpage
\appendix
\numberwithin{equation}{section} 
\startcontents[appendix]

\newpage
\section*{Appendix}
\printcontents[appendix]{}{1}{\setcounter{tocdepth}{2}}
\bigskip

\newpage
\section{Review of Correlative Information Maximization}
\label{appendix:Review_CorInfoMax}

This section briefly reviews the correlative-information-maximization (CorInfoMax) framework underlying the batch method of \citet{Erdogan_CorInfoMax} and its online neural implementation in \citet{bozkurt2023correlative}. Our purpose is to make explicit the information-theoretic quantities that motivate CorInfoMax, the exact batch and online objectives it optimizes, and how its recurrent learning rules differ from those derived in the present paper.

\subsection{Correlative entropy and mutual information}

Following \citet{Erdogan_CorInfoMax}, we use the correlative entropy as a second-order entropy measure. For a finite set of vectors \(\mX=[\vx(1),\dots,\vx(T)]\in\mathbb{R}^{m\times T}\) with sample autocorrelation matrix
\[
\hat{\mR}_{\vx}
=
\frac{1}{T}\sum_{t=1}^{T}\vx(t)\vx(t)^\top,
\]
the corresponding deterministic correlative entropy is
\begin{equation}
    \hat{\mathcal H}_{\mathrm{CE}}^{(\varepsilon)}(\mX)
    :=
    \frac{1}{2}\log\det\!\bigl(\hat{\mR}_{\vx}+\varepsilon \mI\bigr)
    +
    \frac{m}{2}\log(2\pi e),
    \nonumber
\end{equation}
where \(\varepsilon>0\) is a small regularization constant. This quantity was introduced in \citet{Erdogan_CorInfoMax} under the name \emph{log-determinant entropy}. In later work, the terminology \emph{correlative entropy} has been adopted for the same second-order entropy measure \citep{bozkurt2023correlative,bozkurt2023supervisedcorinfomax,Ozsoy22:neurips}, and we follow that convention here.

For a pair of finite sets \(\mX=[\vx(1),\dots,\vx(T)]\in\mathbb{R}^{m\times T}\) and \(\mY=[\vy(1),\dots,\vy(T)]\in\mathbb{R}^{n\times T}\), define
\[
\hat{\mR}_{\vx\vy}
=
\frac{1}{T}\sum_{t=1}^{T}\vx(t)\vy(t)^\top,
\qquad
\hat{\mR}_{\vy\vx}=\hat{\mR}_{\vx\vy}^\top.
\]
The deterministic joint correlative entropy is
\begin{equation}
    \hat{\mathcal H}_{\mathrm{CE}}^{(\varepsilon)}(\mX,\mY)
    :=
    \frac{1}{2}
    \log\det\!\left(
        \begin{bmatrix}
            \hat{\mR}_{\vx}+\varepsilon \mI & \hat{\mR}_{\vx\vy}\\
            \hat{\mR}_{\vy\vx} & \hat{\mR}_{\vy}+\varepsilon \mI
        \end{bmatrix}
    \right)
    +
    \frac{m+n}{2}\log(2\pi e).\nonumber
\end{equation}
Using the block determinant identity, this can be written as
\begin{equation}
    \hat{\mathcal H}_{\mathrm{CE}}^{(\varepsilon)}(\mX,\mY)
    =
    \hat{\mathcal H}_{\mathrm{CE}}^{(\varepsilon)}(\mX)
    +
    \hat{\mathcal H}_{\mathrm{CE}}^{(\varepsilon)}(\mY\mid_L\mX),\nonumber
\end{equation}
where
\begin{equation}
    \hat{\mathcal H}_{\mathrm{CE}}^{(\varepsilon)}(\mY\mid_L\mX)
    :=
    \frac{1}{2}\log\det\!\bigl(\hat{\mR}_{\ve}+\varepsilon \mI\bigr)
    +
    \frac{n}{2}\log(2\pi e),\nonumber
\end{equation}
and
\begin{equation}
    \hat{\mR}_{\ve}
    :=
    \hat{\mR}_{\vy}
    -
    \hat{\mR}_{\vy\vx}
    \bigl(\hat{\mR}_{\vx}+\varepsilon \mI\bigr)^{-1}
    \hat{\mR}_{\vx\vy}.
    \label{eq:ld_error_covariance}
\end{equation}
The notation \(\mY\mid_L\mX\) emphasizes that this is the correlative entropy of the residual under the \emph{best linear} predictor of \(\mY\) from \(\mX\), not Shannon conditional entropy.

The corresponding deterministic  \emph{correlative mutual information} is defined by
\begin{equation}
    \hat{\mathcal I}_{\mathrm{CMI}}^{(\varepsilon)}(\mX,\mY)
    :=
    \hat{\mathcal H}_{\mathrm{CE}}^{(\varepsilon)}(\mY)
    -
    \hat{\mathcal H}_{\mathrm{CE}}^{(\varepsilon)}(\mY\mid_L\mX).\nonumber
\end{equation}
For sufficiently small \(\varepsilon\), \(\hat{\mathcal H}_{\mathrm{CE}}^{(\varepsilon)}(\mY\mid_L\mX)\) corresponds to the correlative entropy of the error of the best linear minimum-mean-square predictor of \(\vy\) from \(\vx\). Thus, the associated mutual information measures linear/correlative dependence rather than general statistical dependence.

\subsection{Batch CorInfoMax objective}

Let \(\mX=[\vx(1),\dots,\vx(T)]\in\mathbb{R}^{m\times T}\) denote the mixture matrix and \(\mY=[\vy(1),\dots,\vy(T)]\in\mathbb{R}^{n\times T}\) the separator outputs, constrained by \(\vy(t)\in\mathcal P\) for all \(t\). The batch CorInfoMax objective is the maximization of the deterministic correlative information:
\begin{subequations}
\label{eq:batch_corinfomax}
\begin{align}
    \underset{\mY\in\mathbb{R}^{n\times T}}{\mathrm{maximize}}
    \quad
    \hat{\mathcal I}_{\mathrm{CMI}}^{(\varepsilon)}(\mX,\mY)
    &=
    \frac{1}{2}\log\det\!\bigl(\hat{\mR}_{\vy}+\varepsilon \mI\bigr)
    -
    \frac{1}{2}\log\det\!\bigl(\hat{\mR}_{\ve}+\varepsilon \mI\bigr),\nonumber
    \\
    \text{subject to}\qquad
    \vy(t)&\in\mathcal P,
    \qquad
    t=1,\dots,T.\nonumber
\end{align}
\end{subequations}
Here \(\hat{\mR}_{\ve}\) is defined by Equation~\ref{eq:ld_error_covariance}. For sufficiently small \(\varepsilon\), it is the sample correlation matrix of the error of the best linear minimum-mean-square predictor of \(\vy\) from \(\vx\). In the non-zero-mean setting, the same formulation may be written with sample autocovariances in place of sample autocorrelations. Under the same identifiable-domain and sufficient-scattering assumptions discussed in Section~\ref{sec:ProblemSetup}, global optima of Equation~\ref{eq:batch_corinfomax} recover the sources up to ideal separation.

For the purposes of the present paper, the key point is that batch CorInfoMax optimizes the \emph{exact} regularized log-determinant objective. Consequently, inverse second-order matrices appear directly in the gradient and become central dynamical variables in the online formulation.

\subsection{Online CorInfoMax objective and neural dynamics}
The online CorInfoMax construction replaces the batch correlation matrices by exponentially weighted versions \citep{bozkurt2023correlative}. Using notation parallel to the main text, define
\begin{align}
    \hat{\mR}_{\vy}^{\lambda_y}(t)
    &:=
    \frac{1}{\eta_y(t)}
    \sum_{t'=1}^{t}
    \lambda_y^{\,t-t'}\vy(t')\vy(t')^\top,
    \label{eq:online_corinfomax_Ry}
    \\
    \hat{\mR}_{\ve}^{\lambda_e}(t)
    &:=
    \frac{1}{\eta_e(t)}
    \sum_{t'=1}^{t}
    \lambda_e^{\,t-t'}\ve(t')\ve(t')^\top,
    \label{eq:online_corinfomax_Re}
\end{align}
with forgetting factors \(\lambda_y,\lambda_e\in(0,1)\) and normalization factors
\[
\eta_y(t)=\sum_{t'=1}^{t}\lambda_y^{\,t-t'},
\qquad
\eta_e(t)=\sum_{t'=1}^{t}\lambda_e^{\,t-t'}.
\]
The error vector is
\begin{equation}
    \ve(t)=\vy(t)-\mW(t)\vx(t).\nonumber
\end{equation}
The corresponding online CorInfoMax problem at time \(t\) is
\begin{subequations}
\label{eq:online_corinfomax_obj}
\begin{align}
    \underset{\vy(t)\in\mathbb{R}^n}{\mathrm{maximize}}
    \quad
    \mathcal J_t^{\mathrm{CI}}(\vy(t))
    &:=
    \frac{1}{2}\log\det\!\bigl(\hat{\mR}_{\vy}^{\lambda_y}(t)+\varepsilon\mI\bigr)
    -
    \frac{1}{2}\log\det\!\bigl(\hat{\mR}_{\ve}^{\lambda_e}(t)+\varepsilon\mI\bigr),\nonumber
    \\
    \text{subject to}\qquad
    \vy(t)&\in\mathcal P.\nonumber
\end{align}
\end{subequations}

In the online formulation, the best linear predictor \(\mW(t)\) is treated as an adaptive parameter and updated through an online regularized least-squares criterion, which yields the usual LMS-type rule. Define the inverse-correlation states
\begin{equation}
    \mB_{\vy}(t)
    :=
    \bigl(\hat{\mR}_{\vy}^{\lambda_y}(t)+\varepsilon\mI\bigr)^{-1},
    \qquad
    \mB_{\ve}(t)
    :=
    \bigl(\hat{\mR}_{\ve}^{\lambda_e}(t)+\varepsilon\mI\bigr)^{-1}.\nonumber
\end{equation}
Gradient ascent on Equation~\ref{eq:online_corinfomax_obj} yields
\begin{equation}
    \nabla_{\vy(t)}\mathcal J_t^{\mathrm{CI}}
    =
    \gamma_y(t)\,\mB_{\vy}(t-1)\vy(t)
    -
    \gamma_e(t)\,\mB_{\ve}(t-1)\ve(t),\nonumber
\end{equation}
where
\begin{equation}
    \gamma_y(t)
    :=
    \Bigl(
        \lambda_y\eta_y(t-1)
        +
        \vy(t)^\top \mB_{\vy}(t-1)\vy(t)
    \Bigr)^{-1},
    \qquad
    \gamma_e(t)
    :=
    \Bigl(
        \lambda_e\eta_e(t-1)
        +
        \ve(t)^\top \mB_{\ve}(t-1)\ve(t)
    \Bigr)^{-1}.\nonumber
\end{equation}
Depending on the domain \(\mathcal P\), CorInfoMax then combines this gradient step with either a projected ascent update (for box-type domains) or a primal-dual/proximal step (for sparse and simplex domains).

The feedforward predictor is updated by the local LMS rule
\begin{equation}
    \mW(t+1)
    =
    \mW(t)
    +
    \mu_W(t)\ve(t)\vx(t)^\top.\nonumber
\end{equation}
By contrast, the lateral learning rules are based on inverse-correlation states. Applying the Matrix Inversion Lemma to Equations~\ref{eq:online_corinfomax_Ry} and \ref{eq:online_corinfomax_Re} yields the exact recursions
\begin{align}
    \mB_{\vy}(t+1)
    &=
    \frac{\eta_y(t)}{\lambda_y \eta_y(t-1)}
    \left(
        \mB_{\vy}(t)
        -
        \gamma_y(t)\,
        \mB_{\vy}(t)\vy(t)\vy(t)^\top\mB_{\vy}(t)
    \right),
    \label{eq:online_corinfomax_By_exact}
    \\
    \mB_{\ve}(t+1)
    &=
    \frac{\eta_e(t)}{\lambda_e \eta_e(t-1)}
    \left(
        \mB_{\ve}(t)
        -
        \gamma_e(t)\,
        \mB_{\ve}(t)\ve(t)\ve(t)^\top\mB_{\ve}(t)
    \right).
    \label{eq:online_corinfomax_Be_exact}
\end{align}

In \citet{bozkurt2023correlative}, these exact inverse-state updates are then simplified to obtain a biologically plausible implementation. The inverse error-correlation update in Equation~\ref{eq:online_corinfomax_Be_exact} is not retained as a learned state; instead, the approximation
\begin{equation}
    \hat{\mR}_{\ve}^{\lambda_e}(t)+\varepsilon\mI \approx \varepsilon \mI,
    \qquad
    \text{hence}
    \qquad
    \mB_{\ve}(t)\approx \varepsilon^{-1}\mI,
    \label{eq:online_corinfomax_Be_approx}
\end{equation}
is invoked, motivated by the expectation that the prediction error becomes small in the noiseless linear setting. In addition, for \(\lambda_y\) close to \(1\) and \(t\) sufficiently large, the scalar factor \(\gamma_y(t)\) is approximated by
\begin{equation}
    \gamma_y(t)\approx \frac{1-\lambda_y}{\lambda_y}.
    \label{eq:online_corinfomax_gamma_approx}
\end{equation}
Substituting Equation~\ref{eq:online_corinfomax_gamma_approx} into Equation~\ref{eq:online_corinfomax_By_exact} yields the simplified lateral update used in the CorInfoMax neural networks:
\begin{equation}
    \mB_{\vy}(t+1)
    =
    \lambda_y^{-1}
    \left(
        \mB_{\vy}(t)
        -
        \frac{1-\lambda_y}{\lambda_y}\,
        \mB_{\vy}(t)\vy(t)\vy(t)^\top\mB_{\vy}(t)
    \right).
    \label{eq:online_corinfomax_By_used}
\end{equation}

\subsection{Comparison with Predictive Entropy Maximization.}
\label{app:Comparison}
The distinction from the present paper can now be stated precisely. In CorInfoMax, the exact regularized log-determinant objective produces gradients involving inverse second-order states, so recurrent plasticity is expressed through \(\mB_{\vy}(t)\) rather than through covariance traces. Even after the approximations in Equations~\ref{eq:online_corinfomax_Be_approx} and \ref{eq:online_corinfomax_gamma_approx}, the lateral update in Equation~\ref{eq:online_corinfomax_By_used} remains a rank-one update built from transformed population activities.

Indeed, if we define \(\vz(t):=\mB_{\vy}(t)\vy(t)\), then Equation~\ref{eq:online_corinfomax_By_used} becomes
\[
\mB_{\vy}(t+1)
=
\lambda_y^{-1}
\left(
    \mB_{\vy}(t)
    -
    \frac{1-\lambda_y}{\lambda_y}\,
    \vz(t)\vz(t)^\top
\right).
\]
Hence the update of an individual lateral coefficient depends on the transformed signal \(\vz(t)\), not directly on the co-activity of the neuron pair it couples. In the \(3\times 3\) case, if
\[
\mB_{\vy}(t)
=
\begin{bmatrix}
b_{11}(t) & b_{12}(t) & b_{13}(t)\\
b_{21}(t) & b_{22}(t) & b_{23}(t)\\
b_{31}(t) & b_{32}(t) & b_{33}(t)
\end{bmatrix},
\qquad
\vy(t)
=
\begin{bmatrix}
y_1(t)\\
y_2(t)\\
y_3(t)
\end{bmatrix},
\]
then
\[
z_1(t)=b_{11}(t)y_1(t)+b_{12}(t)y_2(t)+b_{13}(t)y_3(t),
\qquad
z_2(t)=b_{21}(t)y_1(t)+b_{22}(t)y_2(t)+b_{23}(t)y_3(t),
\]
and Equation~\ref{eq:online_corinfomax_By_used} gives
\begin{equation}
    b_{21}(t+1)
    =
    \lambda_y^{-1}
    \left(
        b_{21}(t)
        -
        \frac{1-\lambda_y}{\lambda_y}\,
        z_2(t)z_1(t)
    \right).\nonumber
\end{equation}
Thus, the update of \(b_{21}(t)\) depends not only on the pair \((y_2(t),y_1(t))\), but also on \(y_3(t)\). In general, a single recurrent update depends on the activity of the full output population. This violates the local plasticity condition of biological plausibility, as defined in the introduction.

By contrast, Predictive Entropy Maximization introduces a second-order Taylor surrogate of the regularized log-determinant and works directly with covariance traces rather than inverse-covariance states. The corresponding recurrent statistics are the exponentially weighted variance and cross-covariance traces of the centered outputs, defined by \(\hat v_i(t):=[\hat{\mC}^{\lambda}(t)]_{ii}\) and \(\hat c_{ij}(t):=[\hat{\mC}^{\lambda}(t)]_{ij}\) for \(i\neq j\). Their recursions are
\[
\hat v_i(t)
=
\lambda \hat v_i(t-1)
+
(1-\lambda)\bar y_i(t)^2,
\qquad
\hat c_{ij}(t)
=
\lambda \hat c_{ij}(t-1)
+
(1-\lambda)\bar y_i(t)\bar y_j(t),
\]
so the cross-covariance update is written directly in terms of centered pairwise co-activity, while the variance update depends only on the squared centered activity of the corresponding neuron. This is the key distinction emphasized throughout the main text: the surrogate formulation removes the need to maintain inverse-covariance states and yields a more directly local interpretation of recurrent plasticity.

\section{Derivation of the update rules for the online optimization objective}
\label{appendix:UpdateRuleDerivations}

In this section, we derive the recursive updates for the online second-order statistics and the feedforward weights, as well as the output gradient used in the fast neural dynamics. 

\subsection{Recursive updates for the running mean and covariance}

Recall from Equation~\ref{eq:ExponentiallyWeightedSampleCovariance} that the exponentially weighted output covariance is defined by
\begin{equation}
    \hat{\mC}^{\lambda}(t)
    =
    \frac{1}{\eta(t)}
    \sum_{t'=1}^{t}
    \lambda^{\,t-t'}
    \bigl(\vy(t')-\hat{\bm{\mu}}^{\lambda}(t')\bigr)
    \bigl(\vy(t')-\hat{\bm{\mu}}^{\lambda}(t')\bigr)^\top,\nonumber
\end{equation}
where
\begin{equation}
    \eta(t)
    =
    \sum_{t'=1}^{t}\lambda^{\,t-t'}
    =
    \frac{1-\lambda^t}{1-\lambda}.\nonumber
\end{equation}
Since \(\eta(t)=\lambda \eta(t-1)+1\), we define \(\alpha(t):=\frac{1}{\eta(t)}\).

Using this identity, we isolate the most recent sample in the covariance sum and obtain
\begin{align}
    \hat{\mC}^{\lambda}(t)
    &=
    \frac{1}{\eta(t)}
    \left[
        \lambda
        \sum_{t'=1}^{t-1}
        \lambda^{\,t-1-t'}
        \bigl(\vy(t')-\hat{\bm{\mu}}^{\lambda}(t')\bigr)
        \bigl(\vy(t')-\hat{\bm{\mu}}^{\lambda}(t')\bigr)^\top
    \right]
    \nonumber\\
    &\qquad
    +
    \frac{1}{\eta(t)}
    \bigl(\vy(t)-\hat{\bm{\mu}}^{\lambda}(t)\bigr)
    \bigl(\vy(t)-\hat{\bm{\mu}}^{\lambda}(t)\bigr)^\top
    \nonumber\\
    &=
    \frac{\lambda \eta(t-1)}{\eta(t)}
    \hat{\mC}^{\lambda}(t-1)
    +
    \frac{1}{\eta(t)}
    \bigl(\vy(t)-\hat{\bm{\mu}}^{\lambda}(t)\bigr)
    \bigl(\vy(t)-\hat{\bm{\mu}}^{\lambda}(t)\bigr)^\top
    \nonumber\\
    &=
    \bigl(1-\alpha(t)\bigr)\hat{\mC}^{\lambda}(t-1)
    +
    \alpha(t)
    \bigl(\vy(t)-\hat{\bm{\mu}}^{\lambda}(t)\bigr)
    \bigl(\vy(t)-\hat{\bm{\mu}}^{\lambda}(t)\bigr)^\top.
    \label{eq:UpdateOfWeightedCovarianceMatrix}
\end{align}
Exactly the same argument applied to the exponentially weighted mean
\[
    \hat{\bm{\mu}}^{\lambda}(t)
    =
    \frac{1}{\eta(t)}
    \sum_{t'=1}^{t}\lambda^{\,t-t'}\vy(t')
\]
yields the recursion
\begin{equation}
    \hat{\bm{\mu}}^{\lambda}(t)
    =
    \bigl(1-\alpha(t)\bigr)\hat{\bm{\mu}}^{\lambda}(t-1)
    +
    \alpha(t)\vy(t).
    \label{eq:UpdateOfWeightedMeanVector}
\end{equation}

For large \(t\), the normalization converges to \(\eta(t)\approx (1-\lambda)^{-1}\), hence \(\alpha(t)\approx 1-\lambda\). Under this steady-state approximation, Equations~\ref{eq:UpdateOfWeightedCovarianceMatrix} and \ref{eq:UpdateOfWeightedMeanVector} reduce to
\begin{align}
    \hat{\bm{\mu}}^{\lambda}(t)
    &=
    \lambda \hat{\bm{\mu}}^{\lambda}(t-1)
    +
    (1-\lambda)\vy(t),
    \label{eq:UpdateOfWeightedMeanVectorSimplified}
    \\
    \hat{\mC}^{\lambda}(t)
    &=
    \lambda \hat{\mC}^{\lambda}(t-1)
    +
    (1-\lambda)
    \bigl(\vy(t)-\hat{\bm{\mu}}^{\lambda}(t)\bigr)
    \bigl(\vy(t)-\hat{\bm{\mu}}^{\lambda}(t)\bigr)^\top.
    \label{eq:UpdateOfWeightedCovarianceMatrixSimplified}
\end{align}
For convenience, define the centered activity
\[
    \bar{\vy}(t):=\vy(t)-\hat{\bm{\mu}}^{\lambda}(t).
\]
Then Equation~\ref{eq:UpdateOfWeightedCovarianceMatrixSimplified} can be written as
\[
    \hat{\mC}^{\lambda}(t)
    =
    \lambda \hat{\mC}^{\lambda}(t-1)
    +
    (1-\lambda)\bar{\vy}(t)\bar{\vy}(t)^\top.
\]

We now decompose the covariance matrix as
\begin{equation}
    \hat{\mC}^{\lambda}(t)
    =
    \hat{\mD}^{\lambda}(t)
    +
    \hat{\mO}^{\lambda}(t),
\end{equation}
where \(\hat{\mD}^{\lambda}(t)\) is diagonal and \(\hat{\mO}^{\lambda}(t)\) has zero diagonal. Reading off the diagonal and off-diagonal parts of Equation~\ref{eq:UpdateOfWeightedCovarianceMatrixSimplified} yields
\begin{align}
    \hat{v}_i(t)
    &=
    \lambda \hat{v}_i(t-1)
    +
    (1-\lambda)\bar y_i(t)^2,
    \qquad i=1,\dots,n,
    \label{eq:DiagonalVarianceUpdate}
    \\
    \hat{c}_{ij}(t)
    &=
    \lambda \hat{c}_{ij}(t-1)
    +
    (1-\lambda)\bar y_i(t)\bar y_j(t),
    \qquad i\neq j.
    \label{eq:OffDiagonalCovarianceUpdate}
\end{align}

\subsection{Derivation of the feedforward weight update}

Let \(\ve(t):=\vy(t)-\mW\vx(t)\) denote the instantaneous prediction error. We define the instantaneous quadratic cost
\[
    \mathcal{J}_t(\mW)
    \propto
    \frac{1}{2}\|\ve(t)\|_2^2
    =
    \frac{1}{2}\|\vy(t)-\mW\vx(t)\|_2^2.
\]
Differentiating with respect to \(\mW\) gives
\begin{align}
    \nabla_{\mW}\mathcal{J}_t(\mW)
    &=
    \frac{\partial}{\partial \mW}
    \left[
        \frac{1}{2}
        (\vy(t)-\mW\vx(t))^\top
        (\vy(t)-\mW\vx(t))
    \right]
    \nonumber\\
    &=
    -(\vy(t)-\mW\vx(t))\vx(t)^\top
    \nonumber\\
    &=
    -\ve(t)\vx(t)^\top.\nonumber
\end{align}
A gradient step with learning rate \(\alpha_W(t)\) therefore yields
\begin{align}
    \mW(t)
    &=
    \mW(t-1)-\alpha_W(t)\nabla_{\mW}\mathcal{J}_t(\mW(t-1))
    \nonumber\\
    &=
    \mW(t-1)+\alpha_W(t)\ve(t)\vx(t)^\top.
    \label{eq:FeedforwardUpdateDerivation}
\end{align}

This formulation yields a biologically plausible, local error-modulated Hebbian learning rule where the synaptic update is proportional to the product of the presynaptic activity \(\vx(t)\) and the local postsynaptic error signal \(\ve(t)\).

\subsection{Gradient of the surrogate objective for the output dynamics}
\label{Appendix:OutputGradDerivation}

We now derive the gradient of the fast-time-scale objective with respect to the current output \(\vy(t)\). Recall from Equation~\ref{eq:full_cost_J} that
\begin{equation}
    \mathcal{J}_t(\vy)
    =
    - \sum_{i=1}^n \log\!\bigl(\hat{v}_i(t)+\varepsilon\bigr)
    +
    \frac{1}{2}
    \sum_{i=1}^n
    \sum_{\substack{j=1\\j\neq i}}^n
    \frac{\hat{c}_{ij}(t)^2}
         {\bigl(\hat{v}_i(t)+\varepsilon\bigr)\bigl(\hat{v}_{j}(t)+\varepsilon\bigr)}
    +
    \gamma \|\vy(t)-\mW(t-1)\vx(t)\|_2^2.
    \nonumber
\end{equation}
At time \(t\), the current statistics \(\hat{\bm{\mu}}^{\lambda}(t)\), \(\hat{\mD}^{\lambda}(t)\), and \(\hat{\mO}^{\lambda}(t)\) depend on \(\vy(t)\) through the steady-state updates in Equations~\ref{eq:UpdateOfWeightedMeanVectorSimplified}, \ref{eq:DiagonalVarianceUpdate}, and \ref{eq:OffDiagonalCovarianceUpdate}. We first compute the necessary derivatives.

\paragraph{Derivative of the centered activity.}
For each component,
\[
    \hat{\mu}_{k}(t)
    =
    \lambda \hat{\mu}_{k}(t-1)
    +
    (1-\lambda)y_k(t).
\]
Hence
\[
    \bar y_k(t)
    =
    y_k(t)-\hat{\mu}_{k}(t)
    =
    \lambda \bigl(y_k(t)-\hat{\mu}_{k}(t-1)\bigr),
\]
and therefore
\begin{equation}
    \frac{\partial \bar y_i(t)}{\partial y_k(t)}
    =
    \lambda\,\delta_{ik},
    \label{eq:DerivativeCenteredOutput}
\end{equation}
where \(\delta_{ik}\) is the Kronecker delta.

\paragraph{Derivatives of the diagonal and off-diagonal statistics.}
Using Equation~\ref{eq:DiagonalVarianceUpdate} together with Equation~\ref{eq:DerivativeCenteredOutput}, we obtain
\begin{equation}
    \frac{\partial \hat{v}_i(t)}{\partial y_k(t)}
    =
    2\lambda(1-\lambda)\bar y_i(t)\,\delta_{ik}.
    \label{eq:DerivativeDiagonalStatistic}
\end{equation}
Similarly, from Equation~\ref{eq:OffDiagonalCovarianceUpdate},
\begin{equation}
    \frac{\partial \hat{c}_{ij}(t)}{\partial y_k(t)}
    =
    \lambda(1-\lambda)
    \bigl(
        \delta_{ik}\bar y_j(t)+\delta_{jk}\bar y_i(t)
    \bigr),
    \qquad i\neq j.
    \label{eq:DerivativeOffDiagonalStatistic}
\end{equation}

\paragraph{Derivative of the variance term.}
Using Equation~\ref{eq:DerivativeDiagonalStatistic},
\begin{align}
    \frac{\partial}{\partial y_k(t)}
    \left(
        -\sum_{i=1}^n \log(\hat{v}_i(t)+\varepsilon)
    \right)
    &=
    -\sum_{i=1}^n
    \frac{1}{\hat{v}_i(t)+\varepsilon}
    \frac{\partial \hat{v}_i(t)}{\partial y_k(t)}
    \nonumber\\
    &=
    -\frac{2\lambda(1-\lambda)\bar y_k(t)}
           {\hat{v}_k(t)+\varepsilon}.
    \label{eq:DerivativeVarianceTerm}
\end{align}

\paragraph{Derivative of the cross-covariance penalty.}
Define
\[
    \mathcal{R}_t
    :=
    \frac{1}{2}
    \sum_{i=1}^n
    \sum_{\substack{j=1\\j\neq i}}^n
    \frac{\hat{c}_{ij}(t)^2}
         {\bigl(\hat{v}_i(t)+\varepsilon\bigr)\bigl(\hat{v}_{j}(t)+\varepsilon\bigr)}.
\]
Only terms with \(i=k\) or \(j=k\) contribute to the derivative with respect to \(y_k(t)\). Using the symmetry of \(\hat{\mO}^{\lambda}(t)\), together with Equations~\ref{eq:DerivativeDiagonalStatistic} and \ref{eq:DerivativeOffDiagonalStatistic}, a direct calculation gives
\begin{align}
    \frac{\partial \mathcal{R}_t}{\partial y_k(t)}
    &=
    2\lambda(1-\lambda)
    \sum_{\substack{j=1\\j\neq k}}^n
    \frac{
        \hat{c}_{kj}(t)\bar y_j(t)
    }{
        \bigl(\hat{v}_k(t)+\varepsilon\bigr)
        \bigl(\hat{v}_{j}(t)+\varepsilon\bigr)
    }
    \nonumber\\
    &\qquad
    -
    2\lambda(1-\lambda)
    \sum_{\substack{j=1\\j\neq k}}^n
    \frac{
        \hat{c}_{kj}(t)^2 \bar y_k(t)
    }{
        \bigl(\hat{v}_k(t)+\varepsilon\bigr)^2
        \bigl(\hat{v}_{j}(t)+\varepsilon\bigr)
    }.
    \label{eq:DerivativeCrossCovariancePenalty}
\end{align}

\paragraph{Derivative of the prediction term.}
Let \(e_k(t):=[\vy(t)-\mW(t-1)\vx(t)]_k\). Then
\begin{equation}
    \frac{\partial}{\partial y_k(t)}
    \gamma \|\vy(t)-\mW(t-1)\vx(t)\|_2^2
    =
    2\gamma e_k(t).
    \label{eq:DerivativePredictionTerm}
\end{equation}

\paragraph{Full gradient.}
Combining Equations~\ref{eq:DerivativeVarianceTerm}, \ref{eq:DerivativeCrossCovariancePenalty}, and \ref{eq:DerivativePredictionTerm}, we obtain
\begin{align}
    \frac{\partial \mathcal{J}_t(\vy)}{\partial y_k(t)}
    &=
    2\lambda(1-\lambda)
    \left[
        -\frac{\bar y_k(t)}
               {\hat{v}_k(t)+\varepsilon}
        +
        \sum_{\substack{j=1\\j\neq k}}^n
        \frac{
            \hat{c}_{kj}(t)\bar y_j(t)
        }{
            \bigl(\hat{v}_k(t)+\varepsilon\bigr)
            \bigl(\hat{v}_{j}(t)+\varepsilon\bigr)
        }
    \right]
    \nonumber\\
    &\qquad
    -
    2\lambda(1-\lambda)
    \sum_{\substack{j=1\\j\neq k}}^n
    \frac{
        \hat{c}_{kj}(t)^2 \bar y_k(t)
    }{
        \bigl(\hat{v}_k(t)+\varepsilon\bigr)^2
        \bigl(\hat{v}_{j}(t)+\varepsilon\bigr)
    }
    +
    2\gamma e_k(t).
    \label{eq:GradientOfDetMax_TermToBeDiscarded}
\end{align}

The last term in Equation~\ref{eq:GradientOfDetMax_TermToBeDiscarded} is quadratic in the off-diagonal covariance entries. Since the surrogate objective is used precisely in the regime where the normalized cross-covariances are small, this term is of higher order than the leading recurrent interaction retained in the main text. Discarding it yields the approximation
\begin{align}
    \frac{\partial \mathcal{J}_t(\vy)}{\partial y_k(t)}
    &\approx
    2\lambda(1-\lambda)
    \left[
        -\frac{\bar y_k(t)}
               {\hat{v}_k(t)+\varepsilon}
        +
        \sum_{\substack{j=1\\j\neq k}}^n
        \frac{
            \hat{c}_{kj}(t)\bar y_j(t)
        }{
            \bigl(\hat{v}_k(t)+\varepsilon\bigr)
            \bigl(\hat{v}_{j}(t)+\varepsilon\bigr)
        }
    \right]
    +
    2\gamma e_k(t).\nonumber
\end{align}
Absorbing the constant prefactor \(\lambda(1-\lambda)\) into \(\gamma\), we recover the simplified gradient form used in the main text:
\[
    \frac{\partial \mathcal{J}_t(\vy)}{\partial y_k(t)}
    \propto
    \left[
        -\frac{\bar y_k(t)}
               {\hat{v}_k(t)+\varepsilon}
        +
        \sum_{\substack{j=1\\j\neq k}}^n
        \frac{
            \hat{c}_{kj}(t)\bar y_j(t)
        }{
            \bigl(\hat{v}_k(t)+\varepsilon\bigr)
            \bigl(\hat{v}_{j}(t)+\varepsilon\bigr)
        }
    \right]
    +
    \gamma e_k(t).
\]
This is the descent direction used for the fast neural dynamics in Section~\ref{sec:DecorrelationObjective}.

\paragraph{Descent property of the truncated direction.}
We now show that the truncated direction used in the fast neural dynamics is indeed a strict local descent direction for the exact objective under a simple pointwise condition. Define the regularized diagonal matrix
\begin{equation}
    \hat{\mD}^{\lambda,\varepsilon}(t)
    :=
    \hat{\mD}^{\lambda}(t)+\varepsilon \mI,\nonumber
\end{equation}
let \(\vg(t)\in\mathbb{R}^n\) denote the truncated direction used in the fast neural dynamics, and define
\begin{equation}
    \mB^{\lambda, \varepsilon}(t)
    :=
    \bigl(\hat{\mD}^{\lambda,\varepsilon}(t)\bigr)^{-1/2}
    \hat{\mO}^{\lambda}(t)
    \bigl(\hat{\mD}^{\lambda,\varepsilon}(t)\bigr)^{-1/2}.\nonumber
\end{equation}
Further, let \(\vr(t)\in\mathbb{R}^n\) be the vector with entries
\begin{equation}
    r_k(t)
    :=
    \frac{[\mB^{\lambda, \varepsilon}(t)^2]_{kk}}{\hat v_k(t)+\varepsilon}\,\bar y_k(t),
    \qquad k=1,\dots,n.\nonumber
\end{equation}

\begin{proposition}[Sufficient condition for descent of the truncated direction]
\label{prop:descent_truncated_direction}
With the notation above, the exact gradient can be written as
\begin{equation}
    \nabla_{\vy}\mathcal J_t(\vy)
    =
    2\lambda(1-\lambda)\bigl(\vg(t)-\vr(t)\bigr).
    \label{eq:gradient_decomposition_descent}
\end{equation}
Consequently, \(-\vg(t)\) is a strict descent direction for \(\mathcal J_t\) whenever
\begin{equation}
    \|\vr(t)\|_2 < \|\vg(t)\|_2.
    \label{eq:pointwise_descent_condition}
\end{equation}
Moreover,
\begin{equation}
    \|\vr(t)\|_2
    \le
    \frac{\|\mB^{\lambda, \varepsilon}(t)\|_2^2}{\hat v_{\min}^{(\varepsilon)}(t)}\,
    \|\bar{\vy}(t)\|_2,
    \qquad
    \hat v_{\min}^{(\varepsilon)}(t):=\min_{1\le k\le n}\bigl(\hat v_k(t)+\varepsilon\bigr),
    \label{eq:coarse_descent_bound}
\end{equation}
so the simpler sufficient condition
\begin{equation}
    \frac{\|\mB^{\lambda, \varepsilon}(t)\|_2^2}{\hat v_{\min}^{(\varepsilon)}(t)}\,
    \|\bar{\vy}(t)\|_2
    <
    \|\vg(t)\|_2
    \label{eq:coarse_descent_condition}
\end{equation}
also guarantees that \(-\vg(t)\) is a strict descent direction.
\end{proposition}

\begin{proof}
By construction, the \(k\)-th component of \(\vg(t)\) is
\[
g_k(t)
=
-\frac{\bar y_k(t)}{\hat v_k(t)+\varepsilon}
+
\sum_{\substack{j=1\\j\neq k}}^n
\frac{\hat c_{kj}(t)\bar y_j(t)}
     {\bigl(\hat v_k(t)+\varepsilon\bigr)\bigl(\hat v_j(t)+\varepsilon\bigr)}
+
\gamma e_k(t).
\]
Since \(\hat{\mO}^{\lambda}(t)\) has zero diagonal,
\[
[\mB^{\lambda, \varepsilon}(t)^2]_{kk}
=
\sum_{j=1}^n B^{\lambda, \varepsilon}_{kj}(t)^2
=
\sum_{\substack{j=1\\j\neq k}}^n
\frac{\hat c_{kj}(t)^2}
     {\bigl(\hat v_k(t)+\varepsilon\bigr)\bigl(\hat v_j(t)+\varepsilon\bigr)}.
\]
Therefore
\[
r_k(t)
=
\frac{[\mB^{\lambda, \varepsilon}(t)^2]_{kk}}{\hat v_k(t)+\varepsilon}\bar y_k(t)
=
\sum_{\substack{j=1\\j\neq k}}^n
\frac{\hat c_{kj}(t)^2\,\bar y_k(t)}
     {\bigl(\hat v_k(t)+\varepsilon\bigr)^2\bigl(\hat v_j(t)+\varepsilon\bigr)}.
\]
Comparing this identity with Equation~\ref{eq:GradientOfDetMax_TermToBeDiscarded} shows that
\[
\frac{\partial \mathcal J_t(\vy)}{\partial y_k(t)}
=
2\lambda(1-\lambda)\bigl(g_k(t)-r_k(t)\bigr),
\]
which proves Equation~\ref{eq:gradient_decomposition_descent}.

Now let \(\vd(t):=-\vg(t)\). For any sufficiently small step size \(\eta>0\),
\[
\mathcal J_t(\vy+\eta \vd)
=
\mathcal J_t(\vy)
+
\eta\,
\langle \nabla_{\vy}\mathcal J_t(\vy), \vd(t)\rangle
+
O(\eta^2)
=
\mathcal J_t(\vy)
-
\eta\,
\langle \nabla_{\vy}\mathcal J_t(\vy), \vg(t)\rangle
+
O(\eta^2).
\]
Hence it is enough to show that \(\langle \nabla_{\vy}\mathcal J_t(\vy), \vg(t)\rangle>0\). Using Equation~\eqref{eq:gradient_decomposition_descent},
\[
\langle \nabla_{\vy}\mathcal J_t(\vy), \vg(t)\rangle
=
2\lambda(1-\lambda)
\Bigl(
\|\vg(t)\|_2^2
-
\langle \vr(t),\vg(t)\rangle
\Bigr).
\]
By Cauchy--Schwarz,
\[
\langle \nabla_{\vy}\mathcal J_t(\vy), \vg(t)\rangle
\ge
2\lambda(1-\lambda)\,
\|\vg(t)\|_2
\bigl(
\|\vg(t)\|_2-\|\vr(t)\|_2
\bigr).
\]
Since \(2\lambda(1-\lambda)>0\), condition~\eqref{eq:pointwise_descent_condition} implies strict descent.

Finally,
\[
\|\vr(t)\|_2
=
\left\|
\bigl(\hat{\mD}^{\lambda,\varepsilon}(t)\bigr)^{-1}
\text{diag}\!\bigl(\mB^{\lambda, \varepsilon}(t)^2\bigr)\bar{\vy}(t)
\right\|_2
\le
\left\|
\bigl(\hat{\mD}^{\lambda,\varepsilon}(t)\bigr)^{-1}
\text{diag}\!\bigl(\mB^{\lambda, \varepsilon}(t)^2\bigr)
\right\|_2
\|\bar{\vy}(t)\|_2.
\]
Moreover, for every \(k\),
\[
[\mB^{\lambda, \varepsilon}(t)^2]_{kk}\le \|\mB^{\lambda, \varepsilon}(t)^2\|_2=\|\mB^{\lambda, \varepsilon}(t)\|_2^2,
\]
hence
\[
\left\|
\bigl(\hat{\mD}^{\lambda,\varepsilon}(t)\bigr)^{-1}
\text{diag}\!\bigl(\mB^{\lambda, \varepsilon}(t)^2\bigr)
\right\|_2
\le
\frac{\|\mB^{\lambda, \varepsilon}(t)\|_2^2}{\hat v_{\min}^{(\varepsilon)}(t)}.
\]
This proves Equation~\ref{eq:coarse_descent_bound}, and Equation~\eqref{eq:coarse_descent_condition} is then immediate.
\end{proof}


\section{Domain-specific network realizations}
\label{Appendix:NetworkDynamicsIllustrations}

The slow synaptic and statistical updates are identical across all source domains and are derived in Appendix~\ref{appendix:UpdateRuleDerivations}. The only domain-dependent part is the fast output inference step. All domains share the same descent direction, while the source domain determines the output nonlinearity and, when necessary, the dynamics of auxiliary inhibitory variables. 

For each sample \(t\) and neural-dynamics iteration \(\tau\), define
\begin{equation}
    g_k(t;\tau)
    :=
    -\frac{\bar y_k(t;\tau)}{\hat{v}_k(t)+\varepsilon}
    +
    \sum_{\substack{j=1\\j\neq k}}^n
    \frac{
        \hat{c}_{kj}(t)\,\bar y_j(t;\tau)
    }{
        \bigl(\hat{v}_k(t)+\varepsilon\bigr)\bigl(\hat{v}_{j}(t)+\varepsilon\bigr)
    }
    +
    \gamma e_k(t;\tau),\nonumber
\end{equation}
where
\[
\bar y_k(t;\tau)=y_k(t;\tau)-\hat\mu_k(t),
\qquad
e_k(t;\tau)=y_k(t;\tau)-\sum_{\ell=1}^m W_{k\ell}(t-1)x_\ell(t).
\]
We further define the unconstrained pre-activation update
\begin{equation}
    \tilde y_k(t;\tau+1)
    :=
    y_k(t;\tau)-\eta_y(t,\tau)\,g_k(t;\tau).
    \label{eq:preactivation_update_common_appendix}
\end{equation}
The domain-specific architectures below are then obtained by combining Equation~\ref{eq:preactivation_update_common_appendix} with the appropriate projection or proximal operator \citep{parikh2014proximal}.

\begin{algorithm}[t!]
\caption{Predictive Entropy Maximization: generic online procedure}
\label{alg:predictivedecorr}
\begin{algorithmic}[1]
\STATE \textbf{Input:} streaming mixtures \(\{\vx(t)\}_{t=1}^T\), source domain \(\mathcal P\)
\STATE \textbf{Hyperparameters:} forgetting factor \(\lambda\), feedforward step size schedule \(\alpha_W(t)\), neural step size schedule \(\eta_y(t,\tau)\), fast-dynamics horizon \(\tau_{\max}\)
\STATE Initialize \(\mW(0)\), \(\hat{\bm{\mu}}^{\lambda}(0)\), \(\{\hat v_i(0)\}_{i=1}^n\), \(\{\hat c_{ij}(0)\}_{i\neq j}\), and \(\theta_{\mathcal P}(0)\) if needed
\FOR{\(t=1,\ldots,T\)}
    \STATE Initialize \(\vy(t;0)\) and \(\theta_{\mathcal P}(t;0)\) if required by the domain
    \FOR{\(\tau=0,\ldots,\tau_{\max}-1\)}
        \STATE Compute \(\vg(t;\tau)\) from Equation~\ref{eq:output_inference_direction}
        \STATE \(\vy(t;\tau+1)\gets \sigma_{\mathcal P}\!\left(\vy(t;\tau)-\eta_y(t,\tau)\vg(t;\tau)\,;\,\theta_{\mathcal P}(t;\tau)\right)\)
        \STATE Update \(\theta_{\mathcal P}(t;\tau+1)\) if required by the domain
    \ENDFOR
    \STATE Set \(\vy(t)\gets \vy(t;\tau_{\max})\) and \(\ve(t)\gets \vy(t)-\mW(t-1)\vx(t)\)
    \STATE \(\mW(t)\gets \mW(t-1)+\alpha_W(t)\ve(t)\vx(t)^\top\)
    \STATE \(\hat{\bm{\mu}}^{\lambda}(t)\gets \lambda \hat{\bm{\mu}}^{\lambda}(t-1)+(1-\lambda)\vy(t)\), \quad \(\bar{\vy}(t)\gets \vy(t)-\hat{\bm{\mu}}^{\lambda}(t)\)
    \STATE \(\hat v_i(t)\gets \lambda \hat v_i(t-1)+(1-\lambda)\bar y_i(t)^2,\qquad \forall i\)
    \STATE \(\hat c_{ij}(t)\gets \lambda \hat c_{ij}(t-1)+(1-\lambda)\bar y_i(t)\bar y_j(t),\qquad \forall i\neq j\)
\ENDFOR
\end{algorithmic}
\end{algorithm}

\subsection{Antisparse sources}
\label{app:antisparse_sources}

Consider the antisparse domain
\[
\gB_{\mathrm{max}} := \{\vs \in \mathbb{R}^n : \|\vs\|_\infty \le 1\}.
\]
The output constraint is enforced directly through the clipping nonlinearity onto the box \([-1,1]^n\). The fast output dynamics are
\begin{align*}
    \tilde y_k(t;\tau+1)
    &=
    y_k(t;\tau)-\eta_y(t,\tau)\,g_k(t;\tau),\\
    y_k(t;\tau+1)
    &=
    \operatorname{clip}_{[-1,1]}\!\bigl(\tilde y_k(t;\tau+1)\bigr),
    \qquad k=1,\dots,n,
\end{align*}
where
\[
    \operatorname{clip}_{[-1,1]}(u)
    :=
    \begin{cases}
        -1, & u < -1,\\
        u, & -1 \le u \le 1,\\
        1, & u > 1.
    \end{cases}
\]
Thus, the corresponding network consists only of the output population, driven by the feedforward estimate \(\mW(t)\vx(t)\), variance-normalized self-excitation, and covariance-dependent lateral inhibition.\looseness=-1

\subsection{Nonnegative antisparse sources}
\label{app:nn_antisparse_sources}

Consider the nonnegative antisparse domain
\[
\gB_{\mathrm{max}, +} := \gB_{\mathrm{max}} \cap \mathbb{R}_+^n.
\]
The dynamics are identical to the antisparse case except that the projection is now onto the box \([0,1]^n\):
\begin{align*}
    \tilde y_k(t;\tau+1)
    &=
    y_k(t;\tau)-\eta_y(t,\tau)\,g_k(t;\tau),\\
    y_k(t;\tau+1)
    &=
    \operatorname{clip}_{[0,1]}\!\bigl(\tilde y_k(t;\tau+1)\bigr),
    \qquad k=1,\dots,n,
\end{align*}
where
\[
    \operatorname{clip}_{[0,1]}(u)
    :=
    \begin{cases}
        0, & u < 0,\\
        u, & 0 \le u \le 1,\\
        1, & u > 1.
    \end{cases}
\]

\subsection{Sparse sources}
\label{app:sparse_sources}

Consider the sparse domain
\[
\gB_1 := \{\vs \in \mathbb{R}^n : \|\vs\|_1 \le 1\}.
\]
We enforce the \(\ell_1\)-constraint through a shared nonnegative variable \(\lambda_L\), which acts as an adaptive threshold. Equivalently, one may view the fast-time update as a proximal-gradient step for the \(\ell_1\)-penalized objective associated with the Lagrangian
\[
\mathcal J_t(\vy)+\lambda_L(\|\vy\|_1-1),
\qquad \lambda_L\ge 0.
\]
Using the proximal mapping of the \(\ell_1\)-norm \citep{parikh2014proximal}, the resulting updates are
\begin{align*}
    \tilde y_k(t;\tau+1)
    &=
    y_k(t;\tau)-\eta_y(t,\tau)\,g_k(t;\tau),\\
    y_k(t;\tau+1)
    &=
    \operatorname{ST}_{\lambda_L(t;\tau)}\!\bigl(\tilde y_k(t;\tau+1)\bigr),
    \qquad k=1,\dots,n,\\
    \lambda_L(t;\tau+1)
    &=
    \operatorname{ReLU}\!\left(
        \lambda_L(t;\tau)
        +
        \eta_{\lambda}(t,\tau)
        \left(
            \sum_{i=1}^n |y_i(t;\tau+1)| - 1
        \right)
    \right),
\end{align*}
where
\(
    \operatorname{ST}_{\lambda}(u)
    :=
    \operatorname{sign}(u)\max\{|u|-\lambda,0\}
\)
is the \textit{soft-thresholding} operator. The resulting architecture contains one additional inhibitory unit encoding \(\lambda_L\), which delivers a common threshold to all output neurons.

\subsection{Nonnegative sparse sources}
\label{app:nn_sparse_sources}

Consider the nonnegative sparse domain
\[
\gB_{1,+} := \gB_1 \cap \mathbb{R}_+^n.
\]
Since \(\vy\ge 0\), the \(\ell_1\)-constraint reduces to \(\bm{1}^\top\vy\le 1\). As in the sparse case, we introduce a shared nonnegative variable \(\lambda_L\) and interpret the fast-time step as a proximal / projected-gradient update for
\[
\mathcal J_t(\vy)+\lambda_L(\bm{1}^\top\vy-1),
\qquad \lambda_L\ge 0,
\]
followed by projection onto the nonnegative orthant \citep{parikh2014proximal}. The resulting updates are
\begin{align*}
    \tilde y_k(t;\tau+1)
    &=
    y_k(t;\tau)-\eta_y(t,\tau)\,g_k(t;\tau),\\
    y_k(t;\tau+1)
    &=
    \operatorname{ReLU}\!\bigl(\tilde y_k(t;\tau+1)-\lambda_L(t;\tau)\bigr),
    \qquad k=1,\dots,n,\\
    \lambda_L(t;\tau+1)
    &=
    \operatorname{ReLU}\!\left(
        \lambda_L(t;\tau)
        +
        \eta_{\lambda}(t,\tau)
        \left(
            \sum_{i=1}^n y_i(t;\tau+1) - 1
        \right)
    \right).
\end{align*}
Thus, the nonnegative sparse dynamics correspond to a thresholded ReLU nonlinearity driven by the shared inhibitory variable \(\lambda_L\).

\subsection{Simplex sources}
\label{app:simplex_sources}

Consider the simplex domain
\[
\Delta := \{\vs \in \mathbb{R}^n : \vs \ge \bm{0},\ \bm{1}^\top \vs = 1\}.
\]
The only difference with respect to the nonnegative sparse case is that the population constraint is now an equality rather than an inequality. Accordingly, the shared offset variable \(\lambda_L\) is no longer constrained to be nonnegative. The same proximal / projected interpretation then yields
\begin{align*}
    \tilde y_k(t;\tau+1)
    &=
    y_k(t;\tau)-\eta_y(t,\tau)\,g_k(t;\tau),\\
    y_k(t;\tau+1)
    &=
    \operatorname{ReLU}\!\bigl(\tilde y_k(t;\tau+1)-\lambda_L(t;\tau)\bigr),
    \qquad k=1,\dots,n,\\
    \lambda_L(t;\tau+1)
    &=
    \lambda_L(t;\tau)
    +
    \eta_{\lambda}(t,\tau)
    \left(
        \sum_{i=1}^n y_i(t;\tau+1) - 1
    \right).
\end{align*}
Thus, the simplex network has the same output architecture as the nonnegative sparse network, except that the inhibitory unit corresponding to \(\lambda_L\) is linear rather than rectified.

\paragraph{Summary.}
All five domains share the same feedforward pathway, the same covariance-driven recurrent interactions, and the same slow synaptic updates. What changes with the source domain is only the output-layer nonlinearity and, for \(\gB_1\), \(\gB_{1,+}\), and \(\Delta\), the inclusion of one shared inhibitory variable enforcing the corresponding population constraint. In particular, box-type domains require only local clipping operations, whereas sparse and simplex domains lead to piecewise-linear thresholding dynamics mediated by a single auxiliary inhibitory unit.

\section{Error analysis and analytical bounds for the second-order Taylor approximation}
\label{sec:TaylorRemainderAnalysis}

In this section, we analyze the approximation error induced by the second-order surrogate in Equation~\ref{eq:DecorrelationObjectiveBSS}. The analysis proceeds in two steps. We first study the Taylor remainder pointwise for a generic covariance matrix, deriving an exact spectral representation together with sharp spectral and norm-based bounds. We then specialize these results to the online covariance sequence used by the algorithm and, finally, use the same bounds to relate the batch surrogate objective to the exact regularized determinant objective.

\subsection{Pointwise remainder representation and bounds}

We begin with a generic covariance matrix. Writing \(\mC=\mD+\mO\), where \(\mD\) collects the diagonal entries of \(\mC\) and \(\mO\) its off-diagonal part, and introducing the regularized diagonal matrix \(\mD^\varepsilon:=\mD+\varepsilon \mI\), we define the normalized perturbation
\[
\mB^\varepsilon:=(\mD^\varepsilon)^{-1/2}\mO(\mD^\varepsilon)^{-1/2}.
\]
This yields an exact spectral representation of the Taylor remainder for \(\log\det(\mC+\varepsilon \mI)\).

\begin{theorem}[Exact spectral representation of the second-order remainder]
\label{thm:logdet_spectral_remainder_main}
Fix \(\varepsilon>0\). Let \(\mC \in \mathbb{R}^{n\times n}\) be a symmetric positive semidefinite matrix, and write
\begin{equation}
\mC = \mD + \mO,
\qquad
\mD = \mathrm{diag}(\mC),
\nonumber
\end{equation}
where \(\mD\) is the diagonal matrix formed from the diagonal of \(\mC\) and \(\mO\) is the off-diagonal part of \(\mC\), so that \(O_{ii}=0\) for all \(i\). Define the regularized diagonal matrix
\begin{equation}
\mD^\varepsilon := \mD + \varepsilon \mI,
\nonumber
\end{equation}
the normalized off-diagonal matrix
\begin{equation}
\mB^\varepsilon := (\mD^\varepsilon)^{-1/2}\mO(\mD^\varepsilon)^{-1/2},
\nonumber
\end{equation}
and let \(\lambda_1,\dots,\lambda_n\) denote the eigenvalues of \(\mB^\varepsilon\). Then:
\begin{enumerate}
    \item \(\mB^\varepsilon\) is symmetric, \(\Tr(\mB^\varepsilon)=0\), and \(\lambda_i>-1\) for all \(i=1,\dots,n\);
    \item the regularized log-determinant admits the exact decomposition
    \begin{equation}
    \log\det(\mC+\varepsilon \mI)
    =
    \sum_{i=1}^n \log (C_{ii}+\varepsilon)
    -
    \frac{1}{2}
    \sum_{i=1}^n \lambda_i^2
    +
    R_2,
    \label{eq:exact_logdet_decomposition}
    \end{equation}
    where the remainder is given exactly by
    \begin{equation}
    R_2
    =
    \sum_{i=1}^n
    \left[
        \log(1+\lambda_i)-\lambda_i+\frac{1}{2}\lambda_i^2
    \right].
    \label{eq:exact_R2_formula}
    \end{equation}
\end{enumerate}
Moreover, the second-order term can be written entrywise as
\begin{equation}
\sum_{i=1}^n \lambda_i^2
=
\Tr((\mB^\varepsilon)^2)
=
\sum_{i=1}^n \sum_{\substack{j=1\\j\neq i}}^n
\frac{C_{ij}^2}{(C_{ii}+\varepsilon)(C_{jj}+\varepsilon)},
\label{eq:entrywise_second_order_term}
\end{equation}
and therefore
\begin{equation}
\log\det(\mC+\varepsilon \mI)
=
\sum_{i=1}^n \log (C_{ii}+\varepsilon)
-
\frac{1}{2}
\sum_{i=1}^n \sum_{\substack{j=1\\j\neq i}}^n
\frac{C_{ij}^2}{(C_{ii}+\varepsilon)(C_{jj}+\varepsilon)}
+
R_2.
\label{eq:entrywise_exact_expansion}
\end{equation}
\end{theorem}

\begin{proof}
Since \(\mC\) is symmetric positive semidefinite, each diagonal entry satisfies \(C_{ii}\ge 0\). Hence \(\mD^\varepsilon=\mD+\varepsilon \mI\) is positive definite and \((\mD^\varepsilon)^{-1/2}\) is well-defined. Because \(\mO=\mC-\mD\) is symmetric and \((\mD^\varepsilon)^{-1/2}\) is diagonal, the matrix
\[
\mB^\varepsilon=(\mD^\varepsilon)^{-1/2}\mO(\mD^\varepsilon)^{-1/2}
\]
is symmetric. Its diagonal entries are zero, since
\[
(\mB^\varepsilon)_{ii}
=
(D_{ii}+\varepsilon)^{-1/2} O_{ii} (D_{ii}+\varepsilon)^{-1/2}
=
0,
\qquad i=1,\dots,n.
\]
Therefore
\begin{equation}
\Tr(\mB^\varepsilon)=0.
\label{eq:trace_B_zero}
\end{equation}

Next, observe that
\begin{equation}
\mC+\varepsilon \mI
=
\mD^\varepsilon+\mO
=
(\mD^\varepsilon)^{1/2}(\mI+\mB^\varepsilon)(\mD^\varepsilon)^{1/2}.
\label{eq:C_factorization}
\end{equation}
Since \(\mC+\varepsilon \mI\) is positive definite and \((\mD^\varepsilon)^{1/2}\) is invertible, Equation~\ref{eq:C_factorization} implies that
\[
\mI+\mB^\varepsilon
=
(\mD^\varepsilon)^{-1/2}(\mC+\varepsilon \mI)(\mD^\varepsilon)^{-1/2}
\]
is also positive definite. Hence all eigenvalues of \(\mI+\mB^\varepsilon\) are strictly positive. If \(\lambda_1,\dots,\lambda_n\) are the eigenvalues of \(\mB^\varepsilon\), then the eigenvalues of \(\mI+\mB^\varepsilon\) are \(1+\lambda_1,\dots,1+\lambda_n\), and therefore \(\lambda_i>-1\) for all \(i=1,\dots,n\).

Using Equation~\ref{eq:C_factorization} and multiplicativity of the determinant, we obtain
\[
\det(\mC+\varepsilon \mI)=\det(\mD^\varepsilon)\det(\mI+\mB^\varepsilon).
\]
Taking logarithms yields
\begin{equation}
\log\det(\mC+\varepsilon \mI)
=
\log\det(\mD^\varepsilon)+\log\det(\mI+\mB^\varepsilon).
\label{eq:logdet_factorization}
\end{equation}
Since \(\mD^\varepsilon\) is diagonal,
\begin{equation}
\log\det(\mD^\varepsilon)=\sum_{i=1}^n \log (C_{ii}+\varepsilon).
\label{eq:logdet_D}
\end{equation}

Now let
\[
\mB^\varepsilon = \mQ \,\mathrm{diag}(\lambda_1,\dots,\lambda_n)\,\mQ^\top
\]
be an eigendecomposition of \(\mB^\varepsilon\), with \(\mQ\) orthogonal. Then
\[
\mI+\mB^\varepsilon
=
\mQ\,\mathrm{diag}(1+\lambda_1,\dots,1+\lambda_n)\,\mQ^\top,
\]
so
\begin{equation}
\log\det(\mI+\mB^\varepsilon)=\sum_{i=1}^n \log(1+\lambda_i).
\label{eq:logdet_IplusB}
\end{equation}
Substituting Equations~\ref{eq:logdet_D} and \ref{eq:logdet_IplusB} into Equation~\ref{eq:logdet_factorization} gives
\begin{equation}
\log\det(\mC+\varepsilon \mI)
=
\sum_{i=1}^n \log (C_{ii}+\varepsilon)
+
\sum_{i=1}^n \log(1+\lambda_i).
\label{eq:pre_remainder}
\end{equation}

Using Equation~\ref{eq:trace_B_zero}, we have \(\sum_{i=1}^n \lambda_i = 0\). Therefore we may add and subtract \(\sum_i \lambda_i - \frac{1}{2}\sum_i \lambda_i^2\) inside Equation~\ref{eq:pre_remainder} to obtain
\[
\log\det(\mC+\varepsilon \mI)
=
\sum_{i=1}^n \log (C_{ii}+\varepsilon)
-
\frac{1}{2}\sum_{i=1}^n \lambda_i^2
+
\sum_{i=1}^n
\left[
\log(1+\lambda_i)-\lambda_i+\frac{1}{2}\lambda_i^2
\right].
\]
This proves Equations~\ref{eq:exact_logdet_decomposition} and \ref{eq:exact_R2_formula}.

It remains to identify the quadratic term in coordinates. Since \(\mB^\varepsilon\) is symmetric,
\[
\sum_{i=1}^n \lambda_i^2=\Tr((\mB^\varepsilon)^2).
\]
Also, because \((\mB^\varepsilon)_{ii}=0\),
\[
\Tr((\mB^\varepsilon)^2)
=
\sum_{i=1}^n \sum_{\substack{j=1\\j\neq i}}^n (\mB^\varepsilon)_{ij}^2.
\]
Finally, for \(i\neq j\),
\[
(\mB^\varepsilon)_{ij}
=
(D_{ii}+\varepsilon)^{-1/2} O_{ij} (D_{jj}+\varepsilon)^{-1/2}
=
\frac{C_{ij}}{\sqrt{(C_{ii}+\varepsilon)(C_{jj}+\varepsilon)}}.
\]
Substituting this identity into the previous display proves Equation~\ref{eq:entrywise_second_order_term}, and Equation~\ref{eq:entrywise_exact_expansion} follows immediately.
\end{proof}

Theorem~\ref{thm:logdet_spectral_remainder_main} shows that the approximation error is determined by the spectrum of the regularized normalized perturbation \(\mB^\varepsilon=(\mD^\varepsilon)^{-1/2}\mO(\mD^\varepsilon)^{-1/2}\). This identifies \(\mB^\varepsilon\) as the natural perturbation variable for assessing the accuracy of the regularized Taylor surrogate. The next result turns the exact representation above into a sharp two-sided bound on the remainder.

\begin{theorem}[Two-sided spectral bound]
\label{thm:sharper_spectral_bound}
Under the assumptions and notation of Theorem~\ref{thm:logdet_spectral_remainder_main}, the remainder \(R_2\) satisfies
\begin{equation}
-\frac{1}{3}
\sum_{\lambda_i<0}
\frac{|\lambda_i|^3}{1+\lambda_i}
\;\le\;
R_2
\;\le\;
\frac{1}{3}
\sum_{\lambda_i\ge 0}
\lambda_i^3.
\label{eq:two_sided_R2_bound}
\end{equation}
Consequently,
\begin{equation}
|R_2|
\le
\max\left\{
\frac{1}{3}
\sum_{\lambda_i\ge 0}\lambda_i^3,\,
\frac{1}{3}
\sum_{\lambda_i<0}\frac{|\lambda_i|^3}{1+\lambda_i}
\right\}.
\label{eq:max_R2_bound}
\end{equation}
\end{theorem}

\begin{proof}
Define the scalar function
\begin{equation}
r(x):=\log(1+x)-x+\frac{1}{2}x^2,
\qquad x>-1.\nonumber
\end{equation}
By Equation~\ref{eq:exact_R2_formula},
\begin{equation}
R_2=\sum_{i=1}^n r(\lambda_i).
\label{eq:R2_as_sum_r}
\end{equation}
Also,
\[
r(0)=0,
\qquad
r'(x)=\frac{x^2}{1+x}.
\]
Hence
\begin{equation}
r(x)=\int_0^x \frac{t^2}{1+t}\,dt,
\qquad x>-1.
\label{eq:r_integral_representation}
\end{equation}

If \(x\ge 0\), then \(1+t\ge 1\) for all \(t\in[0,x]\), and Equation~\ref{eq:r_integral_representation} gives
\[
0\le r(x)=\int_0^x \frac{t^2}{1+t}\,dt \le \int_0^x t^2\,dt = \frac{x^3}{3}.
\]
Thus
\begin{equation}
0\le r(x)\le \frac{x^3}{3},
\qquad x\ge 0.
\label{eq:r_positive}
\end{equation}

If \(-1<x<0\), then
\[
r(x)= - \int_x^0 \frac{t^2}{1+t}\,dt \le 0.
\]
Moreover, for \(t\in[x,0]\), we have \(1+t\ge 1+x>0\), so
\[
\frac{1}{1+t}\le \frac{1}{1+x}.
\]
Therefore
\[
|r(x)|
=
-r(x)
=
\int_x^0 \frac{t^2}{1+t}\,dt
\le
\frac{1}{1+x}\int_x^0 t^2\,dt
=
\frac{|x|^3}{3(1+x)}.
\]
Hence
\begin{equation}
-\frac{|x|^3}{3(1+x)}
\le
r(x)
\le
0,
\qquad -1<x<0.
\label{eq:r_negative}
\end{equation}

Applying Equation~\ref{eq:r_positive} to the nonnegative eigenvalues and Equation~\ref{eq:r_negative} to the negative eigenvalues, and summing the resulting inequalities in Equation~\ref{eq:R2_as_sum_r}, yields Equation~\ref{eq:two_sided_R2_bound}. Equation~\ref{eq:max_R2_bound} follows immediately.
\end{proof}

The preceding spectral bound immediately yields a simpler estimate expressed directly in terms of the size and conditioning of the regularized normalized perturbation \(\mB^\varepsilon\).

\begin{corollary}[Simpler norm-based bound]
\label{corr:norm_based_bound}
Under the assumptions and notation of Theorem~\ref{thm:logdet_spectral_remainder_main}, let \(\lambda_{\min}(\mB^\varepsilon)\) denote the smallest eigenvalue of the normalized off-diagonal matrix \(\mB^\varepsilon\). Then
\begin{equation}
|R_2|
\le
\frac{1}{3\bigl(1+\lambda_{\min}(\mB^\varepsilon)\bigr)}
\sum_{i=1}^n |\lambda_i|^3
\le
\frac{\|\mB^\varepsilon\|_F^2\|\mB^\varepsilon\|_2}{3\bigl(1+\lambda_{\min}(\mB^\varepsilon)\bigr)}.
\label{eq:norm_based_bound}
\end{equation}
\end{corollary}

\begin{proof}
From Theorem~\ref{thm:sharper_spectral_bound}, we have
\begin{equation}
|R_2|
\le
\max\left\{
\frac{1}{3}\sum_{\lambda_i\ge 0}\lambda_i^3,\,
\frac{1}{3}\sum_{\lambda_i<0}\frac{|\lambda_i|^3}{1+\lambda_i}
\right\}.
\label{eq:max_bound_for_corollary}
\end{equation}
We bound the two terms on the right-hand side separately.

First, since \(\Tr(\mB^\varepsilon)=0\), either \(\mB^\varepsilon=\mathbf{0}\) or \(\lambda_{\min}(\mB^\varepsilon)\le 0\). In either case,
\[
1+\lambda_{\min}(\mB^\varepsilon)\le 1.
\]
Therefore, for every eigenvalue with \(\lambda_i\ge 0\),
\[
\lambda_i^3
\le
\frac{\lambda_i^3}{1+\lambda_{\min}(\mB^\varepsilon)}
\le
\frac{|\lambda_i|^3}{1+\lambda_{\min}(\mB^\varepsilon)}.
\]
Hence
\begin{equation}
\frac{1}{3}\sum_{\lambda_i\ge 0}\lambda_i^3
\le
\frac{1}{3\bigl(1+\lambda_{\min}(\mB^\varepsilon)\bigr)}
\sum_{\lambda_i\ge 0} |\lambda_i|^3.
\label{eq:positive_part_bound}
\end{equation}

Second, for every eigenvalue with \(\lambda_i<0\), we have \(\lambda_i\ge \lambda_{\min}(\mB^\varepsilon)\), so
\[
1+\lambda_i \ge 1+\lambda_{\min}(\mB^\varepsilon)>0.
\]
Consequently,
\[
\frac{|\lambda_i|^3}{1+\lambda_i}
\le
\frac{|\lambda_i|^3}{1+\lambda_{\min}(\mB^\varepsilon)}.
\]
Summing over all negative eigenvalues gives
\begin{equation}
\frac{1}{3}\sum_{\lambda_i<0}\frac{|\lambda_i|^3}{1+\lambda_i}
\le
\frac{1}{3\bigl(1+\lambda_{\min}(\mB^\varepsilon)\bigr)}
\sum_{\lambda_i<0} |\lambda_i|^3.
\label{eq:negative_part_bound}
\end{equation}

Combining Equations~\ref{eq:positive_part_bound} and \ref{eq:negative_part_bound} with Equation~\ref{eq:max_bound_for_corollary}, we obtain
\[
|R_2|
\le
\frac{1}{3\bigl(1+\lambda_{\min}(\mB^\varepsilon)\bigr)}
\sum_{i=1}^n |\lambda_i|^3.
\]

For the second inequality, note that
\[
\sum_{i=1}^n |\lambda_i|^3
=
\sum_{i=1}^n |\lambda_i|^2 |\lambda_i|
\le
\left(\max_{1\le i\le n} |\lambda_i|\right)\sum_{i=1}^n \lambda_i^2
=
\|\mB^\varepsilon\|_2 \|\mB^\varepsilon\|_F^2.
\]
Substituting this into the previous inequality proves Equation~\ref{eq:norm_based_bound}.
\end{proof}

We now specialize the preceding pointwise analysis to the online covariance sequence generated by the algorithm.

\begin{corollary}[Online version]
\label{corr:OnlineErrorBound}
Consider the online \textit{Predictive Entropy Maximization} algorithm, and suppose that at time \(t\) the exponentially weighted output covariance \(\hat{\mC}^{\lambda}(t)\) is symmetric positive semidefinite. Let
\begin{equation}
\hat{\mD}^{\lambda,\varepsilon}(t)
:=
\hat{\mD}^{\lambda}(t)+\varepsilon \mI,
\qquad
\hat{\mC}^{\lambda}(t)
=
\hat{\mD}^{\lambda}(t)
+
\hat{\mO}^{\lambda}(t),
\nonumber
\end{equation}
and define the normalized off-diagonal matrix
\begin{equation}
\hat{\mB}^{\lambda,\varepsilon}(t)
:=
\bigl(\hat{\mD}^{\lambda,\varepsilon}(t)\bigr)^{-1/2}
\hat{\mO}^{\lambda}(t)
\bigl(\hat{\mD}^{\lambda,\varepsilon}(t)\bigr)^{-1/2}.
\nonumber
\end{equation}
Let \(\hat{\lambda}_1(t),\dots,\hat{\lambda}_n(t)\) denote the eigenvalues of \(\hat{\mB}^{\lambda,\varepsilon}(t)\). Then the remainder \(R_2(t)\) of the second-order surrogate in Equation~\ref{eq:DecorrelationObjectiveBSS} satisfies
\begin{equation}
-\frac{1}{3}
\sum_{\hat{\lambda}_i(t)<0}
\frac{|\hat{\lambda}_i(t)|^3}{1+\hat{\lambda}_i(t)}
\;\le\;
R_2(t)
\;\le\;
\frac{1}{3}
\sum_{\hat{\lambda}_i(t)\ge 0}
\hat{\lambda}_i(t)^3,
\nonumber
\end{equation}
and therefore
\begin{equation}
|R_2(t)|
\le
\max\left\{
\frac{1}{3}
\sum_{\hat{\lambda}_i(t)\ge 0}\hat{\lambda}_i(t)^3,\,
\frac{1}{3}
\sum_{\hat{\lambda}_i(t)<0}\frac{|\hat{\lambda}_i(t)|^3}{1+\hat{\lambda}_i(t)}
\right\}.
\nonumber
\end{equation}
\end{corollary}

\begin{proof}
Apply Theorem~\ref{thm:logdet_spectral_remainder_main} and Theorem~\ref{thm:sharper_spectral_bound} with
\[
\mC=\hat{\mC}^{\lambda}(t),
\qquad
\mD=\hat{\mD}^{\lambda}(t),
\qquad
\mD^\varepsilon=\hat{\mD}^{\lambda,\varepsilon}(t),
\qquad
\mO=\hat{\mO}^{\lambda}(t).
\]
Since \(\hat{\mC}^{\lambda}(t)\) is symmetric positive semidefinite and \(\varepsilon>0\), all assumptions are satisfied.
\end{proof}

Corollary~\ref{corr:OnlineErrorBound} is the bound used in our numerical diagnostics. It shows that the approximation error is controlled by the spectrum of the regularized normalized off-diagonal covariance matrix \(\hat{\mB}^{\lambda,\varepsilon}(t)\). In particular, the surrogate is accurate when the diagonally normalized off-diagonal covariance is small and the eigenvalues of \(\hat{\mB}^{\lambda,\varepsilon}(t)\) stay well away from the singular value \(-1\).

\subsection{Batch surrogate optimality relative to the exact determinant objective}
\label{app:batch_surrogate_optimality}

The previous subsection controls the Taylor remainder pointwise for a fixed covariance matrix. We now use that control at the objective level. Although the surrogate in Section~\ref{sec:DecorrelationObjective} is introduced through the online exponentially weighted covariance \(\hat{\mC}^{\lambda}(t)\), the expansion in Theorem~\ref{thm:logdet_spectral_remainder_main} and the remainder bounds above are pointwise statements about an arbitrary covariance matrix. It is therefore useful to consider the corresponding \emph{batch} counterpart, obtained by replacing \(\hat{\mC}^{\lambda}(t)\) with the ordinary centered sample covariance. This auxiliary analysis does not address the convergence of the online dynamics; rather, it clarifies what minimizing the surrogate preserves from the exact regularized determinant objective itself.

Let \(\mathcal Y\) be a nonempty family of feasible output matrices \(\mY=[\vy(1),\dots,\vy(T)]\in\mathbb{R}^{n\times T}\) such that \(\vy(t)\in\mathcal P\) for all \(t\), and such that the centered sample covariance
\begin{equation}
    \hat{\mC}_\vy(\mY)
    :=
    \frac{1}{T}\sum_{t=1}^{T}
    \bigl(\vy(t)-\hat{\bm{\mu}}^{\lambda}(\mY)\bigr)
    \bigl(\vy(t)-\hat{\bm{\mu}}^{\lambda}(\mY)\bigr)^\top,
    \qquad
    \hat{\bm{\mu}}^{\lambda}(\mY)
    :=
    \frac{1}{T}\sum_{t=1}^{T}\vy(t),\nonumber
\end{equation}
is symmetric positive semidefinite for every \(\mY\in\mathcal Y\). For each \(\mY\in\mathcal Y\), write
\[
\mD(\mY):=\text{diag}\!\bigl(\hat{\mC}_\vy(\mY)\bigr),
\qquad
\mD^\varepsilon(\mY):=\mD(\mY)+\varepsilon \mI,
\qquad
\mO(\mY):=\hat{\mC}_\vy(\mY)-\mD(\mY),
\]
\[
\mB^\varepsilon(\mY):=\mD^\varepsilon(\mY)^{-1/2}\mO(\mY)\mD^\varepsilon(\mY)^{-1/2},
\qquad
\hat v_i(\mY):=[\hat{\mC}_\vy(\mY)]_{ii},
\qquad
\hat c_{ij}(\mY):=[\hat{\mC}_\vy(\mY)]_{ij}.
\]

We consider the exact batch determinant objective
\begin{equation}
    \mathcal{J}_{\mathrm{det}}^{\mathrm{batch}}(\mY)
    :=
    -\log\det\!\bigl(\hat{\mC}_\vy(\mY)+\varepsilon \mI\bigr),
    \label{eq:exact_batch_det_objective}
\end{equation}
and its second-order batch surrogate
\begin{equation}
    \mathcal{J}_{\mathrm{sur}}^{\mathrm{batch}}(\mY)
    :=
    -\sum_{i=1}^{n}\log \bigl(\hat v_i(\mY)+\varepsilon\bigr)
    +
    \frac{1}{2}
    \sum_{i=1}^{n}
    \sum_{\substack{j=1\\j\neq i}}^{n}
    \frac{\hat c_{ij}(\mY)^2}{\bigl(\hat v_i(\mY)+\varepsilon\bigr)\bigl(\hat v_j(\mY)+\varepsilon\bigr)}.\nonumber
\end{equation}

The next result shows that if the regularized normalized off-diagonal perturbation remains uniformly controlled over the feasible family \(\mathcal Y\), then every global minimizer of the surrogate is nearly optimal for the exact regularized determinant objective.

\begin{theorem}[Uniform approximation and exact-objective near-optimality]
\label{thm:batch_surrogate_near_optimality}
Define
\begin{equation}
    \bar{\varepsilon}_{\mathcal Y}
    :=
    \sup_{\mY\in\mathcal Y}
    \frac{\|\mB^\varepsilon(\mY)\|_F^2\|\mB^\varepsilon(\mY)\|_2}
         {3\bigl(1+\lambda_{\min}(\mB^\varepsilon(\mY))\bigr)}.
    \label{eq:uniform_remainder_constant_batch}
\end{equation}
Assume \(\bar{\varepsilon}_{\mathcal Y}<\infty\). Then:
\begin{enumerate}
    \item for every \(\mY\in\mathcal Y\),
    \begin{equation}
        \bigl|
        \mathcal{J}_{\mathrm{sur}}^{\mathrm{batch}}(\mY)
        -
        \mathcal{J}_{\mathrm{det}}^{\mathrm{batch}}(\mY)
        \bigr|
        \le
        \bar{\varepsilon}_{\mathcal Y};
        \label{eq:uniform_gap_batch_objectives}
    \end{equation}
    \item if \(\mY_{\mathrm{sur}}\) is a global minimizer of \(\mathcal{J}_{\mathrm{sur}}^{\mathrm{batch}}\) over \(\mathcal Y\), then
    \begin{equation}
        \mathcal{J}_{\mathrm{det}}^{\mathrm{batch}}(\mY_{\mathrm{sur}})
        \le
        \inf_{\mY\in\mathcal Y}\mathcal{J}_{\mathrm{det}}^{\mathrm{batch}}(\mY)
        +
        2\bar{\varepsilon}_{\mathcal Y}.
        \label{eq:exact_batch_near_optimality}
    \end{equation}
\end{enumerate}
\end{theorem}

\begin{proof}
By Equation~\ref{eq:entrywise_exact_expansion}, applied with \(\mC=\hat{\mC}_\vy(\mY)\), the difference between the surrogate and exact batch objectives is exactly the Taylor remainder:
\[
\mathcal{J}_{\mathrm{sur}}^{\mathrm{batch}}(\mY)
-
\mathcal{J}_{\mathrm{det}}^{\mathrm{batch}}(\mY)
=
R_2(\mY).
\]
Applying Corollary~\ref{corr:norm_based_bound} with \(\mC=\hat{\mC}_\vy(\mY)\) yields
\[
|R_2(\mY)|
\le
\frac{\|\mB^\varepsilon(\mY)\|_F^2\|\mB^\varepsilon(\mY)\|_2}
     {3\bigl(1+\lambda_{\min}(\mB^\varepsilon(\mY))\bigr)}
\le
\bar{\varepsilon}_{\mathcal Y},
\]
which proves Equation~\ref{eq:uniform_gap_batch_objectives}.

Now let \(\mY_{\mathrm{sur}}\) be a global minimizer of \(\mathcal{J}_{\mathrm{sur}}^{\mathrm{batch}}\) over \(\mathcal Y\). Then
\[
\mathcal{J}_{\mathrm{det}}^{\mathrm{batch}}(\mY_{\mathrm{sur}})
\le
\mathcal{J}_{\mathrm{sur}}^{\mathrm{batch}}(\mY_{\mathrm{sur}})
+
\bar{\varepsilon}_{\mathcal Y}
=
\inf_{\mY\in\mathcal Y}\mathcal{J}_{\mathrm{sur}}^{\mathrm{batch}}(\mY)
+
\bar{\varepsilon}_{\mathcal Y}.
\]
Using Equation~\ref{eq:uniform_gap_batch_objectives} once more,
\[
\inf_{\mY\in\mathcal Y}\mathcal{J}_{\mathrm{sur}}^{\mathrm{batch}}(\mY)
\le
\inf_{\mY\in\mathcal Y}
\left(
\mathcal{J}_{\mathrm{det}}^{\mathrm{batch}}(\mY)
+
\bar{\varepsilon}_{\mathcal Y}
\right)
=
\inf_{\mY\in\mathcal Y}\mathcal{J}_{\mathrm{det}}^{\mathrm{batch}}(\mY)
+
\bar{\varepsilon}_{\mathcal Y}.
\]
Combining the last two displays proves Equation~\ref{eq:exact_batch_near_optimality}.
\end{proof}

Theorem~\ref{thm:batch_surrogate_near_optimality} shows that the Taylor surrogate does more than approximate the exact regularized determinant objective pointwise: if the regularized normalized off-diagonal perturbation is uniformly small on the feasible family, then minimizing the surrogate yields an output matrix whose exact determinant objective value is close to the batch optimum.

\begin{corollary}[Explicit spectral form]
\label{cor:batch_surrogate_near_optimality_rho}
Assume that there exists \(\rho\in(0,1)\) such that
\begin{equation}
    \|\mB^\varepsilon(\mY)\|_2\le \rho,
    \qquad
    \forall\,\mY\in\mathcal Y.
    \label{eq:uniform_spectral_control_batch}
\end{equation}
Then
\begin{equation}
    \bar{\varepsilon}_{\mathcal Y}
    \le
    \frac{n\rho^3}{3(1-\rho)}.
    \nonumber
\end{equation}
Consequently, every global minimizer \(\mY_{\mathrm{sur}}\) of \(\mathcal{J}_{\mathrm{sur}}^{\mathrm{batch}}\) over \(\mathcal Y\) satisfies
\begin{equation}
    \mathcal{J}_{\mathrm{det}}^{\mathrm{batch}}(\mY_{\mathrm{sur}})
    \le
    \inf_{\mY\in\mathcal Y}\mathcal{J}_{\mathrm{det}}^{\mathrm{batch}}(\mY)
    +
    \frac{2n\rho^3}{3(1-\rho)}.
    \label{eq:exact_batch_near_optimality_rho}
\end{equation}
\end{corollary}

\begin{proof}
Under Equation~\ref{eq:uniform_spectral_control_batch}, we have
\[
\lambda_{\min}(\mB^\varepsilon(\mY))\ge -\rho,
\qquad
\|\mB^\varepsilon(\mY)\|_F^2\le n\rho^2,
\qquad
\forall\,\mY\in\mathcal Y.
\]
Substituting these inequalities into Equation~\ref{eq:uniform_remainder_constant_batch} gives
\[
\bar{\varepsilon}_{\mathcal Y}
\le
\frac{n\rho^2\cdot \rho}{3(1-\rho)}
=
\frac{n\rho^3}{3(1-\rho)}.
\]
The bound in Equation~\ref{eq:exact_batch_near_optimality_rho} then follows directly from Theorem~\ref{thm:batch_surrogate_near_optimality}.
\end{proof}

\section{Supplementary on numerical experiments}
\label{appendix:Supplementary_Numerical_Experiments}

This section provides technical implementation details, extended experimental results, and additional information on the experimental protocols used throughout the paper.

\subsection{Performance evaluation metric}

To evaluate source recovery performance, we report the mean signal-to-noise ratio (mSNR) of the estimated sources. Let \(\vs_i\) and \(\Tilde{\vy}_i\) denote the \(i\)-th ground-truth source and its corresponding permutation- and sign-corrected estimate, respectively. The mSNR is defined as
\begin{equation}
    \mathrm{mSNR\ (dB)}
    =
    \frac{1}{n}\sum_{i=1}^n
    10\log_{10}
    \left(
        \frac{\|\vs_i\|_2^2}{\|\vs_i-\Tilde{\vy}_i\|_2^2}
    \right).
    \label{eq:SNR_Mean_Definition}
\end{equation}

To assess statistical variability, all reported metrics are aggregated over \(N=30\) independent realizations with different random mixing matrices and noise seeds. We report the sample mean \(\bar{\mu}\) together with the \(95\%\) confidence interval. Let \(\sigma\) denote the sample standard deviation and \(SE=\sigma/\sqrt{N}\) the standard error of the mean. The confidence interval is computed as
\begin{equation}
    \mathrm{CI}_{95\%}
    =
    \bar{\mu}
    \pm
    t_{\alpha/2,N-1}
    \left(
        \frac{\sigma}{\sqrt{N}}
    \right),
    \label{eq:CI_Definition}
\end{equation}
where \(t_{\alpha/2,N-1}\) is the critical value of the Student \(t\)-distribution with \(\alpha=0.05\) and \(N-1\) degrees of freedom. For \(N=30\), this gives \(t_{0.025,29}\approx 2.045\). All shaded envelopes in the reported plots correspond to this \(95\%\) confidence interval.

\subsection{Transform invariance in auditory source separation}

In Section \ref{sec:NumericalExperiments}, we applied \textit{Predictive Entropy Maximization} to audio mixtures in a sparse wavelet domain. The justification is straightforward. In the noise-free setting, let
\[
\mX=\mA\mS,
\qquad
\mX\in\mathbb{R}^{m\times T},
\qquad
\mS\in\mathbb{R}^{n\times T},
\]
and let \(\mathbf{\Phi}\in\mathbb{R}^{T\times T}\) denote a linear transform acting on the sample axis, such as the discrete wavelet transform. The transformed mixtures satisfy
\begin{align*}
    \tilde{\mX}
    &=
    \mX\mathbf{\Phi}
    =
    (\mA\mS)\mathbf{\Phi}
    =
    \mA(\mS\mathbf{\Phi})
    =
    \mA\tilde{\mS},
\end{align*}
where \(\tilde{\mS}:=\mS\mathbf{\Phi}\). Thus, the same mixing matrix \(\mA\) governs the linear relation in the transform domain. This preserves the mixing geometry while allowing us to exploit the approximate sparsity of wavelet coefficients.

Accordingly, we learn a separator \(\mW\) such that \(\mW\tilde{\mX}\approx \tilde{\mS}\) using the \(\gB_1\) architecture. After separation, the time-domain sources are recovered through the inverse transform:
\begin{align*}
    \hat{\mS}
    =
    (\mW\tilde{\mX})\mathbf{\Phi}^{-1}
    \approx
    \tilde{\mS}\mathbf{\Phi}^{-1}
    =
    \mS.
\end{align*}

For the specific trial illustrated in Figure~\ref{fig:sound_sep_appendix}, the mixing matrix was
\[
    \mA
    =
    \begin{bmatrix}
        -1.131 & 0.696 & -0.432 \\
         0.741 & -0.478 & 1.386 \\
         0.125 & 1.149 & -2.350 \\
         0.183 & -0.311 & -0.294 \\
         0.400 & 1.006 & 0.502
    \end{bmatrix},
\]
and the resulting source-wise SNR values were \(25.76\) dB, \(24.36\) dB, and \(30.54\) dB.

\begin{figure}[t]
    \centering
    \includegraphics[width=0.78\linewidth]{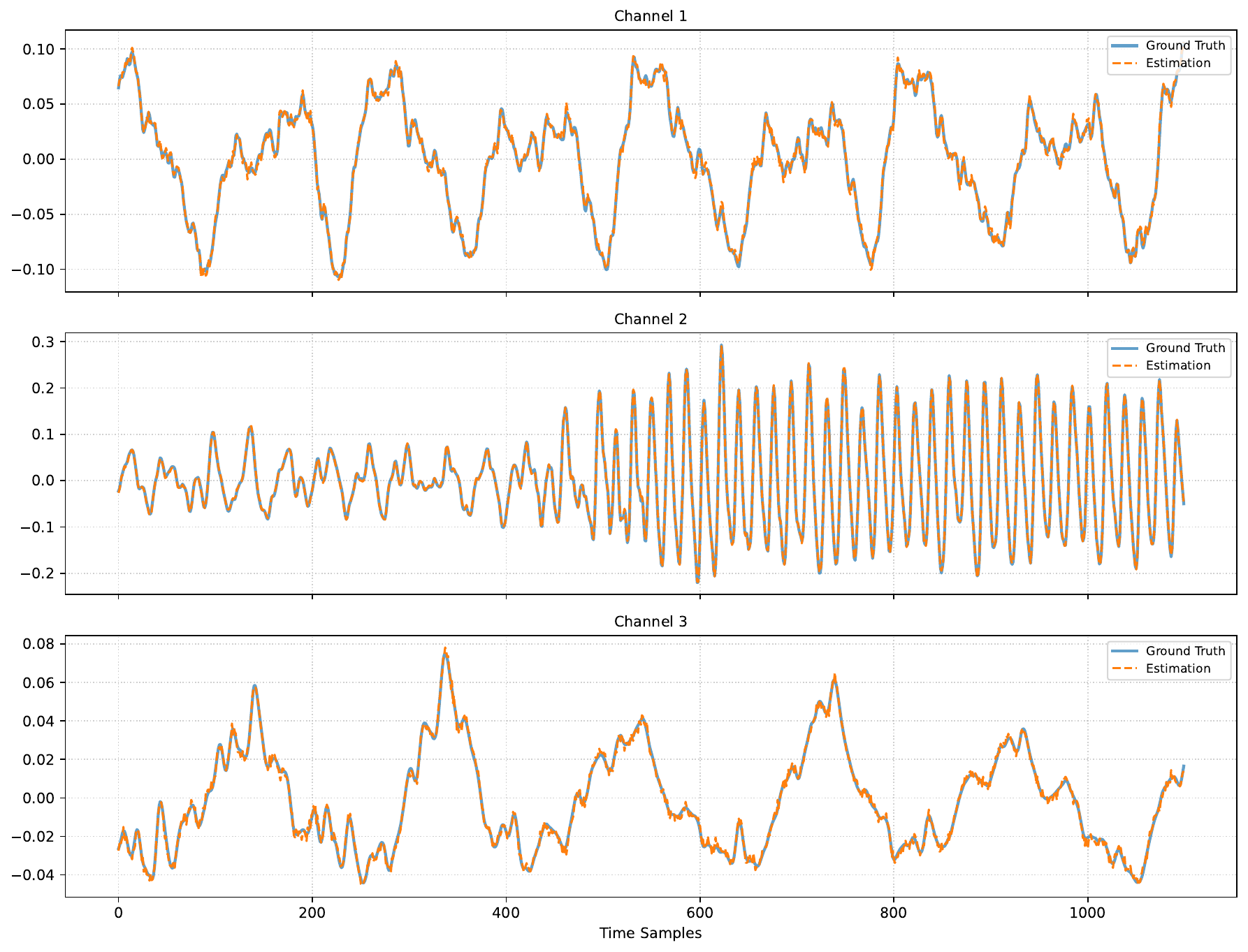}
    \caption{\textbf{Representative auditory source-separation result.}
    Temporal alignment between the ground-truth and recovered sources for one trial in the cocktail-party experiment described in Section~\ref{sec:NumericalExperiments}.}
    \label{fig:sound_sep_appendix}
\end{figure}

\subsection{Antisparse, nonnegative sparse, and simplex source separation examples}

Complementing the representative results in Figure~\ref{fig:numerical_experiments_main}, Figure~\ref{fig:numerical_experiments_appendix} reports three additional synthetic domains.

For the antisparse experiment, we use the domain \(\gB_{\mathrm{max}}\). As in the nonnegative antisparse case from the main text, sources are generated from a Copula-\(t\) model with four degrees of freedom, the source correlation level varies over \(\rho\in\{0,0.05,\dots,0.5\}\), and the input SNR is fixed at \(30\) dB. Figure~\ref{fig:antisparse_corr_appendix} shows that both \textit{PEM} and \textit{u-PEM} remain robust as correlation increases. In particular, they maintain strong performance relative to the biologically plausible baselines, while the batch \textit{LD-InfoMax} method again provides the strongest overall reference performance. While \textit{u-PEM} performs similarly to \textit{PEM}, its performance declines slightly faster with increasing source correlation.

For the nonnegative sparse experiment, we use the domain \(\gB_{1,+}\) and vary the observation noise over \(\mathrm{SNR}_{\mathrm{in}}\in\{30,25,\dots,5\}\) dB. Sources are generated uniformly from \(\gB_{1,+}\). Figure~\ref{fig:nnsparse_noisy_appendix} shows that \textit{PEM} and \textit{u-PEM} remain competitive with the batch \textit{LD-InfoMax} baseline across the full noise range, confirming that the surrogate-based online formulation extends naturally to nonnegative sparse domains as well.

For the simplex experiment, we use the domain \(\Delta\) and vary the observation noise over \(\mathrm{SNR}_{\mathrm{in}}\in\{30,25,\dots,5\}\) dB. Sources are generated uniformly from \(\Delta\). Figure~\ref{fig:simplex_noisy_appendix} shows that \textit{PEM} remains competitive with the batch \textit{LD-InfoMax} baseline across the full noise range, confirming that the surrogate-based online formulation extends beyond box-constrained and sparse domains to simplex-structured sources as well. u-PEM also follows the other models closely, although it performs noticeably worse at high SNR values.

\begin{figure}[t!]
    \centering
    \begin{subfigure}{0.32\linewidth}
        \centering
        \includegraphics[width=\linewidth]{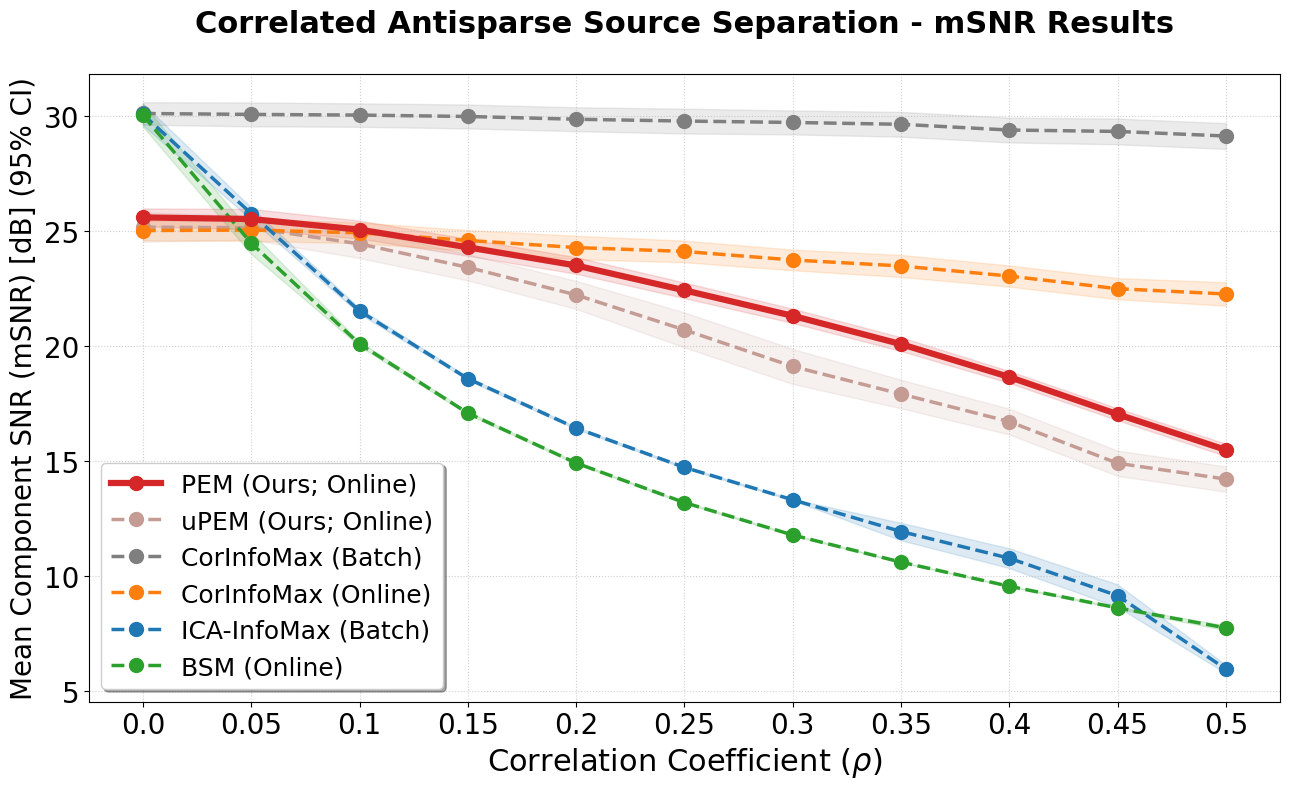}
        \caption{Correlated antisparse (\(\gB_{\mathrm{max}}\))}
\label{fig:antisparse_corr_appendix}
    \end{subfigure}
    \hfill
    \begin{subfigure}{0.32\linewidth}
        \centering
        \includegraphics[width=\linewidth]{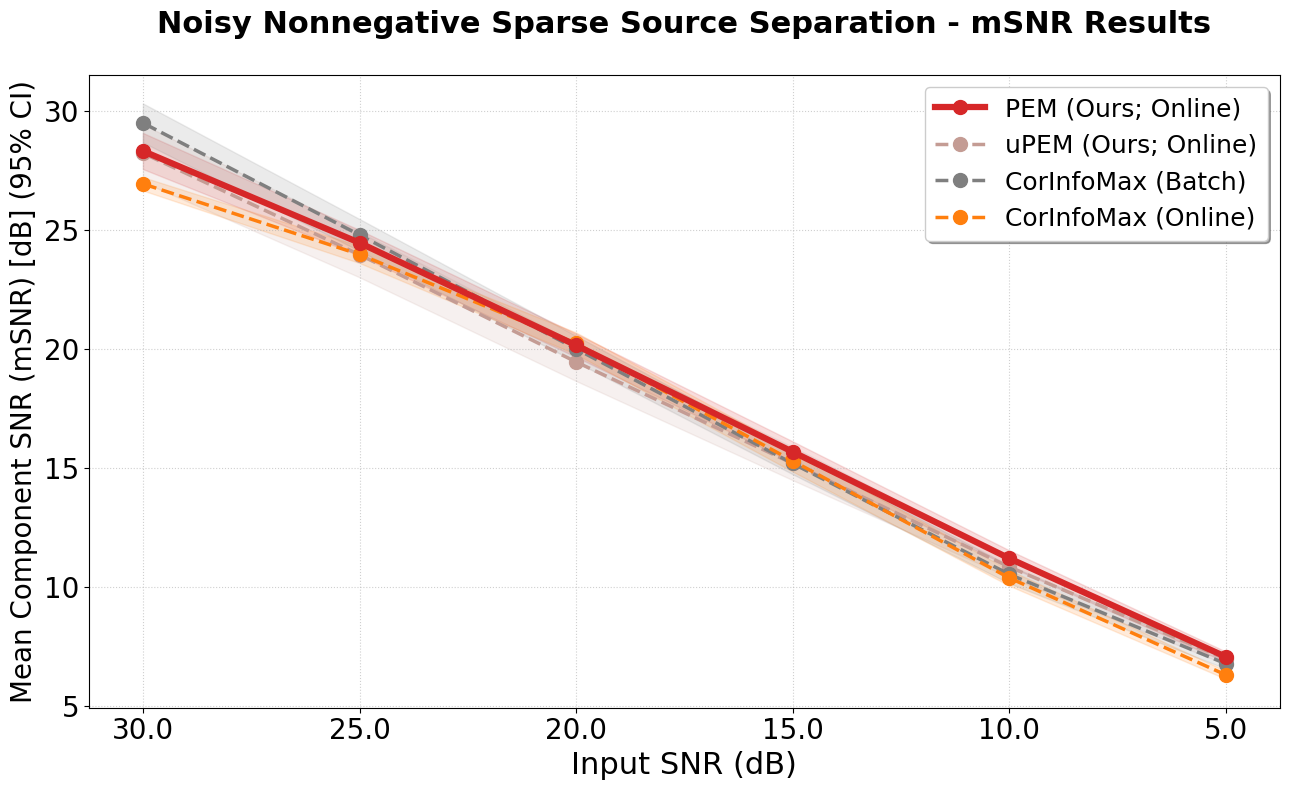}
        \caption{Noisy nonneg. sparse (\(\gB_{1,+}\))}
        \label{fig:nnsparse_noisy_appendix}
    \end{subfigure}
    \hfill
    \begin{subfigure}{0.32\linewidth}
        \centering
        \includegraphics[width=\linewidth]{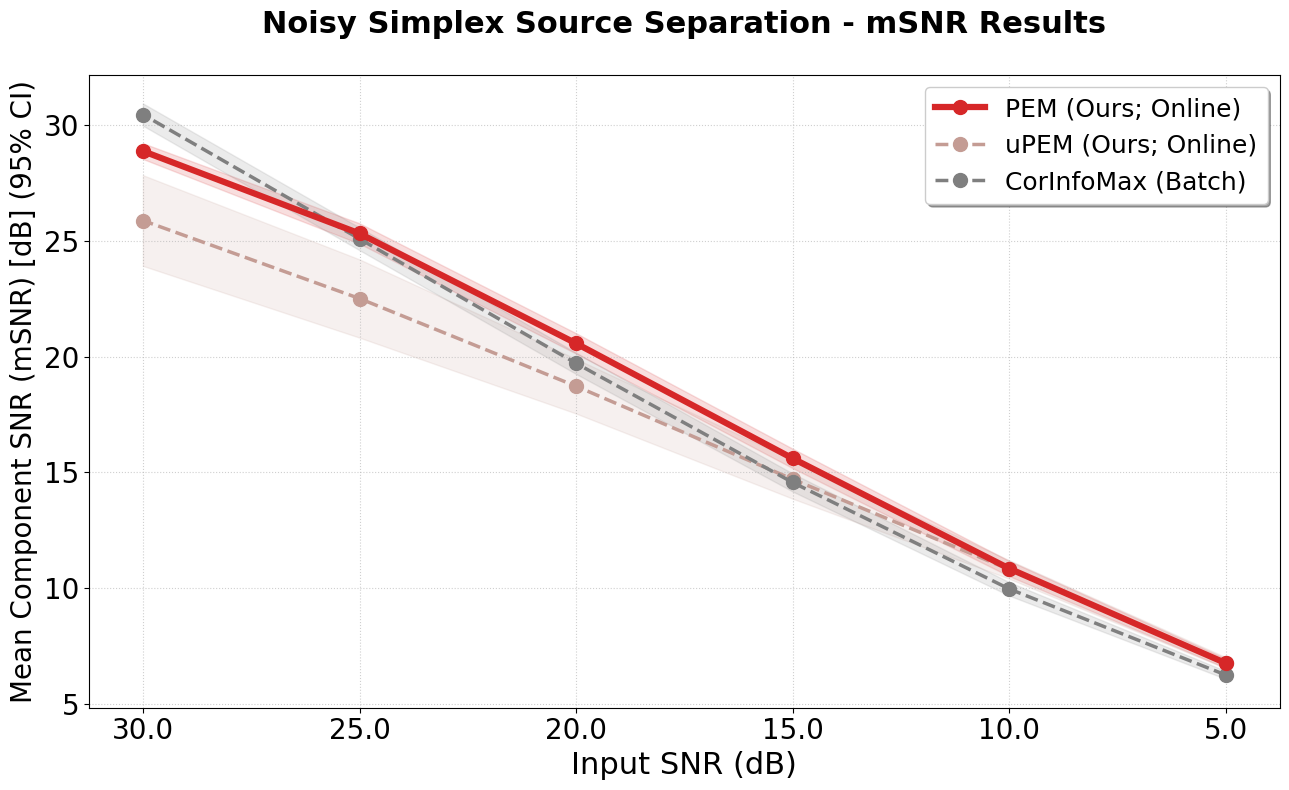}
        \caption{Noisy simplex (\(\Delta\))}
        \label{fig:simplex_noisy_appendix}
    \end{subfigure}
    
    \caption{\textbf{Additional performance comparisons.}
    \textbf{(a)} Mean component SNR (mSNR) as a function of the source correlation level \(\rho\) for antisparse sources.
    \textbf{(b)} Mean component SNR as a function of the input SNR for nonnegative sparse sources.
    \textbf{(c)} Mean component SNR as a function of the input SNR for simplex sources.
    These complementary experiments confirm the same pattern observed in the main text: \textit{Predictive Entropy Maximization} remains robust under source correlation in antisparse domains and remains competitive with batch Det-Max baselines under increasing observation noise in nonnegative sparse- and simplex-structured domains. Shaded envelopes show \(95\%\) confidence intervals over \(30\) independent realizations.}
    \label{fig:numerical_experiments_appendix}
\end{figure}

\subsection{Transient diagnostics for the Taylor approximation}
\label{appendix:taylor_transient_diagnostics}

Figure~\ref{fig:taylor_transient_appendix} reports the time evolution of the exact Taylor remainder and the corresponding spectral upper bound for two representative correlation levels, \(\rho=0\) and \(\rho=0.4\). In both cases, the approximation error decreases during training as the dynamics suppress normalized off-diagonal covariance structure. The correlated case exhibits a larger error scale, as expected, but remains controlled by the theoretical bound throughout training. These trajectories complement the aggregate Taylor-surrogate diagnostics reported in the main text.

\begin{figure}[t]
    \centering
    \begin{subfigure}{0.48\linewidth}
        \centering
        \includegraphics[width=\linewidth]{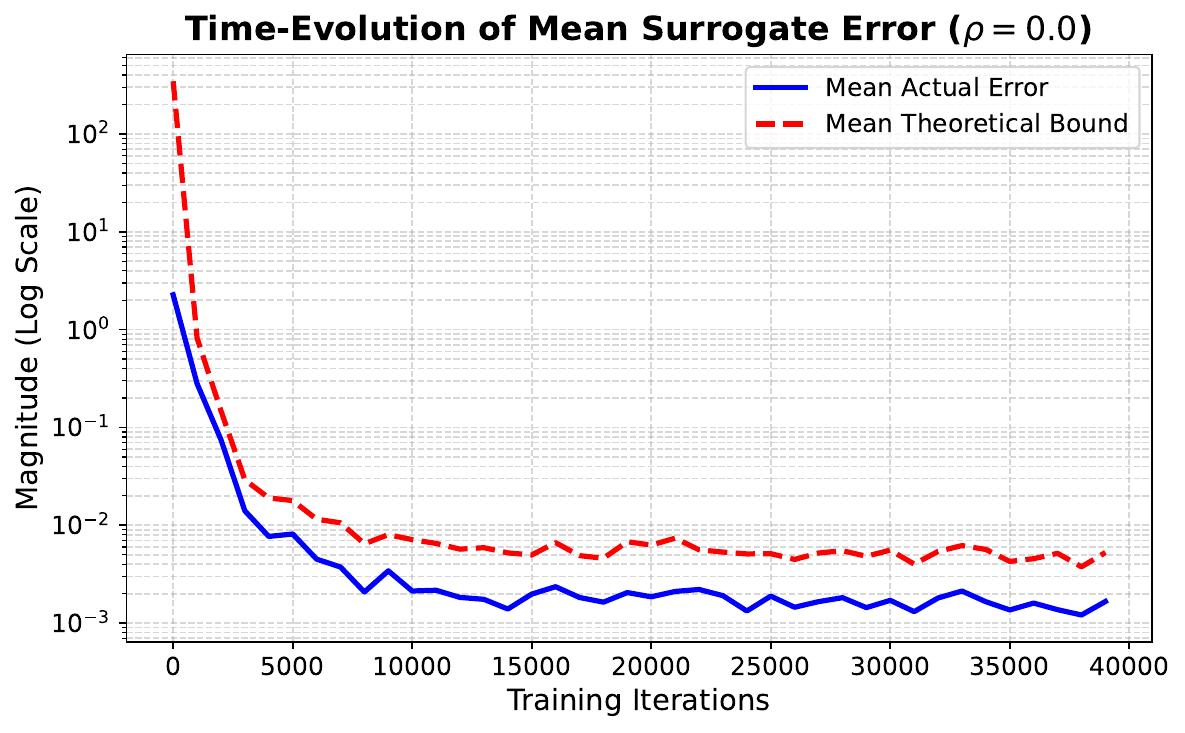}
        \caption{Uncorrelated sources (\(\rho=0\))}
        \label{fig:taylor_transient_rho0}
    \end{subfigure}
    \hfill
    \begin{subfigure}{0.48\linewidth}
        \centering
        \includegraphics[width=\linewidth]{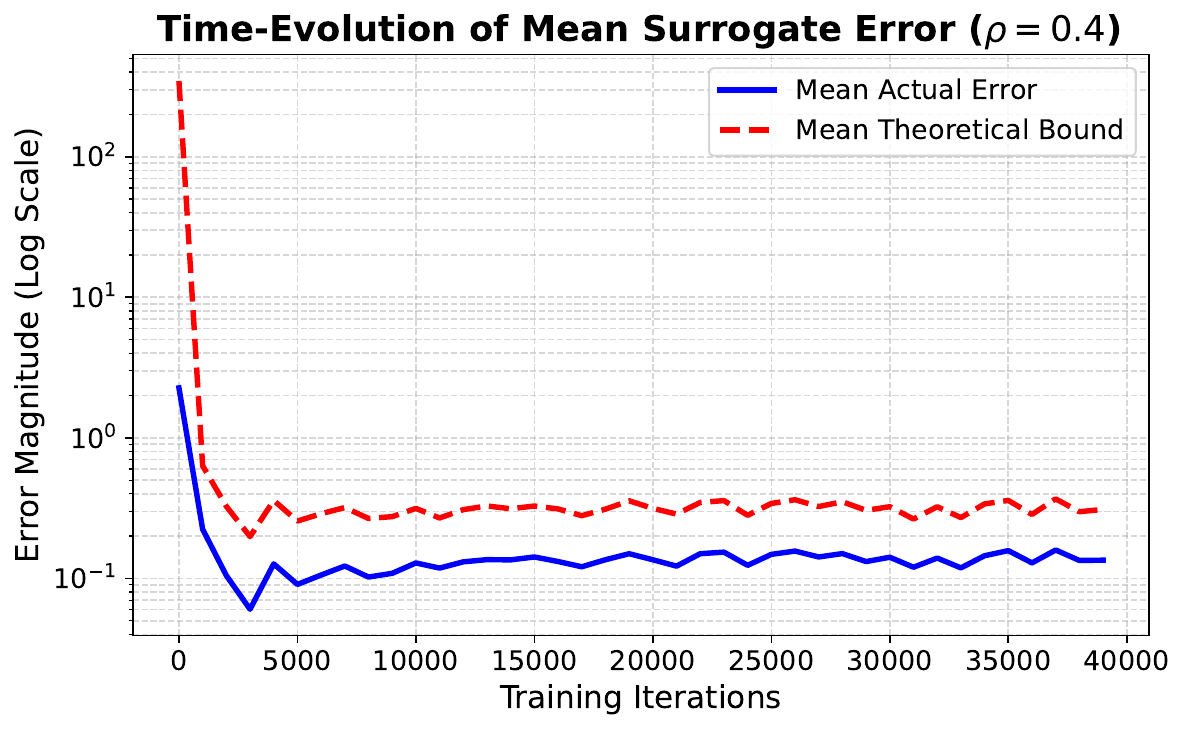}
        \caption{Correlated sources (\(\rho=0.4\))}
        \label{fig:taylor_transient_rho04}
    \end{subfigure}
    \caption{\textbf{Transient Taylor-surrogate diagnostics.}
    Exact Taylor remainder and corresponding spectral upper bound over training iterations for two representative source-correlation levels.}
    \label{fig:taylor_transient_appendix}
\end{figure}

\subsection{Simulation hyperparameters}
\label{appendix:simulation_hyperparameters}

This section summarizes the hyperparameters used in all reported simulations. We first describe the generic implementation choices for \textit{Predictive Entropy Maximization}, and then list the experiment-specific settings.

\paragraph{Generic implementation details for Predictive Entropy Maximization.}
All \textit{Predictive Entropy Maximization} experiments were run in the online setting with a single streaming pass over the data. Unless otherwise stated, we used the regularization constant $\varepsilon = 10^{-5}$. For each incoming sample $x(t)$, the fast inference dynamics were initialized at $y(t;0)=0$. If no custom initialization was supplied, the feedforward matrix was initialized as
\begin{equation}
W(0)=I+\xi_W,
\end{equation}
where $I$ denotes the rectangular identity and $\xi_W$ has i.i.d.\ Gaussian entries with standard deviation $0.01$. Likewise, if not specified explicitly, the running mean and covariance states were initialized as
\begin{equation}
\hat\mu_y(0)=0,
\qquad
\hat C_y(0)=0.2\,I.
\end{equation}

The slow feedforward update uses the step size schedule selected by \texttt{lr\_W\_rule}. In the experiments below, we used either:
\begin{equation}
\alpha_W(t)=\alpha_W^0
\qquad \text{(\texttt{constant})},
\end{equation}
\begin{equation}
\alpha_W(t)=\max\left\{\frac{\alpha_W^0}{t/T_W+1},\,10^{-8}\right\}
\qquad \text{(\texttt{divide\_by\_index})},
\end{equation}
or
\begin{equation}
\alpha_W(t)=\max\left\{\frac{\alpha_W^0}{1+\log(t/T_W+2)},\,10^{-8}\right\}
\qquad \text{(\texttt{divide\_by\_log\_index})},
\end{equation}
where $T_W$ is the decay divider.

The fast neural inference loop uses the step size schedule selected by \texttt{neural\_lr\_rule}. In the experiments below, we used either:
\begin{equation}
\eta_y(\tau)=\eta_y^{0}
\qquad \text{(\texttt{constant})},
\end{equation}
\begin{equation}
\eta_y(\tau)=\max\left\{\frac{\eta_y^{0}}{\tau+1},\,\eta_y^{\min}\right\}
\qquad \text{(\texttt{divide\_by\_loop\_index})},
\end{equation}
or
\begin{equation}
\eta_y(\tau)=\max\left\{\frac{\eta_y^{0}}{\tau T_y+1},\,\eta_y^{\min}\right\}
\qquad \text{(\texttt{divide\_by\_slow\_loop\_index})},
\end{equation}
where $T_y$ is the neural learning-rate decay divider. For sparse, nonnegative sparse, and simplex domains, the scalar threshold/inhibitory variable $\lambda_L$ was updated with step size $\eta_\lambda$, implemented in code as \texttt{stlambda\_lr}.

\paragraph{Predictive Entropy Maximization hyperparameters.}
The synthetic source-separation experiments use $n=5$ sources, $m=10$ mixtures, and $T=10^5$ samples. The auditory experiment uses $n=3$ sources and $m=5$ mixtures. The sparse receptive-field experiment uses $n=m=144$.

\begin{itemize}
    \item \textbf{Correlated antisparse sources ($\mathcal B_{\mathrm{max}}$).}
    We used
    $\lambda_{\mathrm{lat}}=0.99$,
    $\gamma_{\mathrm{pred}}=250$,
    $\alpha_W^0=5\times 10^{-2}$,
    $\eta_y^{0}=0.5$,
    $\eta_y^{\min}=10^{-6}$,
    $\tau_{\max}=250$,
    and tolerance $10^{-7}$.
    The feedforward learning rate used \texttt{divide\_by\_index} with divider $T_W=5000$, and the neural step size used \texttt{divide\_by\_loop\_index}. No threshold variable was used.

    \item \textbf{Correlated nonnegative antisparse sources ($\mathcal B_{\mathrm{max},+}$).}
    We used
    $\lambda_{\mathrm{lat}}=0.95$,
    $\gamma_{\mathrm{pred}}=750$,
    $\alpha_W^0=5\times 10^{-2}$,
    $\eta_y^{0}=0.05$,
    $\eta_y^{\min}=10^{-4}$,
    $\tau_{\max}=500$,
    $ \varepsilon =10^{-4} $
    and tolerance $10^{-6}$.
    The feedforward learning rate used \texttt{divide\_by\_index} with divider $T_W=20000$, and the neural step size used \texttt{divide\_by\_loop\_index}. In this experiment we overrode the default initialization and set
    \begin{equation}
    \hat C_y(0)=2I,
    \qquad
    W(0)=0.01\,I+\frac{1}{15}\,\xi_W,
    \end{equation}
    with $\xi_W$ i.i.d.\ Gaussian.

    \item \textbf{Noisy sparse sources ($\mathcal B_1$).}
    We used
    $\lambda_{\mathrm{lat}}=0.99$,
    $\gamma_{\mathrm{pred}}=150$,
    $\alpha_W^0=5\times 10^{-2}$,
    $\eta_y^{0}=0.05$,
    $\eta_y^{\min}=10^{-4}$,
    $\eta_\lambda=0.5$,
    $\tau_{\max}=100$,
    and tolerance $10^{-6}$.
    The feedforward learning rate used \texttt{divide\_by\_index} with divider $T_W=5000$, and the neural step size used \texttt{divide\_by\_loop\_index}.
    
    \item \textbf{Noisy nonnegative sparse sources ($\mathcal B_{1,+}$).}
    We used
    $\lambda_{\mathrm{lat}}=0.99$,
    $\gamma_{\mathrm{pred}}=250$,
    $\alpha_W^0=5\times 10^{-2}$,
    $\eta_y^{0}=0.1$,
    $\eta_y^{\min}=10^{-4}$,
    $\eta_\lambda=0.5$,
    $\tau_{\max}=100$,
    and tolerance $10^{-7}$.
    The feedforward learning rate used \texttt{divide\_by\_index} with divider $T_W=2000$, and the neural step size used \texttt{divide\_by\_loop\_index}.
    \item \textbf{Noisy simplex sources ($\Delta$).}
    We used
    $\lambda_{\mathrm{lat}}=0.99$,
    $\gamma_{\mathrm{pred}}=150$,
    $\alpha_W^0=5\times 10^{-2}$,
    $\eta_y^{0}=0.1$,
    $\eta_y^{\min}=10^{-4}$,
    $\eta_\lambda=0.05$,
    $\tau_{\max}=100$,
    and tolerance $10^{-7}$.
    The feedforward learning rate used \texttt{divide\_by\_log\_index} with divider $T_W=5000$, and the neural step size used \texttt{divide\_by\_loop\_index}.

    \item \textbf{Auditory source separation.}
    The sparse-domain audio experiment was run after transforming the mixtures to a wavelet domain using a db4 wavelet with decomposition level $3$. The \textit{Predictive Entropy Maximization} hyperparameters were
    $\lambda_{\mathrm{lat}}=0.95$,
    $\gamma_{\mathrm{pred}}=150$,
    $\alpha_W^0=9.5\times 10^{-1}$,
    $\eta_y^{0}=0.01$,
    $\eta_y^{\min}=10^{-4}$,
    $\eta_\lambda=0.5$,
    $\tau_{\max}=100$,
    and tolerance $10^{-6}$.
    The feedforward learning rate used \texttt{divide\_by\_index} with divider $T_W=2000$, and the neural step size used \texttt{divide\_by\_loop\_index}.

    \item \textbf{Sparse receptive-field learning on natural images.}
    For the $12\times 12$ patch experiment we used the sparse-domain architecture with
    \begin{equation}
    \lambda_{\mathrm{lat}}=1-\frac{10^{-3}}{7},
    \qquad
    \gamma_{\mathrm{pred}}=3,
    \qquad
    \alpha_W^0=10^{-4},
    \end{equation}
    together with
    $\eta_y^{0}=0.05$,
    $\eta_y^{\min}=10^{-6}$,
    $\eta_\lambda=5\times 10^{-2}$,
    $\tau_{\max}=500$,
    and tolerance $10^{-6}$.
    The feedforward learning rate was kept constant, while the neural step size used \texttt{divide\_by\_slow\_loop\_index} with divider $T_y=100$.
    We also used custom initializations
    \begin{equation}
    \hat C_y(0)=I+\frac{1}{250}\,\xi_C,
    \qquad
    W(0)=I+\frac{1}{250}\,\xi_W,
    \end{equation}
    where $\xi_C$ and $\xi_W$ have i.i.d.\ Gaussian entries.
\end{itemize}

\subsection{Ablation Studies}
\label{appendix:ablation_studies}

We report two additional ablations to better understand the robustness of \textit{Predictive Entropy Maximization}. The first studies the sensitivity of the method to the distribution of the mixing-matrix entries. The second studies how performance varies with the number of mixtures while keeping the number of sources fixed.

\subsubsection{Sensitivity to the mixing-matrix distribution}

We first study whether the method is sensitive to the particular law used to generate the mixing matrix. For this ablation, we consider the antisparse and nonnegative antisparse domains, fix the source correlation level at \(\rho=0\), and fix the input SNR at \(30\) dB. In both cases, we use \(n=5\) sources, \(m=10\) mixtures, and \(T=10^5\) samples. Results are aggregated over \(30\) random seeds.

For each seed, we keep the source realization and the observation noise fixed and vary only the distribution of the entries of the mixing matrix \(\mA\). We consider five centered distributions, all scaled to have unit variance:
\[
\mathcal N(0,1),\qquad
\mathcal U(-\sqrt{3},\sqrt{3}),\qquad
\mathrm{Laplace}(0,1/\sqrt{2}),\qquad
\mathrm{Rad}(\pm 1),\qquad
\sqrt{3/5}\,t_5.
\]
Here, \(\mathrm{Rad}(\pm1)\) denotes a Rademacher random variable taking values \(\pm1\) with equal probability, and \(t_5\) denotes a Student-\(t\) random variable with \(5\) degrees of freedom. The scaling by \(\sqrt{3/5}\) ensures unit variance, so that the comparison isolates the effect of distributional shape rather than trivial scale differences.

Figure~\ref{fig:ablation_mixing_distribution} shows the resulting box plots for the mean component SNR. Across both domains, the central performance remains broadly stable across the five mixing laws. This indicates that the behavior of \textit{Predictive Entropy Maximization} is not tied to Gaussian mixing matrices and transfers well to heavier-tailed, bounded, and discrete mixing distributions. The nonnegative antisparse case exhibits slightly larger variability across seeds, but again no systematic degradation under non-Gaussian or discrete mixing is observed.

\begin{figure}[t]
    \centering
    \begin{subfigure}{0.48\linewidth}
        \centering
        \includegraphics[width=\linewidth]{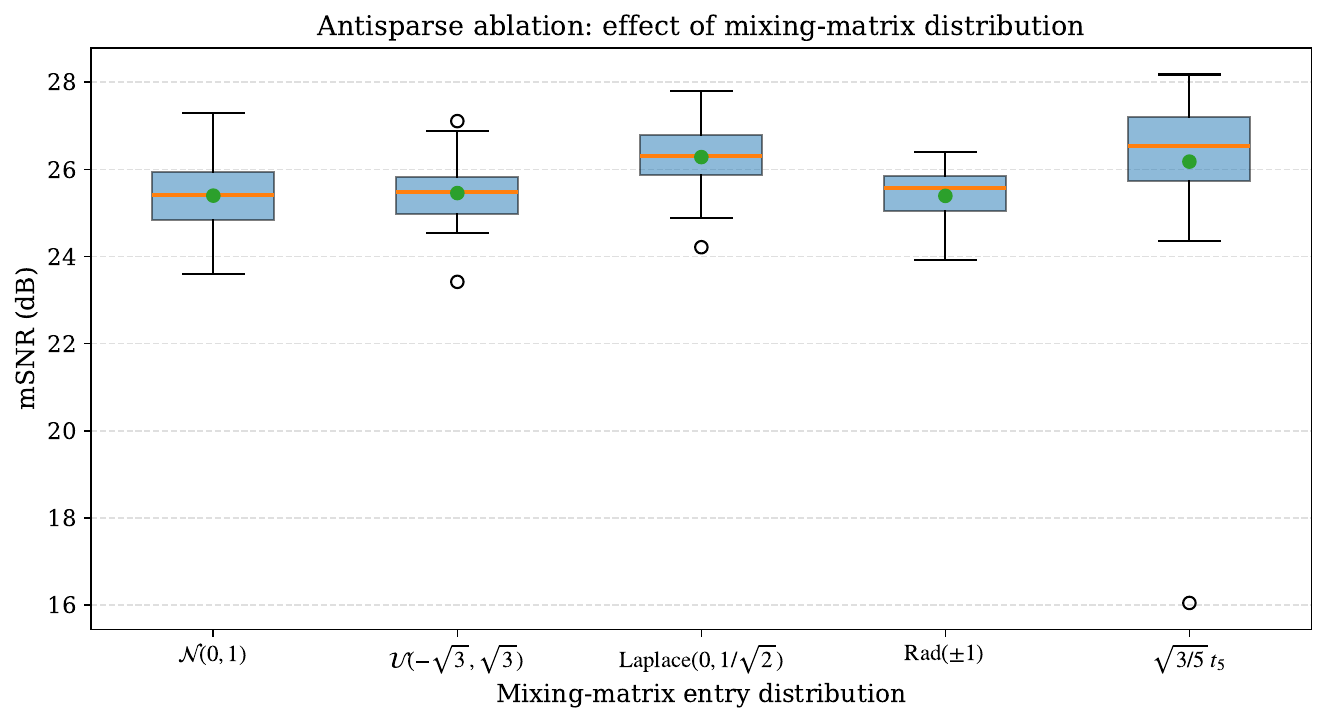}
        \caption{Antisparse sources (\(\mathcal{B}_{\mathrm{max}}\))}
        \label{fig:ablation_mixing_distribution_antisparse}
    \end{subfigure}
    \hfill
    \begin{subfigure}{0.48\linewidth}
        \centering
        \includegraphics[width=\linewidth]{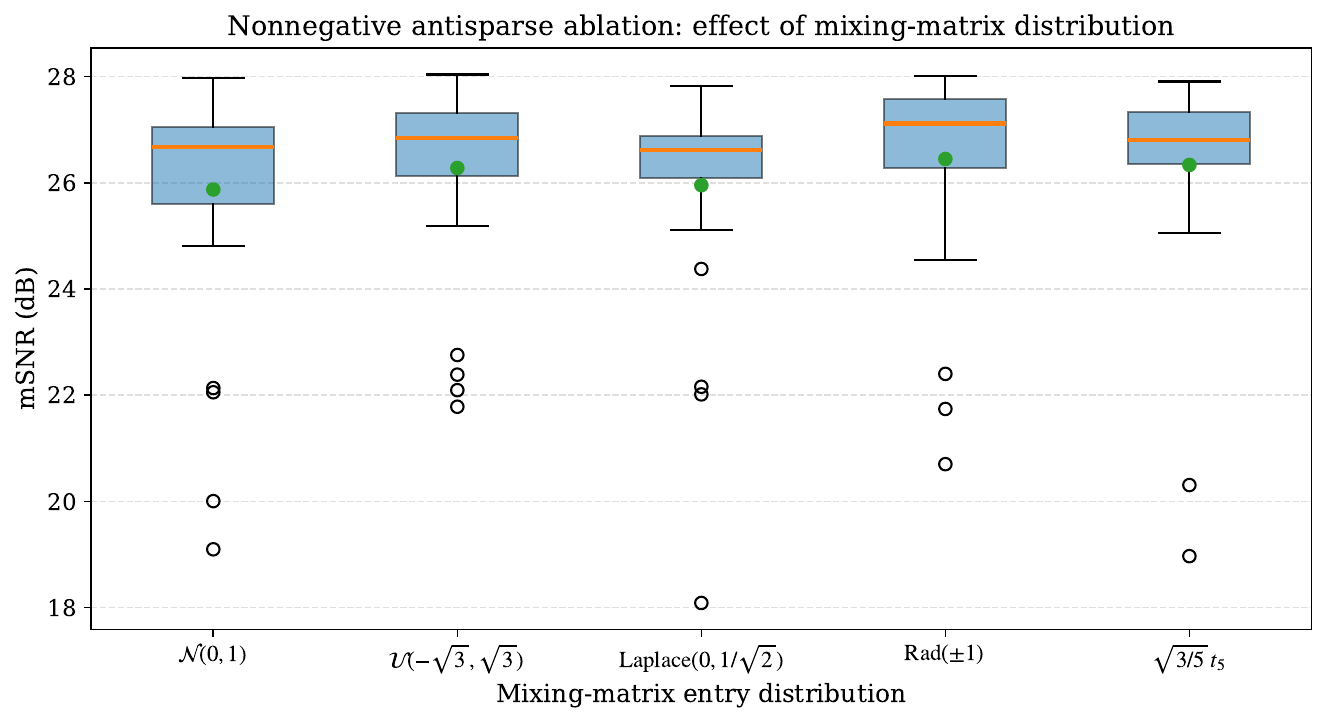}
        \caption{Nonnegative antisparse sources (\(\mathcal{B}_{\mathrm{max},+}\))}
        \label{fig:ablation_mixing_distribution_nnantisparse}
    \end{subfigure}
    \caption{\textbf{Ablation with respect to the distribution of the mixing-matrix entries.}
    Box plots summarize the distribution of mSNR over \(30\) seeds in the uncorrelated setting \((\rho=0)\) at input SNR \(30\) dB. All candidate mixing laws are centered and scaled to unit variance:
    \(\mathcal N(0,1)\), \(\mathcal U(-\sqrt{3},\sqrt{3})\), \(\mathrm{Laplace}(0,1/\sqrt{2})\), \(\mathrm{Rad}(\pm1)\), and \(\sqrt{3/5}\,t_5\). The performance of \textit{Predictive Entropy Maximization} remains broadly stable across these choices, indicating that the method is not sensitive to the precise distributional shape of the mixing coefficients.}
    \label{fig:ablation_mixing_distribution}
\end{figure}

\subsubsection{Effect of the number of mixtures}

We next study how performance varies with the number of observed mixtures. For this ablation, we consider the sparse and simplex domains, fix the number of sources at \(n=5\), fix the input SNR at \(30\) dB, and vary the number of mixtures over
\[
m \in \{7,8,9,10,11,12,13\}.
\]
As before, each experiment uses \(T=10^5\) samples and is repeated over \(30\) random seeds.

For a fixed seed, we keep the source realization fixed and generate a single Gaussian mixing matrix of size \(13\times 5\). The experiment with \(m\) mixtures then uses the first \(m\) rows of this matrix. This nested construction makes the comparison more controlled, since increasing the number of mixtures corresponds to adding new observation channels rather than resampling a completely different mixing matrix.

Figure~\ref{fig:ablation_number_of_mixtures} reports the mean mSNR together with a standard-error envelope across seeds. In both domains, performance improves as the number of mixtures increases. This is consistent with the intuition that additional observation channels make the inverse problem easier by providing richer linear views of the same latent sources. The trend is visible already in the sparse case and is even more pronounced in the simplex setting. These results confirm that the method scales favorably with observation dimension in the overdetermined regime.

\begin{figure}[t]
    \centering
    \begin{subfigure}{0.48\linewidth}
        \centering
        \includegraphics[width=\linewidth]{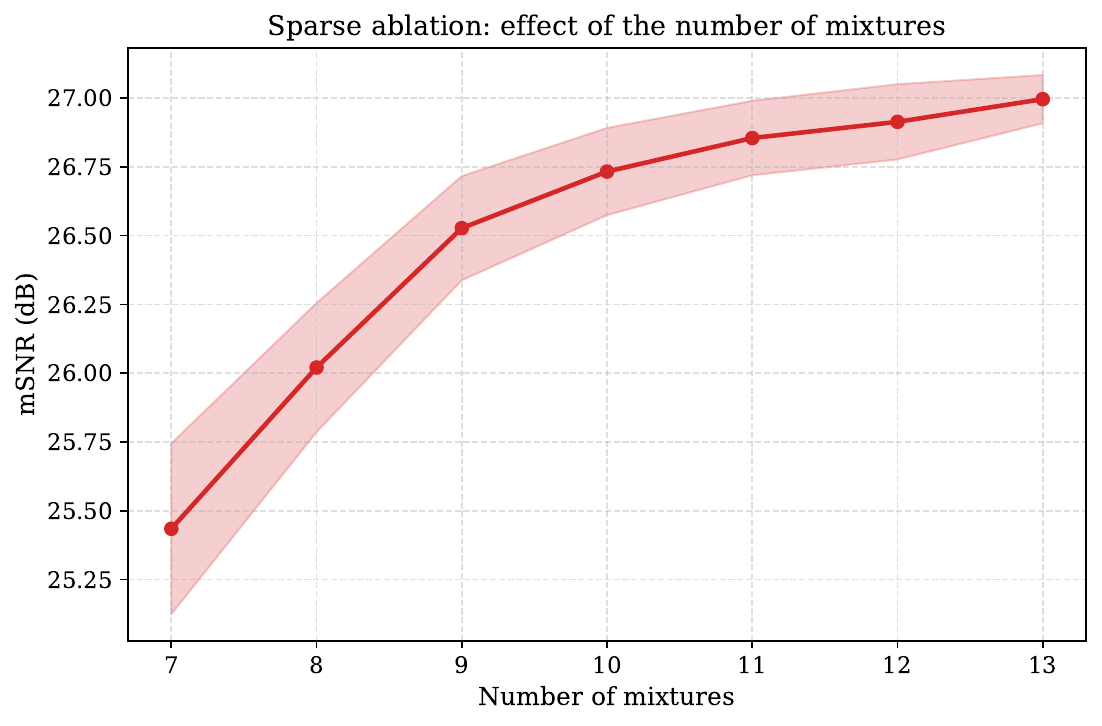}
        \caption{Sparse sources (\(\mathcal{B}_1\))}
        \label{fig:ablation_number_of_mixtures_sparse}
    \end{subfigure}
    \hfill
    \begin{subfigure}{0.48\linewidth}
        \centering
        \includegraphics[width=\linewidth]{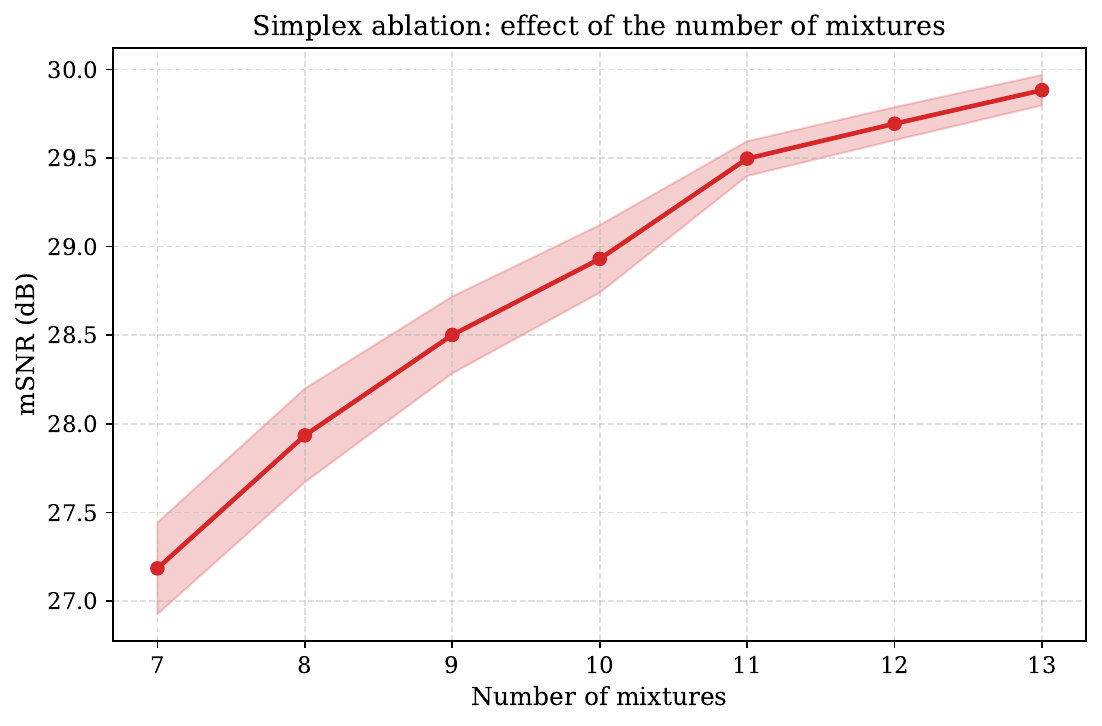}
        \caption{Simplex sources (\(\Delta\))}
        \label{fig:ablation_number_of_mixtures_simplex}
    \end{subfigure}
    \caption{\textbf{Ablation with respect to the number of mixtures.}
    Mean mSNR is plotted against the number of mixtures \(m\in\{7,\dots,13\}\) while keeping the number of sources fixed at \(n=5\). The shaded envelopes show one standard error over \(30\) seeds. For each seed, a single Gaussian \(13\times 5\) mixing matrix is sampled and the experiment with \(m\) mixtures uses its first \(m\) rows, so the comparison isolates the effect of adding observation channels. In both the sparse and simplex domains, performance improves as the number of mixtures increases.}
    \label{fig:ablation_number_of_mixtures}
\end{figure}

\subsection{Unnormalized Predictive Entropy Maximization}
\label{sec:SimplifiedPEM_Appendix}
The \textit{Predictive Entropy Maximization} (\textit{PEM}) model includes cross-covariance terms that are variance normalized in both the energy function and the activity gradient update (Eqs.~\ref{eq:full_cost_J}-\ref{eq:output_inference_direction}). If we remove this normalization term, the objective becomes
\begin{equation}
    \mathcal{J}_t
    =
    - \sum_{i=1}^{n} \log (\hat{v}_i(t)+\varepsilon)
    +
    \frac{1}{2}
    \gamma_{lateral}
    \sum_{i=1}^{n}
    \sum_{\substack{j=1 \\ j\neq i}}^{n}
        \hat{c}_{ij}(t)^{2}
    +
    \gamma \|\vy(t)-\mW(t-1)\vx(t)\|_2^2,
    \label{eq:simpler_cost_J}
\end{equation}
where we also introduced an additional hyperparameter \( \gamma_{lateral}\) that controls the strength of the covariance term, corresponding to lateral connections. We call this unnormalized variant of the model \textit{u-PEM}. Furthermore, the coordinate gradient update becomes (cf. Eq.~\ref{eq:output_inference_direction})
\begin{equation}
    d_k(t,\tau)
    =
    \frac{\bar y_k(t,\tau)}{\hat{v}_k(t)+\varepsilon}
    -
    \gamma_{lateral}\sum_{\substack{j=1\\j\neq k}}^{n}
    \hat{c}_{kj}(t)\,\bar y_j(t,\tau)
    -
    \gamma \left(y_k(t, \tau)-\sum_{\ell=1}^{m}W_{k\ell}(t - 1)x_\ell(t)\right).
    \label{eq:simpler_inference_direction}
\end{equation}
 The update is now directly given by the product of the lateral weight and the neural activity, with no variance scaling, making \textit{u-PEM} arguably even more biologically plausible than \textit{PEM}. The removal of the variance normalization is further motivated by the observation that the mean variance of the \textit{PEM} model appears to stabilise during training, as shown in Figure~\ref{fig:vicreg}. Benchmark results for \textit{u-PEM} and \textit{PEM} are reported in Figures~\ref{fig:numerical_experiments_main} and \ref{fig:numerical_experiments_appendix}. The hyperparameter values \(\gamma_{lateral}\) for the antisparse, non-negative antisparse, sparse, non-negative sparse and simplex domains were 10, 300, 50, 3200 and 100, respectively.
\begin{figure}[t]
    \centering
    \begin{subfigure}{0.48\linewidth}
        \centering
        \includegraphics[width=\linewidth]{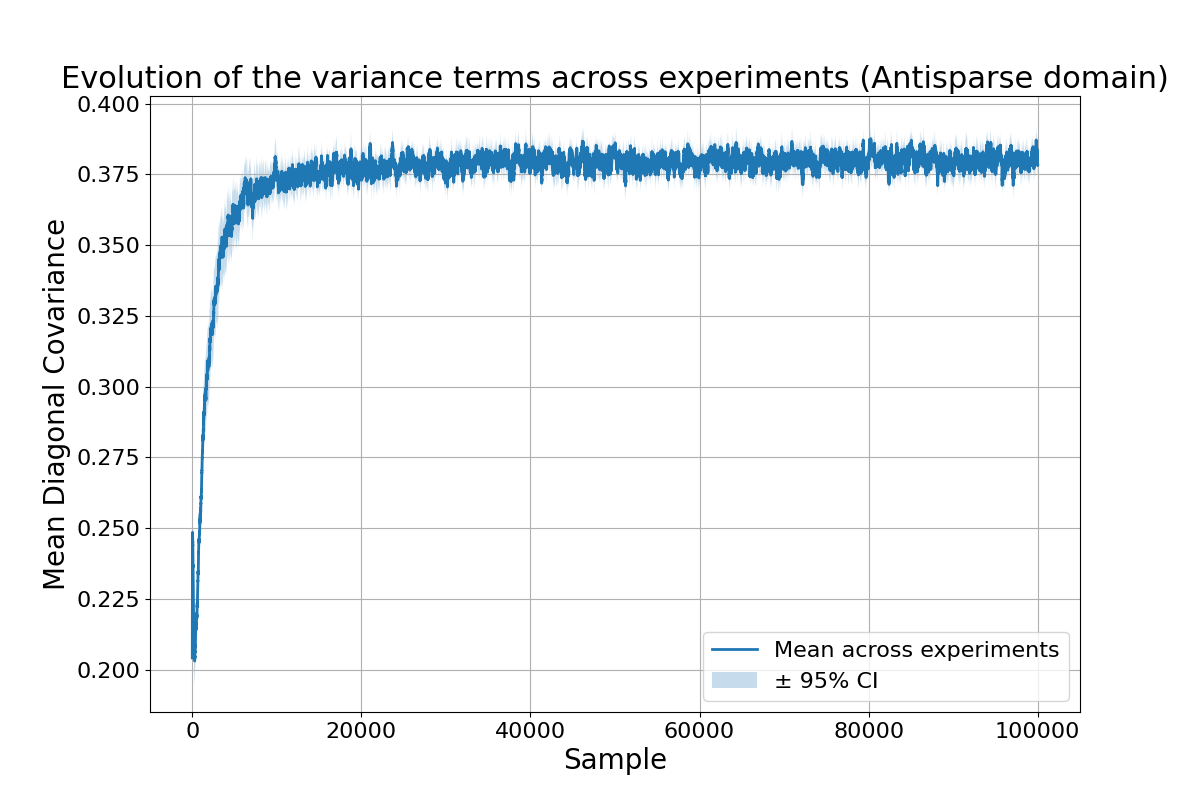}
        \caption{Antisparse sources (\(\mathcal{B}_{\mathrm{max}}\))}
        \label{fig:vicreg_antisparse_exp}
    \end{subfigure}
    \hfill
    \begin{subfigure}{0.48\linewidth}
        \centering
        \includegraphics[width=\linewidth]{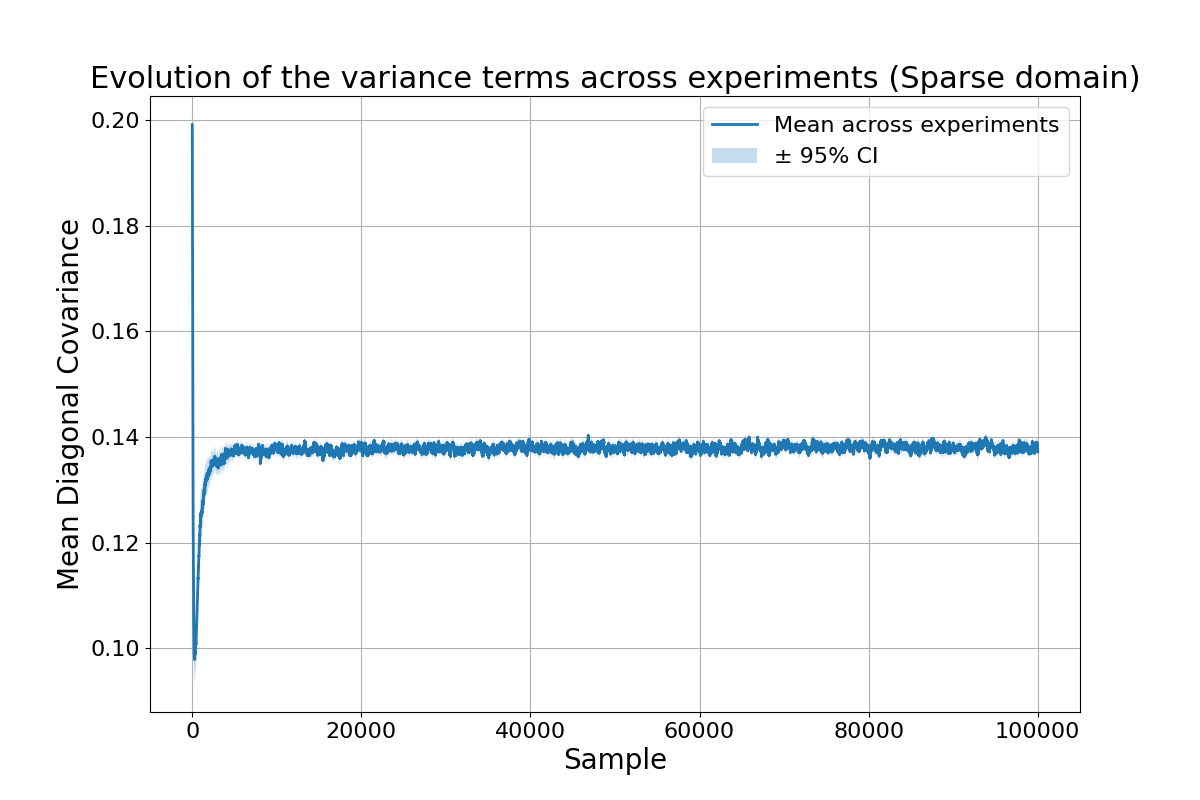}
        \caption{Sparse sources (\(\mathcal{B}_1\))}
        \label{fig:vicreg_sparse_exp}
    \end{subfigure}
    \caption{\textbf{Evolution of the variance  term in \textit{PEM} during training.}
    The evolution of the mean variance (averaged across the sources) is plotted across training samples for the \textit{PEM} model. Solid lines and shaded bands indicate the mean and the 95\% confidence interval over \(30\) seeds, respectively. For each seed, a single Gaussian \(10\times 5\) mixing matrix is sampled. These experiments demonstrate that the variance terms stabilise to a small range of values, motivating our simplification further (as well as providing a starting point for the value of \(\gamma_{lateral}\)).}
 \label{fig:vicreg}
\end{figure}

\subsection{Variance-covariance regularisation in self-supervised learning}
An influential approach to self-supervised learning is to learn embeddings or latent representations of inputs (e.g. images) which are invariant to various transformations (e.g. crops). However, self-supervised objectives are prone to learning constant or degenerate embeddings (e.g. not using all the available dimensions), a phenomenon known as ``representation collapse'' \citep{grill2020bootstrap, zbontar2021barlow, lecun2022path}. Two main methods have been developed to prevent or mitigate against such collapse: (i) contrastive methods \citep{DBLP:journals/corr/abs-2002-05709}, which involve training with both positive and negative examples; and (ii) non-contrastive or regularized approaches, which directly modify the loss function and which are most closely related to the \textit{Predictive Entropy Maximization} model.

For example, the objective of ``Barlow Twins'' \citep{zbontar2021barlow} pushes the cross-correlation matrix between the embeddings of two identical networks towards the identity. This promotes both invariance to distortions (via the diagonal terms) and decorrelation (via the off-diagonal terms). The method was named after the neuroscientist H. Barlow for introducing the idea of redundancy-reduction, hypothesizing that the visual system minimizes the statistical dependency between neural firing rates \citep{Barlow1961} (which relates to our experiment on sparse receptive fields; see Figure~\ref{fig:receptive_field}). 

Building on the Barlow Twins, ``VICReg'' \cite{bardes2022vicreg} introduced a variance term that encourages spread in each embedding dimension. The overall loss function therefore contains a variance, a covariance and an invariance term. This is similar to the objective of the \textit{Predictive Entropy Maximization} model, which also features variance-maximization and covariance-minimization terms under structural constraints (see Eq.~\ref{eq:full_cost_J}). Although the models used for self-supervised learning are not biologically plausible and address a different problem than source separation, they provide a useful point of comparison given the similarity of the objectives (and perhaps a common goal of encouraging sufficient spread in the domain). 

A form of variance-covariance regularization has also been used in recent large-scale models such as Video JEPA \citep{drozdov2024videorepresentationlearningjointembedding}. Overall, extending \textit{Predictive Entropy Maximization} to self-supervised or supervised learning tasks could be an interesting future direction.

\subsection{Computational complexity of the proposed method}
\label{appendix:Compute_Complexity}
In this section, we characterize the computational cost of \textit{Predictive Entropy Maximization}. Let \(n\) denote the number of sources, \(m\) the number of mixtures, and \(\tau_{\max}\) the maximum number of fast neural-dynamics iterations used to compute the output for a single input sample.

The dominant cost comes from the fast inference step. For one neural-dynamics iteration, the feedforward prediction \(\mW(t)\vx(t)\) requires \(O(nm)\) operations, while the covariance-driven lateral term in Equation \ref{eq:output_inference_direction} requires \(O(n^2)\) operations. The domain-dependent output nonlinearity adds only \(O(n)\) work, and for \(\gB_1\), \(\gB_{1,+}\), and \(\Delta\), the additional update of the shared inhibitory variable \(\lambda_L\) also remains \(O(n)\). Therefore, one fast iteration costs
\[
O(nm+n^2),
\]
and the full inference stage costs
\[
O\!\bigl(\tau_{\max}(nm+n^2)\bigr).
\]

After the output has converged, the slow updates are cheaper. The feedforward update
\[
\mW(t)=\mW(t-1)+\alpha_W(t)\ve(t)\vx(t)^\top
\]
costs \(O(nm)\), the mean update costs \(O(n)\), the diagonal variance update costs \(O(n)\), and the off-diagonal covariance update costs \(O(n^2)\). Hence the slow plasticity and statistics updates contribute
\[
O(nm+n^2)
\]
per sample.

Combining both stages, the overall worst-case cost per sample is
\[
O\!\bigl((\tau_{\max}+1)(nm+n^2)\bigr).
\]
In the determined or overdetermined regime \(m\ge n\), this simplifies to
\[
O(\tau_{\max}mn),
\]
since the recurrent inference stage dominates.

Thus, as in other biologically plausible recurrent BSS algorithms, the computational bottleneck is the iterative output inference rather than the synaptic updates. This is the same asymptotic order reported for the antisparse CorInfoMax network and for the determinant-maximization WSM networks, both of which are likewise dominated by recurrent output computation in digital simulations.

\paragraph{Compute resources.}
The implementation is CPU-compatible and does not require GPU acceleration or high-performance computing; the reported experiments can be run on a standard personal computer with sufficient memory. The runs reported here were performed with Python~3.12.12 on a Linux x86\_64 machine (Linux~6.8.0-107-generic, glibc~2.39) using a single-node, single-task CPU Slurm job with 10GB requested memory.
\stopcontents[appendix]

\newpage

\end{document}